%% file: main.tex
\documentclass[preprint,12pt,table]{elsarticle}

\usepackage{amssymb}

\usepackage{makecell}

\usepackage{scalerel}
\usepackage{lipsum}
\usepackage{amsfonts,amssymb}

\usepackage{graphicx}
\usepackage{epstopdf}
\usepackage{dsfont}
\usepackage{bbm}
\usepackage{bbold}
\usepackage{indentfirst}
\usepackage{subfigure}
\usepackage{hhline}
\usepackage{mathtools}
\usepackage{commath}
\usepackage{extarrows}
\usepackage{mathrsfs} 
\usepackage{lineno}
\usepackage{multirow}
\usepackage{tabularray}

\usepackage{textcomp}
\usepackage{stfloats}
\usepackage{pgfplots}
\pgfplotsset{compat=1.17} 
\usepackage{stmaryrd}
\usepackage{algorithm}
\usepackage[noend]{algpseudocode}

\allowdisplaybreaks

\makeatletter
\newcommand\multiline[1]{\parbox[t]{\dimexpr\linewidth-\ALG@thistlm}{#1}}
\makeatother

\usepackage{tikz}
\usetikzlibrary{shapes, arrows.meta, patterns}

\usepackage{figure/carshapes}
\usepackage{centernot}
\usepackage{hyperref,doi,url}
\hypersetup{colorlinks=true, linkcolor=blue, breaklinks=true, urlcolor=blue}

\ifpdf
  \DeclareGraphicsExtensions{.eps,.pdf,.png,.jpg}
\else
  \DeclareGraphicsExtensions{.eps}
\fi

\newtheorem{defi}{\textbf{Definition}}
\newtheorem{thom}{\textbf{Theorem}}
\newtheorem{asp}{Assumption}
\newtheorem{rek}{\textbf{Remark}}
\newtheorem{pro}{\textbf{Proposition}}
\newtheorem{lema}{\textbf{Lemma}}

\newtheorem{cor}{\textbf{Corollary}}
\newdefinition{examp}{\textbf{Example}}

\newenvironment{proof}{\paragraph{Proof}}{\hfill$\square$ \vskip8pt}

\newcommand{\defiref}[1]{Definition~\ref{#1}}

\newcommand{\thomref}[1]{Theorem~\ref{#1}}
\newcommand{\aspref}[1]{Assumption~\ref{#1}}
\newcommand{\rekref}[1]{Remark~\ref{#1}}
\newcommand{\proref}[1]{Proposition~\ref{#1}}
\newcommand{\algoref}[1]{Algorithm~\ref{#1}}

\newcommand{\cororef}[1]{Corollary~\ref{#1}}
\newcommand{\lemaref}[1]{Lemma~\ref{#1}}
\newcommand{\tabref}[1]{Table~\ref{#1}}
\newcommand{\sectref}[1]{Section~\ref{#1}}
\newcommand{\egref}[1]{Example~\ref{#1}}

\renewcommand{\emptyset}{\varnothing}

\DeclareMathOperator*{\opt}{opt}
\DeclareMathOperator*{\argmax}{argmax}
\DeclareMathOperator*{\argmin}{argmin}

\newcommand{\obs}{\mathit{obs}}

\newcommand{\Loc}{\mathit{Loc}}
\newcommand{\Per}{\mathit{Per}}

\newcommand{\loc}{\mathit{loc}}
\newcommand{\per}{\mathit{per}}
\newcommand{\csg}{\mathsf{C}}

\newcommand{\game}{\mathsf{G}}

\newcommand{\agent}{\mathsf{Ag}}

\newcommand{\sem}[1]{\llbracket {#1} \rrbracket}

\newcommand{\ipaths}{\mathit{IPaths}}
\newcommand{\fpaths}{\mathit{FPaths}}

\newcommand{\startpara}[1]{{%
\vskip5pt\noindent
{\bf #1.}}}

\newcommand{\gethinM}[1]{{\marginpar{\color{blue}\textsf{GN: #1}}}} 

\newcommand{\ruiM}[1]{{\marginpar{\color{red}\textsf{RY: #1}}}} 

\newcommand{\martaM}[1]{{\marginpar{\color{orange}\textsf{MK}}}} 

\newcommand{\revise}{}

\makeatletter
\renewcommand{\ALG@name}{\sc Algorithm}
\makeatother

\usepackage{amsopn}

\journal{Information and Computation}

\begin{document}

\begin{frontmatter}



\title{Strategy Synthesis for Zero-Sum Neuro-Symbolic Concurrent Stochastic Games}


\author[Oxford]{Rui~Yan}\ead{rui.yan@cs.ox.ac.uk}
\author[Oxford]{Gabriel~Santos}
\ead{gabriel.santos@cs.ox.ac.uk}
\author[Oxford,Glasgow]{Gethin~Norman}
\ead{gethin.norman@glasgow.ac.uk}
\author[Oxford]{David~Parker}
\ead{david.parker@cs.ox.ac.uk}
\author[Oxford]{Marta~Kwiatkowska}
\ead{marta.kwiatkowska@cs.ox.ac.uk}

\affiliation[Oxford]{organization={Department of Computer Science, University of Oxford},
            city={Oxford},
            postcode={OX1 2JD}, 
            country={UK}}
\affiliation[Glasgow]{organization={School of Computing Science, University of Glasgow},
            city={Glasgow},
            postcode={G12 8QQ}, 
            country={UK}}

\begin{abstract}
Neuro-symbolic approaches to artificial intelligence,
which combine neural networks with classical symbolic techniques,
are growing in prominence, necessitating formal approaches
to reason about their correctness.
We propose a novel modelling formalism called neuro-symbolic concurrent stochastic games (NS-CSGs),
which comprise {\revise two}
probabilistic finite-state agents interacting in a shared continuous-state environment. {\revise Each agent observes the environment using a neural perception mechanism, which converts inputs such as images into symbolic percepts, and makes decisions symbolically.}
We focus on the class of NS-CSGs with Borel state spaces and
prove the existence and measurability of the value function for zero-sum discounted cumulative rewards
under piecewise-constant restrictions on the components of this class of models. 
To compute values and synthesise strategies, we present, for the first time, {\revise practical}
value iteration (VI) and policy iteration (PI) algorithms to solve {\revise this new subclass} of continuous-state CSGs. These require a finite {\revise decomposition of the environment induced by the neural perception mechanisms of the agents}
and 
rely on finite abstract representations of  value functions and strategies closed under VI or PI. 
First, we introduce a Borel measurable piecewise-constant (B-PWC) representation of value functions, extend 
minimax backups to this representation and propose {\revise a value iteration algorithm called}
B-PWC VI. 
Second, we introduce 
two novel representations for the value functions and strategies, constant-piecewise-linear (CON-PWL) and constant-piecewise-constant (CON-PWC) respectively, and propose Minimax-action-free PI by extending a recent PI method based on alternating player choices for finite state spaces to Borel state spaces, 
which does not require normal-form games to be solved. 
We illustrate our approach with a dynamic vehicle parking example by generating approximately optimal strategies using a prototype implementation of the B-PWC VI algorithm. 
\end{abstract}




\end{frontmatter}



\input{introduction}
\input{background}
\input{nscsgs}

\input{structures}
\input{values}
\input{value_iteration}
\input{policy_iteration}
\input{conclusions}



\bibliographystyle{elsarticle-num} 
 \bibliography{references}





\end{document}

%% file: introduction.tex
\section{Introduction}



Game theory offers an attractive framework for analysing strategic interactions among agents,
with application to, for instance, the game of Go~\cite{DS-AH-CJM-AG-LS-others:16}, autonomous driving \cite{SSS-SS-AS:16} and robotics \cite{JKG-ME-MK:17}. An important class of dynamic games is \emph{stochastic games}~\cite{LSS:53}, which move between states according to transition probabilities controlled jointly by multiple agents (also called players). Extending both strategic-form games to dynamic environments and Markov decision processes (MDPs) to multiple players, stochastic games have long been used to model sequential decision-making problems with more than one agent, ranging from multi-agent reinforcement learning~\cite{RY-XD-ZS-YS-JRM-FB:21} to quantitative verification and synthesis for equilibria~\cite{MK-GN-DP-GS:21}.

{\revise Increasingly, agents in these settings incorporate machine learning components.}
Recent years have witnessed encouraging advances in the use of neural networks (NNs) to approximate either value functions or strategies \cite{RY-YW-AT-JH-PA-IM:17} for stochastic games that model large, complex environments.
Such \emph{end-to-end} NNs directly map environment states to Q-values or actions.
This means that they have a relatively complex structure and a large number of weights and biases,
since they interweave multiple tasks (e.g., object detection and recognition, decision making) within a single NN.
An emerging trend in autonomous and robotic systems is \emph{neuro-symbolic} approaches,
where some components that are synthesized from data (e.g., perception modules) are implemented as NNs,
while others (e.g., nonlinear controllers) are formulated using traditional symbolic methods.
This can greatly simplify the design and training process, and yield smaller NNs.  


Even with the above advances, there remains a lack of modelling and verification frameworks
which can reason formally about the correctness of neuro-symbolic systems.
Progress has been made on techniques for both multi-agent verification~\cite{MEA-EB-PK-AL:20,LDR-SD-RM-GM:20}
and safe reinforcement learning~\cite{AVDC20} in this context,
but without the ability to reason formally about stochasticity,
which is crucial for modelling uncertainty.
Elsewhere, concurrent stochastic games (CSGs) have been widely studied~\cite{JVDW:78,BT:89,JF-KV:97,JP-BS-BP-OP:15,DB:18},
and also integrated into formal modelling and verification frameworks~\cite{MK-GN-DP-GS:21},
but primarily in the context of \emph{finite} state spaces,
which are insufficient for many real-life systems.


We propose a new modelling formalism called \emph{neuro-symbolic concurrent stochastic games (NS-CSGs)}, overviewed in the conference paper~\cite{mfcs22},
which comprise two finite-state probabilistic agents 
interacting in a shared continuous-state environment.
{\revise Each agent observes the environment 
using a neural perception mechanism, which classifies inputs such as images and sensor values into symbolic \emph{percepts} from a finite set and makes decisions using a conventional, symbolic mechanism. 
During execution, the agents alternate between invoking perception and symbolic decisions, and transition between states according to joint transition probabilities. The agent’s transition to its next local state is based on  the current local states of the agents and the agent's current percept, rather than directly accessing the continuous environment, and can thus model knowledge acquisition from the neural perception mechanism. 
}%

Under the assumption that agents have full state observability and working with Borel state spaces, we establish restrictions on the modelling formalism
which ensure that NS-CSGs belong to a {\revise new subclass}
of uncountable state-space CSGs~\cite{PRK-THS:81}
that are \emph{determined} for zero-sum discounted cumulative objectives, and therefore prove the existence and measurability of the value function for such objectives. 

{\revise
Our assumptions on the NS-CSG model enable symbolic reasoning with (exactly) learnt percepts (e.g., transition functions that preserve the decomposition of continuous-state spaces into finitely many regions), and induce finite representations of value functions that incorporate symbolic combinations of percepts and knowledge acquired by the agents through symbolic reasoning.  
We also show that our formalism can capture neural perception mechanisms instantiated as ReLU NN classifiers, which yield piecewise constant perception functions and impose a finite decomposition of the continuous environment into polytopes via preimage computation.}


{\revise Next, we develop computationally tractable algorithms for optimal strategy synthesis for NS-CSGs, which exploit the Borel and piecewise constant structure of the transition, perception and reward functions.} 
We propose a new 
representation for the value function and show its closure under a minimax operator {\revise defined for it}. Using this (finite) representation, {\revise we develop} a  
value iteration (VI) algorithm for NS-CSGs that \emph{approximates} the value of the game and prove the algorithm's convergence.

Then, we present a 
policy iteration (PI) algorithm for NS-CSGs inspired by recent work for finite state spaces \cite{DB:21}, which we generalise by using novel representations for the value functions and strategies, 
to ensure finite representability and measurability. This allows us to overcome the main issue that arises when solving Borel state space CSGs
with PI, namely that the value function may change from a Borel measurable function to a non-Borel measurable function across iterations.

The PI algorithm adopts the alternating player choices proposed in \cite{DB:21} and removes the need to solve normal-form games and MDPs at each iteration.
To the best of our knowledge, these are the first implementable algorithms for solving zero-sum CSGs over Borel state spaces with  convergence guarantees.
%
Finally, we illustrate our approach by modelling a dynamic vehicle parking as an NS-CSG and synthesizing (approximately optimal) strategies using a prototype implementation of our 
VI algorithm.

We note that we assume a \emph{fully observable} game setting. While it is relatively straightforward to generalise the NS-CSG model to partial observability, 
since NS-CSGs already include perception functions that generate observations,
there are no general algorithmic methods for value and strategy computation in the partially observable game setting; {\revise see \cite{RY-GS-GN-DP-MK:23-2} for a one-sided variant of neuro-symbolic partially observable stochastic games that shares its syntax with NS-CSGs. In the general two-sided case,} we believe that an approach similar to~\cite{NB-AB-AL-QG:20,VK-MS-NB-MB-VL:22}, which converts imperfect-information games to perfect-information, can potentially be used to enable the solution of partially observable NS-CSGs.

\input{execsumm}


\input{figure_outline}
\input{table_listings}



\subsection{Related work} 
Stochastic games were introduced by Shapley \cite{LSS:53}, who assumed a finite state space.
Since then, many researchers have considered CSGs with uncountable state spaces,
e.g., \cite{PRK-THS:81,AM-TP:70,ASN:99}.
Maitra and Parthasarathy~\cite{AM-TP:70} were the first to study discounted zero-sum CSGs in this setting, assuming that the state space is a compact metric space. Following this, more general results for discounted zero-sum CSGs with Borel state spaces have been derived, e.g., \cite{PRK-THS:81,ASN:85-2,ASN:99,OH-JBL:00}. These aim at providing sufficient conditions for the \emph{existence} of either values or optimal strategies for players.

Another important and practical problem for zero-sum CSGs with uncountable state spaces is the \emph{computation} of values and optimal strategies. Since the seminal policy iteration (PI) methods were introduced by Hoffman and Karp \cite{AJH-RMK:66} and Pollatschek and Avi-Itzhak \cite{MAP-BAI:69}, a wide range of fixed-point algorithms have been developed for zero-sum CSGs with finite state 
spaces \cite{JVDW:78,BT:89,JF-KV:97,JP-BS-BP-OP:15}. Recent work by Bertsekas \cite{DB:21} proposed a distributed optimistic abstract PI algorithm, which inherits the attractive structure of the Pollatschek and Avi-Itzhak algorithm while resolving its convergence difficulties. Value iteration (VI) and PI algorithms have been improved for simple stochastic games \cite{JK-ER-AS-MW:22,JE-EK-JK-MW:22}. However, all of the above approaches assume finite state spaces and, to the best of our knowledge, there are no existing VI or PI algorithms for CSGs with uncountable, or more specifically Borel, state spaces. VI and PI algorithms for stochastic control (i.e., the one player case) with Borel state spaces can be found in \cite{HY-DB:15,HY:15}. Other problems for zero-sum CSGs with uncountable state spaces have been studied and include information structure \cite{IHB-SY:21}, specialized strategy spaces \cite{AB-LS:20}, continuous time setup \cite{AC-13} and payoff criteria \cite{OH-JBL:00}.

A variety of other objectives, for instance, mean-payoff \cite{KC-RIJ:15,NB-MK-CW:18}, ratio \cite{NB-MK-CW:18} and reachability \cite{TB-VF-JK-JK-AK:13,JF-MNR-SS-LZ:16} objectives, have also been studied for CSGs~\cite{JVDW:78,BT:89,JF-KV:97,JP-BS-BP-OP:15}.
But these are primarily in the context of finite/countable state spaces which,
as argued above, are insufficient for our setting,  where 
uncountable real vector spaces are usually supplied as inputs to NNs.
%
Building on an earlier version of this work~\cite{YSN+22},
there has been recent progress on solving NS-CSGs~\cite{YSD+22},
but focusing on finite-horizon objectives and using equilibria-based (nonzero-sum) properties.


{\revise Compared to the above works, NS-CSGs are shown in 
 \sectref{nscsgs-sect} to form a new subclass of uncountable state-space CSGs, whose determinacy (see  \sectref{values-sect}) relies on~\cite{PRK-THS:81} under some restrictions. The closure of B-PWC functions under the minimax operator in \sectref{vi-sect} is a new result; combining the closure with determinacy, the new B-PWC VI algorithm directly relies on Banach's fixed point theorem. Our Minimax-action-free PI algorithm takes ideas from  recent work~\cite{DB:21}, which proposed a new PI method to solve zero-sum stochastic games with finite state spaces, but has to resolve a number of issues (see \sectref{pi-sect})
due to the uncountability of the underlying state space and the need to ensure Borel measurability at each iteration, including finite representability of measurable functions and strategies, strategy generation in the policy improvement and the division
of the uncountable state space into a finite set of regions at each iteration.}


Finally, we note that this paper assumes a \emph{fully observable} game setting;
a natural extension would be partially observable stochastic games (POSGs), for which there are no general VI and PI computation algorithms.
A variant of POSGs, called factored-observation stochastic games (FOSGs),
was recently proposed~\cite{VK-MS-NB-MB-VL:22}
that distinguishes between private and public observations in a similar fashion to our model,
but for finite-state models without NNs.
Partial observability in FOSGs is dealt with via a mechanism that converts imperfect-information games
into continuous-state (public belief state) perfect-information games~\cite{NB-AB-AL-QG:20,VK-MS-NB-MB-VL:22}, 
such that many techniques for perfect-information games can also be applied. Our fully observable model can arguably serve as a vehicle to later solve the more complex case with imperfect information. 
{\revise In recent work \cite{RY-GS-GN-DP-MK:23-2}, a point-based value iteration algorithm was presented for one-sided neuro-symbolic POSGs, thus dealing with partial observability under some mild assumptions.}

%% file: execsumm.tex
\subsection{Executive summary}


{\revise The modelling formalism of NS-CSGs introduced in this paper induces a new class of stochastic games with discrete observations and uncountable state spaces. The underlying continuous state space raises multiple technical issues that need to be resolved to ensure determinacy and enable practical algorithms for value computation and strategy synthesis, which are the main goal of this work. This is achieved through placing appropriate structural restrictions on NS-CSGs and relying on Borel measurability, which allows us to obtain finite abstract representations of the uncountable state space during value computation.

The key modelling decisions that allowed us to obtain our results are summarised below. To help navigate through the technical sections, we have listed the main concepts and dependencies between them in Fig.~\ref{fig:outline}.  \tabref{tab:defs}
summarises the role of the definitions,  including where they are used.

In \sectref{nscsgs-sect}, the definition of NS-CSGs restricts the interface between an agent’s neural perception mechanism and its symbolic decision-making mechanism. In particular, when making decisions, the agent can only access the continuous environment through its perception mechanism, which is limited to a finite set of observations (percepts). This allows us to abstract the (uncountable) continuous state space by partitioning it into a finite set of regions (called a finite connected partition, or FCP), with the elements of each region mapped to the same percept. 
The initial partition is obtained via preimage computation applied to the perception mechanism and is then refined in subsequent iterations. 

\sectref{structures-sect} requires Borel measurability over the continuous environment to constrain the abstractions of the environment to Borel FCPs (called BFCPs) and restrict the choices that the strategies can make to ensure measurability of the induced sets of paths. By requiring Borel measurability of the environment transition function and the fact that the structural BFCP property of the abstraction of the environment is preserved under this transition function (\aspref{asp:continuous-PWC}(i)), referred to as BFCP invertibility, 
we can derive a minimax operator for our setting. We can then apply, in \sectref{values-sect}, the result of \cite{PRK-THS:81} to prove the existence of the value for discounted expected reward as a fixed point of the minimax operator (\thomref{thom:value-function}) and convergence of value iteration (\propref{prop:convergence-seq}). However, this result does not ensure finite representability of the value functions. 

In \sectref{vi-sect}, structural assumptions are placed on the model to require Borel measurable piecewise constant (B-PWC) perception and reward functions (\aspref{asp:continuous-PWC}(ii) and (iii)). These assumptions allow us to derive BFCPs, in which the states in each region are equivalent with respect to either the perception, reward or transition function. Working with B-PWC functions as a finite representation of the value functions, in conjunction with proving that B-PWC functions are closed with respect to the minimax operator, we show that the value can be approximated by a sequence of B-PWC functions (\thomref{thom:B-PWC-closure-VI}), although the optimal value is not necessarily B-PWC representable.  

\sectref{vi-sect} demonstrates that the above assumptions and requirements allow us to derive B-PWC VI (\algoref{alg:PWC-value-VI}), which uses the B-PWC representation of value functions. The algorithm refines the BFCP induced by the value function at each iteration by computing the preimage of the transition function (\algoref{alg:BFCP-iteration-VI}). The refinement step is necessary because, for a given action, two states with the same percept may result in states
that are not observationally equivalent.

In \sectref{pi-sect} we introduce two representations, CON-PWC (constant-piecewise-constant) for Borel measurable strategies and CON-PWL (constant-piecewise-linear) for Borel measurable functions, and, under the assumptions and requirements above, prove that these representations are closed under Max-Min and Min-Max operators (\thomref{thom:repsentation-consistency}). This allows us to generalise the recent Minimax-action-free PI algorithm \cite{DB:21} for finite CSGs to NS-CSGs (\algoref{alg:maxmin-free-PI}) by ensuring finite representability and Borel measurability at each iteration. At each iteration, the algorithm refines the BCFPs used for representing strategies and value functions through a preimage computation (\algoref{alg:BFCP-iteration-PI}) and then computes values for regions in the resulting BFCPs by Max-Min and Min-Max operators (\algoref{alg:BFCP-based-computation-PI}). 

Finally, we show that our framework is amenable to neural perception mechanisms instantiated as ReLU NNs trained from data. This yields polytope BFCPs (\sectref{vi-sect}), which suffice as a symbolic representation of perception BFCPs, thus enabling value computation and policy iteration with the (exactly) learnt percepts. We also provide a prototype implementation of \algoref{alg:PWC-value-VI} and evaluate it on a dynamic parking example with the simpler perception mechanism given by a regression function, demonstrating the feasibility of strategy synthesis for NS-CSGs in practice. }

%% file: figure_outline.tex
\tikzset{%
    box/.style 2 args = {
        rectangle split, rectangle split horizontal,
        rectangle split parts=2,
        rectangle split part fill = { 
            #1, #2
        },
        draw=black,
        minimum width=\textwidth,
        inner sep=0.2cm,
        every one node part/.style={
            text=black,
            text width=0.9\textwidth
        },
        every two node part/.style={
            text=black,
            font=\bfseries,
            text width=0.03\textwidth
        }
    },
    box/.default={white}{gray!30},
    every box/.append style={background}
}

{\revise 
\begin{figure}[t!]
\centering
\begin{tikzpicture}[node distance=3.0cm]
    \node[box] (1) {\nodepart{one} \nodepart{two} \rotatebox{270}{\revise {\scriptsize 2.~Background}}};
    \node[draw=black, minimum width=0.93\textwidth, fill=gray!30, inner sep=0.4cm, above of=1,xshift=-0.41cm,yshift=-1.1cm] (bx) {$\;$};
    \node[above of=1,xshift=-4.50cm,yshift=-1.1cm] (Algo) {\scriptsize {\bf Algorithms}};
    \node[above of=1,xshift=-0.75cm,yshift=-1.1cm] (Main) {\scriptsize  {\bf Main results}};
    \node[above of=1,xshift=3.70cm,yshift=-1.1cm] (Aux) {\scriptsize {\bf \shortstack[c]{Auxiliary results \\ and definitions}}};
    \node[below of=Aux,yshift=1.1cm] (1-2) {\scriptsize \shortstack[c]{CSGs, FCPs \\ PWC and PWL functions \\ Neural networks \\ Borel measurability}};
    \node[box,below of=1,yshift=0.35cm] (2) {\nodepart{one} 
    \nodepart{two} \rotatebox{270}{\scriptsize {\revise  3.~NS-CSGs}}};
    \node[below of=Main,yshift=-1.55cm] (2-1) {\scriptsize NS-CSG Model};
    \node[below of=Aux,yshift=-1.5cm] (2-2) {\scriptsize \shortstack[c]{Reward structures \\ Strategies \\ Borel spaces}};
    %
    \node[box,below of=2] (3) {\nodepart{one} \nodepart{two} \rotatebox{270}{\scriptsize {\revise  \shortstack[c]{4.~Game structures \\ of NS-CSGs}}}};
    \node[below of=Algo,yshift=-4.55cm,circle,fill=black,inner sep=1.5pt] (3-0) {};
    \node[below of=Aux,yshift=-4.55cm] (3-2) {\scriptsize \shortstack[c]{Perception, \\ preimage \\ and reward BFCPs}};
    \draw[] (3-2.west) -- (3-0.center);
    \draw[] (2-1.west) -| (3-0.center);
    \node[box,below of=3,yshift=-0.0cm] (4) {\nodepart{one} \nodepart{two} \rotatebox{270}{\scriptsize {\revise  \shortstack[c]{5.~Values of \\  NS-CSGs}}}};
    \node[below of=Main,yshift=-7.5cm] (4-1) {\revise  {\scriptsize \shortstack[c]{Theorem~\ref{thom:value-function}: \\
    $V^{\star}$ is unique fixed point of $T$ \\ \\ Proposition~\ref{prop:convergence-seq}: \\ convergence of VI}}};
    \node[below of=Aux,yshift=-7.5cm] (4-2) {\scriptsize $\left\{ \text{\shortstack[c]{\vspace{2pt} \\ Value function $V^{\star}$ \\ Borel sets \\ Stochastic kernel \\ Minimax operator $T$ \\ \vspace{-19pt}}}\right.$};
    \path (4-2.west) edge[-{Latex[length=1.8mm]},dashed] node {} (4-1.east);
    \path (2-2.south) edge[-{Latex[length=1.8mm]},dashed] node {} (4-2.north);
    \node[box,below of=4,yshift=0.85cm] (5) {\nodepart{one} \nodepart{two} \rotatebox{270}{\scriptsize {\revise  \shortstack[c]{6.~Value \\ iteration}}}};
    \node[below of=Algo,yshift=-9.7cm] (5-0) {\scriptsize {\revise  \shortstack[c]{Algorithms~\ref{alg:PWC-value-VI} and \ref{alg:BFCP-iteration-VI} \\ B-PWC VI}}};
    \node[below of=Main,yshift=-9.7cm] (5-1) {\scriptsize {\revise  \shortstack[c]{Theorem~\ref{thom:B-PWC-closure-VI}: \\ B-PWC closure}}};
    \path (4-1.south) edge[-{Latex[length=1.8mm]}] node {} (5-1.north);
    \path (5-1) edge[-{Latex[length=1.8mm]}, above] node {\tiny \shortstack[c]{Closure and \\ convergence}} (5-0);
    \node[below of=Aux,yshift=-9.7cm] (5-2) {\scriptsize \shortstack[c]{Polytope regions \\ ReLU networks}};
    \draw[-{Latex[length=1.8mm]}] (3-0.center) -- (5-0.north);
    \node[right of=5-2,xshift=-1.0cm] (5-3) {};
    \draw[-] (3-2.east) -| (5-3.center);
    \draw[-] (5-3.center) -- (5-2.east);
    \node[box,below of=5,yshift=1.1cm] (6) {\nodepart{one} \nodepart{two} \rotatebox{270}{\scriptsize {\revise  \shortstack[c]{7.~Policy \\ iteration}}}};
    \node[below of=Algo,yshift=-11.55cm] (6-0) {\scriptsize {\revise \shortstack[c]{Algorithm~\ref{alg:maxmin-free-PI} \\ Minimax-action-free PI}}};
    \node[below of=Main,yshift=-11.55cm] (6-1) {\scriptsize {\revise \shortstack[c]{Theorems~\ref{thom:repsentation-consistency} and \ref{convergence-thm}: \\ Representation closures \\ Convergence for PI}}};
    \node[below of=Aux,yshift=-11.55cm] (6-2) {\scriptsize \shortstack[c]{Max-min \& min-max operators \\ CON-PWL function \\ CON-PWC kernel \\ Constant solutions}};
    \draw[-{Latex[length=1.8mm]}] (3-0.center) -| (6-0.north west);
\end{tikzpicture}
\caption{{\revise Outline of the main concepts and contributions of our paper, as well as dependencies between them.}}
\label{fig:outline}
\end{figure}}

%% file: table_listings.tex
\setlength\extrarowheight{2pt}
\begin{table}[!t]
\renewcommand{\arraystretch}{1.4} 
    \scriptsize
    \centering
    \begin{tabular}{|l|l|}
    \hline
    \multicolumn{2}{|l|}{\textbf{\cellcolor{gray!20}\hyperref[background-sect]{Section 2: Background}}} \\ \hline
    \makecell[l]{\hyperref[defi:FCP_Borel_FCP]{\textbf{1:} FCP and Borel FCP}} & \makecell[l]{Decomposes uncountable state spaces into finite sets of\\ regions for a finite abstraction of the environment.
    \\ Used in Definitions~\ref{defi:PWC-Borel-func}, \ref{defi:PWL-Borel-func}, \ref{defi:s-CON-u1-PWL}-\ref{defi:con3-solution}. } \\ \hline
    \makecell[l]{\hyperref[defi:PWC-Borel-func]{\textbf{2:} PWC Borel measureable}} & \makecell[l]{Provides measurable finite representations. \\ Used in Definitions~\ref{defi:s-CON-u1-PWL}-\ref{defi:con3-solution}; Lemmas~\ref{lema:perception-BFCP}, \ref{lema:reward-BFCP}, \ref{lema:evaluation-consistency-agent-1}-\ref{lema:improvement-consistency-agent-2}, \ref{pwc-strat-lem} \\ Theorems~\ref{thom:B-PWC-closure-VI}-\ref{convergence-thm}.} \\ \hline
    \makecell[l]{\hyperref[defi:PWL-Borel-func]{\textbf{3:} PWL Borel measureable}}  & \makecell[l]{Provides measurable finite representations. \\ Used in Definitions~\ref{defi:s-CON-u1-PWL}, \ref{defi:con1-solution}; Lemmas~\ref{lema:evaluation-consistency-agent-1}-\ref{lema:improvement-consistency-agent-2}; \\ Theorems~\ref{thom:repsentation-consistency}, \ref{convergence-thm}.} \\ 
    \hline
    \makecell[l]{\hyperref[defi:BFCP_invertible]{\textbf{4:} BFCP invertible}}  & \makecell[l]{Preserves the BFCP representation. \\ Used in Assumption~\ref{asp:continuous-PWC}}\\ \hline
    \makecell[l]{\hyperref[csgs-def]{\textbf{5:} CSGs}} & \makecell[l]{Recalls the definition of discrete CSGs.} \\ \hline \hline
    
    \multicolumn{2}{|l|}{\hyperref[nscsgs-sect]{\cellcolor{gray!20}\textbf{Section 3: NS-CSGss}}} \\ \hline
    \makecell[l]{\hyperref[defi:NS-CSG]{\textbf{6:} NS-CSGs}} & \makecell[l]{Introduces our new neuro-symbolic game model.}  \\ \hline 
    \makecell[l]{\hyperref[semantics-def]{\textbf{7:} Semantics of an NS-CSG}} & \makecell[l]{A formal semantic definition for the model.} \\ \hline
    \makecell[l]{\hyperref[defi:stationary-strategy]{\textbf{8:} Strategy}} & \makecell[l]{Defines the stationary strategy for each agent.} \\ \hline \hline

    \multicolumn{2}{|l|}{\hyperref[values-sect]{\textbf{\cellcolor{gray!20}Section 5: Values of zero-sum NS-CSGs}}} \\ \hline
    \makecell[l]{\hyperref[defi:value_function]{\textbf{9:} Value function}} & \makecell[l]{Formal definition of the game's value. \\ Used in Theorems~\ref{thom:value-function}, \ref{thom:B-PWC-closure-VI}, \ref{convergence-thm}; Proposition~\ref{prop:convergence-seq}.} \\ \hline
    \makecell[l]{\hyperref[defi:minimax-operator]{\textbf{10:} Minimax operator}} & \makecell[l]{Characterises and compute the value function. \\ Used in Theorems~\ref{thom:value-function}, \ref{thom:B-PWC-closure-VI}; Proposition~\ref{prop:convergence-seq}.} \\ \hline \hline

    \multicolumn{2}{|l|}{\hyperref[pi-sect]{\cellcolor{gray!20}\textbf{Section 7: Policy iteration}}} \\ \hline
    \makecell[l]{\hyperref[def:operator-H1-definition]{\textbf{11:} Operator for the Max-min value}} & \makecell[l]{Induces max-min policy evaluation of maximizer. \\ Used in Lemma~\ref{lema:evaluation-consistency-agent-1}; Theorem~\ref{thom:repsentation-consistency}; Corollary~\ref{cor:bfcp_minimax_pi}.} \\ \hline
    \makecell[l]{\hyperref[def:operator-H2-definition]{\textbf{12:} Operator for the Min-max value}} & \makecell[l]{Induces min-max policy evaluation of minimizer. \\ Used in Lemma~\ref{lema:evaluation-consistency-agent-2}; Theorem~\ref{thom:repsentation-consistency}; Corollary~\ref{cor:bfcp_minimax_pi}.} \\ \hline
    \makecell[l]{\hyperref[defi:s-CON-u1-PWL]{\textbf{13:} CON-PWL Borel measurable function}} & \makecell[l]{Finite representations for minimizer's functions. \\ Used in Definitions~\ref{defi:con1-solution}; Lemmas~\ref{lema:evaluation-consistency-agent-1}-\ref{lema:improvement-consistency-agent-2}; \\ Theorem~\ref{thom:repsentation-consistency}; Corollary~\ref{cor:bfcp_minimax_pi}.} \\ \hline
    \makecell[l]{\hyperref[defi:s-CON-u1-PWC]{\textbf{14:} CON-PWC stochastic kernel}} & \makecell[l]{Finite representation for minimizer's strategies. \\ Used in Lemmas~\ref{lema:evaluation-consistency-agent-2}, \ref{lema:improvement-consistency-agent-2}; Theorem~\ref{thom:repsentation-consistency}; Corollary~\ref{cor:bfcp_minimax_pi}.} \\ \hline
    \makecell[l]{\hyperref[defi:con1-solution]{\textbf{15:} CON-1 solution}} & \makecell[l]{Induces policy improvement of maximizer. \\ Used in Lemma~\ref{lema:improvement-consistency-agent-1}; Theorem~\ref{thom:repsentation-consistency}; Corollary~\ref{cor:bfcp_minimax_pi}.} \\ \hline
    \makecell[l]{\hyperref[defi:con2-solution]{\textbf{16:} CON-2 solution}} & \makecell[l]{Induces policy improvement of minimizer. \\ Used in Lemma~\ref{lema:improvement-consistency-agent-2}; Theorem~\ref{thom:repsentation-consistency}; Corollary~\ref{cor:bfcp_minimax_pi}.} \\ \hline
    \makecell[l]{\hyperref[defi:con3-solution]{\textbf{17:} CON-3 solution}} & \makecell[l]{Extracts PWC strategies from value functions \\ Used in Lemma~\ref{pwc-strat-lem}.} \\ \hline
    \end{tabular}
    \caption{Listing of definitions (by section) and their roles.}
    \label{tab:defs}
\end{table}

%% file: background.tex
\section{Background}\label{background-sect}


In this section we summarise the background notation, definitions and concepts used in this paper. {\revise Our results crucially rely on Borel measurability over the continuous environment and piecewise constant functions as a representation of value functions. The latter is inspired by the observation that a neural perception mechanism such as an NN classifier is a piecewise constant function of the environment and results in a finite decomposition of the environment into 
regions.} 

\subsection{Borel measurable spaces and functions} 
Given a non-empty set $X$, we denote its Borel $\sigma$-algebra by $\mathcal{B}(X)$, and the sets in $\mathcal{B}(X)$ are called \emph{Borel sets} of $X$. The pair $(X,\mathcal{B}(X))$ is a (standard) \emph{Borel space} if there exists a metric on $X$ that makes it a complete separable metric space (unless required for clarity, $\mathcal{B}(X)$ will be omitted). For convenience we will work with real vector spaces; however, this is not essential and any complete separable metric spaces could be used. For Borel spaces $X$ and $Y$, a function $f:X\to Y$ is \emph{Borel measurable} if $f^{-1}(B) \in \mathcal{B}(X)$ for all $B \in \mathcal{B}(Y)$ and \emph{bimeasurable} if it is Borel measurable and $f(B) \in \mathcal{B}(Y)$ for all $B \in \mathcal{B}(X)$.



We denote by $\mathbb{F}(X)$ the space of all bounded, Borel measurable real-valued functions on a Borel space $X$, with respect to the unweighted sup-norm $\|J\|=\sup_{x \in X}|J(x)|$ for $J \in \mathbb{F}(X)$. For functions $J,K \in \mathbb{F}(X)$, we use $\max[J,K]$ and $\min[J,K]$ to denote the respective pointwise maximum and minimum functions of $J$ and $K$, i.e.,  we have $\opt[J,K](x) \coloneqq \opt\{J(x), K(x)\}$ for $\opt  \in \{ \min,\max\}$ and $x \in X$.
%

%

We now introduce notation and definitions 
for concepts that are fundamental to the abstraction on which our algorithms are performed.
The abstraction is based on a decomposition of the uncountable state space into 
finitely many abstract regions. 
In the definitions below, let $X \subseteq \mathbb{R}^{n_1}$ and $Y \subseteq \mathbb{R}^{n_2}$ for $n_1,n_2 \in \mathbb{N}$. 
\begin{defi}[FCP and Borel FCP]\label{defi:FCP_Borel_FCP}
A \emph{finite connected partition (FCP)} of $X$, denoted $\Phi$, is a finite collection of disjoint connected subsets \emph{(regions)} that cover $X$. 
Furthermore, $\Phi$ is a \emph{Borel FCP} (BFCP) if each region $\phi \in \Phi$ is a Borel set of $X$.
\end{defi}

%
%
\begin{defi}[PWC Borel measurable]\label{defi:PWC-Borel-func}
A function $f:X\to Y$ is \emph{piecewise constant Borel measurable (B-PWC)} if there exists a BFCP $\Phi$ of $X$ such that $f : \phi \to Y$ is constant for all $\phi \in \Phi$ and $\Phi$ is called a constant-BFCP of $X$ for $f$. 
\end{defi} 
%
\begin{defi}[PWL Borel measurable]\label{defi:PWL-Borel-func}
A function $f:X\to Y$ is  \emph{piecewise linear Borel measurable (B-PWL)} if there exists a BFCP $\Phi$ of $X$ such that  $f : \phi \to Y$ is linear and bounded for all $\phi \in \Phi$.
\end{defi} 
\begin{defi}[BFCP invertible]\label{defi:BFCP_invertible}
A function $f:X\to Y$ is \emph{BFCP invertible} if, for any BFCP $\Phi_Y$ of $Y$, there exists a BFCP $\Phi_X$ of $X$, called a preimage BFCP of $\Phi_Y$ for $f$, such that for any $\phi_X \in \Phi_X$ we have $\{ f(x) \mid x \in \phi_X \} \subseteq \phi_Y$ for some $\phi_Y \in \Phi_Y$.
\end{defi}
For BFCPs $\Phi_1$ and $\Phi_2$ of $X$, we denote by $\Phi_1+\Phi_2$ the smallest BFCP of $X$ such that $\Phi_1+\Phi_2$ is a refinement of both $\Phi_1$ and $\Phi_2$, which can be obtained by taking all the intersections between regions of $\Phi_1$ and $\Phi_2$.


%

\subsection{Probability measures} 
Let $X$ be a Borel space. A function $f : \mathcal{B}(X) \to [0,1]$ is a probability measure on $X$ if $f(X)=1$ and $\sum_{i \in I} f(B_i) = f ( \cup_{i\in I} B_i )$ for any countable disjoint family of Borel sets $(B_i)_{i \in I}$. We denote the space of all probability measures on a Borel space $X$ by $\mathbb{P}(X)$.  For Borel spaces $X$ and $Y$, a Borel measurable function $\sigma:Y\to\mathbb{P}(X)$ is called a \emph{stochastic kernel} on $X$ given $Y$ (also known as a transition probability function from $Y$ to $X$), and we denote by $\mathbb{P}(X \mid Y)$ the set of all stochastic kernels on $X$ given $Y$. If $\sigma \in \mathbb{P}(X \mid Y)$,  $y \in Y$ and $B \in \mathcal{B}(X)$, then we write $\sigma(B \mid y)$ for $\sigma(y)(B)$. It follows that $\sigma \in \mathbb{P}(X \mid Y)$ if and only if $\sigma( \; \cdot \mid y) \in \mathbb{P}(X)$ for all $y \in Y$ and $\sigma(B \mid \cdot \;)$ is Borel measurable for all $B \in \mathcal{B}(X)$. 

\subsection{Neural networks}\label{nn-sect}
A \emph{neural network (NN)} is a real vector-valued function $f:\mathbb{R}^m \to \mathbb{R}^c$, where $m,c \in \mathbb{N}$, composed of a sequence of \emph{layers} $h_1,\dots,h_k$,  where $h_i :\mathbb{R}^{m_i} \to \mathbb{R}^{c_i}$ for $1 \leq i \leq k$, $m_1=m$, $c_{i} = m_{i+1}$ for $1 \leq i \leq k-1$ and $c_k= c$. Each layer $h_i$ is a data-processing module explicitly formulated as $h_i(x_i)=act_i(W_i x_i+b_i)$, where $x_i$ is the input to the $i$th layer given by the output $h_{i-1}(x_{i-1})$ of the $(i - 1)$th layer, $act_i$ is an activation function, and $W_ix_i+b_i$ is a weighted sum of $x_i$ for a weight matrix $W_i$ and a bias vector $b_i$. An NN $f$ is continuous for all popular activation functions, e.g., Rectified Linear Unit (ReLU), Sigmoid and Softmax \cite{SS-SS-AA:17}. 

An NN $f$ is said to be a \emph{classifier} for a set of classes $C$ of size $c$ if, for any input $x \in \mathbb{R}^m$, the output $f(x) \in \mathbb{R}^c$ is a probability vector where the $i$th element of $f(x)$ represents the confidence probability of the $i$th class of $C$, i.e., a classifier is a function $f : \mathbb{R}^m \rightarrow \mathbb{P}(C)$. {\revise Let $f^{\max} : \mathbb{R}^m \rightarrow C$ denote a function that returns the class with the largest confidence probability in $f(x)$, and call $f^{\max}(x)$ the \emph{class} of $x$. To allow for situations where the class with the highest probability returned by $f$ is not unique, and hence $f^{\max}(x)$ would be undefined, we assume the classifier includes
a \emph{tie-breaking rule} defined by a function $\kappa :2^{C} \to C$ which, given a set of classes, i.e., those with the highest probability, returns the selected class.

Given an NN classifier $f$ with the tie-breaking rule $\kappa$, the \emph{preimage} of $f$ divides $\mathbb{R}^m$ into a BFCP $\Phi$ of $\mathbb{R}^m$, i.e., for any $\phi \in \Phi$, there exists a class $y$ such that $f^{\max}(x)=y$ for all $x \in \phi$.
The preimage of an NN classifier with PWL activation functions can be computed via polyhedral decomposition as described in \cite{KM-FF:20}.} 

\subsection{Concurrent stochastic games} Finally, in this section, we recall the model of two-player \emph{concurrent stochastic games}.


\begin{defi}\label{csgs-def}
A (two-player) \emph{concurrent stochastic game} (CSG) is a tuple
$\game = (N, S, A, \Delta, \delta)$ where:
\begin{itemize}
\item $N=\{1,2\}$ is a set of two players;
\item $S$ is a finite set of states;
\item $A = (A_1\cup\{\bot\}) {\times} (A_2\cup\{\bot\})$ where $A_i$ is a finite set of actions available to player $i \in N$ and $\bot$ is an idle action disjoint from the set $A_1 \cup A_2$;
\item $\Delta \colon S \rightarrow 2^{(A_1 \cup A_2)}$ is an action available
function;
\item $\delta \colon (S {\times} A) \rightarrow \mathbb{P}(S)$ is a probabilistic transition function.
\end{itemize}
\end{defi}
In a state $s$ of a CSG $\game$, each player $i \in N$ selects an action from its available actions, i.e., from the set $\Delta(s) \cap A_i$,
if this set is non-empty, and selects the idle action $\bot$ otherwise.
We denote the action choices for each player $i$ in state $s$ by $A_i(s)$, 
i.e., $A_i(s)$ equals $\Delta(s) \cap A_i$ if $\Delta(s) \cap A_i \neq \emptyset$ and equals $\{ \bot \}$ otherwise,  and by $A(s)$ the possible joint actions in a state, i.e., $A(s)= A_1(s) \times A_2(s)$.
Supposing each player $i$ chooses action $a_i$, then with probability $\delta(s, (a_1,a_2))(s')$ there is a transition to state $s' \in S$. A \emph{path} $\pi$ of $\mathsf{G}$ is a sequence $\pi=s_{0}\xrightarrow{\alpha_{0}}s_1\xrightarrow{\alpha_{1}}
\cdots$ such that $s_k
\in S$, $\alpha_k\in A(s_k)$ and $\delta(s_k,\alpha_k)(s_{k+1})>0$ for all $k \geq 0$.
We let $\fpaths_{\game}$ and $\ipaths_{\game}$ denote the sets of finite and infinite paths of $\game$, respectively.
For a path $\pi$, we denote by $\pi(k)$ the $(k+1)$th state,
and $\pi[k]$ the action for the transition from $\pi(k)$ to $\pi(k+1)$.


A \emph{strategy} for a player of a CSG $\game$ resolves its action choices in each state. These choices can depend on the history of the CSG's execution and can be randomised. Formally, a strategy for player $i$ is a function $\sigma_i : \fpaths_{\game} \rightarrow \mathbb{P}(A_i \cup \{ \bot \})$
mapping finite paths to distributions over available actions,
such that, if $\sigma_i(\pi)(a_i){>}0$, then $a_i \in A_i(\mathit{last}(\pi))$ where $\mathit{last}(\pi)$ is the final state of $\pi$. A strategy is said to be \emph{stationary} if it makes the same choices for paths that end in the same state.
Furthermore, a \emph{strategy profile} of $\game$ is a pair $\sigma = (\sigma_1,\sigma_2)$ of strategies for each player. Given a strategy profile $\sigma$ and state $s$, letting 
$\ipaths^\sigma_{s}$ denote the set of
infinite paths from $s$ under the choices of $\sigma$, we can define a probability measure $\mathit{Prob}^{\sigma}_{s} \in \mathbb{P}(\ipaths^{\sigma}_{s})$~\cite{KSK76}.

%% file: nscsgs.tex
\section{Zero-sum neuro-symbolic concurrent stochastic games}\label{nscsgs-sect}

This section introduces our model of \emph{neuro-symbolic concurrent stochastic games} (NS-CSGs). We restrict our attention to two-agent (which we also refer to as two-player) games, as we are concerned with zero-sum games, in which there are two agents with directly opposing objectives. However, the approach extends to multi-agent games, by allowing the agents to form two coalitions with directly opposing objectives.
{\revise Our results depend on the uniqueness of value functions, and therefore do not extend 
to more than two 
coalitions.}

A (two-agent) NS-CSG comprises two interacting \emph{neuro-symbolic agents} acting in a \emph{shared, continuous-state environment}. Each agent $\agent_i$ has finitely many local states and actions; it 
{\revise 
observes the environment through a (trained) \emph{neural perception} mechanism (the perception function $\obs_i$), which depends on the local states of both agents, and relies on \emph{symbolic} decision-making mechanisms (the transition function $\delta_i$). During execution, the agents alternate between invoking perception and symbolic decisions, where the interface between them is suitably constrained to enable symbolic reasoning with the (exactly) learnt concepts (regions of the continuous inputs space), which we call \emph{percepts} to distinguish them from local states. When invoking perception, continuous inputs are converted into symbolic percepts, and the agent’s transition  to the next local state is based on the current local state and percept, rather than the environment state, and can thus model knowledge acquisition from the neural perception mechanism.

}


%
%
%
\begin{defi}\label{defi:NS-CSG}
A (two-agent) neuro-symbolic concurrent stochastic game (NS-CSG)\ $\csg$ comprises
agents $(\agent_i)_{i \in N}$ for $N=\{1, 2\}$ and environment $E$ where: $\agent_i  = (S_i,A_i,\Delta_i,\obs_i,\delta_i)$, $E = (S_E,\delta_E)$ and we have:
\begin{itemize}
    \item $S_i = \Loc_i\times \Per_i$ is a set of states for $\agent_i$, and {\revise $\Loc_i$ and $\Per_i$}
    are finite sets of local states and percepts,
    respectively; 
    
    \item $S_E\subseteq \mathbb{R}^e$ for $e \in \mathbb{N}$ is a closed uncountable set of environment states;  
  
    \item $A_i$ is a nonempty finite set of actions for $\agent_i$,
    and $A \coloneqq (A_1 \cup \{ \bot \}) \times (A_2 \cup \{ \bot \})$ is the set of joint actions,
    where $\bot$ is an idle action disjoint from $A_1 \cup A_2$; 

    \item $\Delta_i: S_i \to 2^{A_i}$ is an available action function for $\agent_i$, defining the actions the agent can take in each of its states;

    \item $\obs_i : (\Loc_1 \times \Loc_2 \times S_E)\to \Per_i$ is a perception function for $\agent_i$, mapping the local states of the agents and environment state to a percept of the agent;

    
    \item $\delta_i: (S_i \times A) \to \mathbb{P}(\Loc_i)$ is a 
    probabilistic transition function for $\agent_i$ determining the distribution over the agent's local states given its current state and joint action;
    
    \item $\delta_E: (S_E \times A) \to S_E$ is a 
    deterministic transition function for the environment determining its next state given its current state and joint action.
\end{itemize}
\end{defi}
\begin{rek}\label{rek:perception-via-NN}
{\revise 
We restrict the range of observation functions induced by perception mechanisms to \emph{finite} sets, which admits a wide class of functions, including ReLU neural network classifiers. 
%
For an example instantiation of an observation function $\obs_i$, for each pair of local states $(\loc_1, \loc_2) \in \Loc_1 \times \Loc_2$ we can associate an NN classifier $f_{\loc_1, \loc_2} : S_E \to \mathbb{P}(Per_i)$, see Section~\ref{nn-sect}, such that 
$\obs_i(\loc_1, \loc_2, s_E) = f^{\max}_{\loc_1, \loc_2}(s_E)$ for all $s_E \in S_E$. 
%
These NNs need not be distinct for different pairs of local agent states, but we have allowed this modelling choice to reflect the design of existing NN-enabled systems that we target, e.g., in the VCAS collision avoidance system \cite{KDJ-MJK:19} there are nine NN advisories and which of these advisories is selected is based on the local state of the system. 
%
Assuming the range of the perception functions is finite, our framework is also capable of modelling perception mechanisms that input (continuous) numerical sensor information (e.g., via output discretisation) and other machine learning models (which may require additional assumptions to ensure finite representability of percepts). 
}
\end{rek}

%
%
%
In an NS-CSG $\csg$ the agents and environment execute concurrently and
agents move between their local states probabilistically.
For simplicity, we consider deterministic environments,
but all the results extend directly to probabilistic environments with finite branching. 

A (global) state of an NS-CSG comprises a state $s_i = (\loc_i, \per_i)$ for each agent $\agent_i$ (a local-state-percept pair) and an environment state $s_E$. A state $s = ((\loc_1, \per_1),(\loc_2, \per_2), s_E)$ is \emph{percept compatible} if $\per_i = \obs_i(\loc_1, \loc_2, s_E)$ for $1 \leq i \leq 2$. {\revise Percept compatibility indicates that each agent always accesses its percept via the perception function,
and observes the environment state only through the perception function.}
In state $s=(s_1,s_2,s_E)$, each $\agent_i$ simultaneously chooses one of the actions available in its state $s_i$ (if no action is available, i.e., $\Delta_i(s_i)=\emptyset$, then $\agent_i$ chooses the idle action $\bot$), resulting in a joint action $\alpha=(a_1,a_2)\in A$. Next, each $\agent_i$ updates its local state to some $\loc_i'\in \Loc_i$,
according to the distribution $\delta_i(s_i,\alpha)$. At the same time, the environment updates its state to some $s_E'\in S_E$ according to the transition $\delta_E(s_E,\alpha)$. 
Finally, each $\agent_i$, based on its new local state, observes the new local state of the other agent and the new environment state to generate a new percept $\per_i' = \obs_i(\loc'_1,\loc'_2,s'_E)$.
Thus, the game reaches the state $s'=(s_1',s_2', s_E')$, where $s_i' = (\loc_i', \per_i')$ for $1 \leq i \leq 2$. 

\begin{examp}\label{eg1-eg}
As an illustration, we present an NS-CSG model
of a dynamic vehicle parking problem (a static version is presented in~\cite{DA-OW-BX-BD-JL:11}). Fig.~\ref{fig:vehicle-parking-setup} (left) shows two agents, $\agent_1$ (the red vehicle)  and $\agent_2$ (the blue vehicle), in a (continuous) environment $\mathcal{R} = \{(x, y) \in \mathbb{R}^2 \mid 0 \leq x, y \leq 4\}$ and two preferred parking spots $\mathit{ps}_1, \mathit{ps}_2 \in \mathcal{R}$ (the green circles),  which are known to the agents. {\revise This example employs a  classifier $f_\mathcal{R} : \mathcal{R} \rightarrow \mathbb{P}(\mathit{Grid})$, where $\mathit{Grid} = \{(i,j)\mid i,j \in \{1, 2, 3, 4\}\}$, which takes the coordinates of a vehicle (or parking spot) as input and outputs a probability distribution over 16 abstract grid cells, see Fig.~\ref{fig:vehicle-parking-setup} (centre).} 

The actions of the agents are to move either up, down, left or right, or park. The vehicles of the agents start from different positions in $\mathcal{R}$ and have the same speed. 
The red agent initially chooses one
parking spot and changes its parking spot with probability 0.5 when the blue agent is observed to be closer to its chosen parking spot and both agents move towards this spot, see Fig.~\ref{fig:vehicle-parking-setup} (centre and right).
Formally, the agents and the environment are defined as follows.


\input{figure/vehicle-alternative}





\begin{itemize}
    \item  $\Loc_1 = \{ \mathit{ps}_1, \mathit{ps}_2 \}$ and $\Loc_2 = \{ \bot \}$, 
    i.e., the local state of $\agent_1$ is its current chosen parking spot and the local state of $\agent_2$ is a dummy state. For $1 \leq i \leq 2$, the set of percepts of $\agent_i$ is given by $\Per_i= \mathit{Grid} \times \mathit{Grid}$, representing the abstract grid cells that each agent perceives as the positions of the two vehicles. 
 
    \item $S_E = \mathcal{R} \times \mathcal{R}$, i.e., the environment is in state $s_E=(w_1,w_2)$ if $w_i$ is the continuous coordinate of $\agent_i$'s vehicle for $1 \leq i \leq 2$. 
 
     \item 
     $A_i = \{ \mathit{up}, \mathit{down}, \mathit{left}, \mathit{right}, \mathit{park} \}$ for $1 \leq i \leq 2$.

     \item For $1 \leq i \leq 2$, local state $\loc_i$ and  {\revise percept $\per_i = (\mathit{cell}_1,\mathit{cell}_2)$, where the perceived grid cell of agent $\agent_j$ is $\mathit{cell}_j$, we let $\Delta_i(\loc_i,\per_i)$ equal $A_i$ if $\mathit{cell}_i \in \{ f_\mathcal{R}^{\max}(\mathit{ps}_1) , f_\mathcal{R}^{\max}(\mathit{ps}_2) \}$} and equal $A_i \setminus \{ \mathit{park} \}$ otherwise, i.e., an agent's available actions are to move up, down, left and right, and additionally park when the agent is perceived to have reached a parking spot.
     

    

    \item {\revise 
    For $1 \leq i \leq 2$, local states $\loc_1$ and $\loc_2$ and environment state $(w_1,w_2)$, we let $\obs_i(\loc_1,\loc_2,(w_1,w_2)) = (f_\mathcal{R}^{\max}(w_1),f_\mathcal{R}^{\max}(w_2))$, i.e., independently of the local states of the agents, the perception function returns the perceived grid cell of each agent under the classifier $f_\mathcal{R}$.}

        \item For any $\agent_1$ state $s_1=(\loc_1,(\mathit{cell}_1,\mathit{cell}_2))$ and joint action $\alpha$, 
        to define $\delta_1$ we have the following two cases to consider:
        \begin{itemize}
        \item
        if $\| {\revise f_\mathcal{R}^{\max}(\loc_1)} - \mathit{cell}_1 \|_2 > \| {\revise f_\mathcal{R}^{\max}(\loc_1)} - \mathit{cell}_2 \|_2$, where $\|\cdot\|_2$ is the Euclidean norm, i.e. $\agent_1$ observes $\agent_2$ is closer to its currently chosen parking spot, and the joint action $\alpha$ indicates both agents are approaching $\loc_1$, then $\delta_1(s_1, \alpha)(\mathit{ps}_j) = 0.5$ for $1 \leq j \leq 2$, i.e., $\agent_1$ changes its chosen parking spot with probability 0.5;
        \item
        otherwise $\delta_1(s_1, \alpha)(\loc_1) = 1$, i.e., $\agent_1$ sticks with its chosen parking spot. 
        \end{itemize}
         Considering $\delta_2$, since $\Loc_2 =\{\bot \}$, we have $\delta_2(s_2,\alpha)(\bot) = 1$ for any $\agent_2$ state $s_2=(\bot,(\mathit{cell}_1,\mathit{cell}_2)) \in S_2$ and joint action $\alpha$.       

\item For any environment state $(w_1,w_2)$ and joint action  $\alpha = (a_1,a_2)$, we let $\delta_E((w_1,w_2),\alpha) = (w_1',w_2')$ where,  for $1 \leq i \leq 2$, we have $ w_i' = w_i+ d_{a_i} \Delta t $ if $(w_i+ d_{a_i} \Delta t) \in \mathcal{R}$ and $ w_i' = w_i$ otherwise,
and $d_{a_i}$ is the direction of movement of the action $a_i$, e.g., $d_{\mathit{up}}=(0,1)$, and $\Delta t = 0.5$ is the time step.  \hfill$\blacksquare$
\end{itemize} 
\end{examp}

\subsection{Semantics of an NS-CSG}\label{sub:semantics-ns-csg}
The semantics of an NS-CSG $\csg$ is a CSG 
$\sem{\csg}$ over the product of the states of the agents and the environment formally defined as follows.

\begin{defi}[Semantics of an NS-CSG]\label{semantics-def}
Given an NS-CSG $\csg$ consisting of two agents and an environment, its semantics is the CSG $\sem{\csg} = (N,S,A,\Delta,\delta)$ where:
\begin{itemize}
    \item $S \subseteq S_1\times S_2\times S_E$ is the set of percept compatible states;
\item $A = (A_1 \cup \{ \bot \}) \times (A_2 \cup \{ \bot \})$;
\item $\Delta(s_1,s_2,s_E) = \Delta_1(s_1) \cup \Delta_2(s_2)$;
    \item $\delta: (S \times ((A_1 \cup \{\bot \}) \times (A_2  \cup \{\bot \}))) \to \mathbb{P}(S)$ is the 
    probabilistic transition function, 
    where for states $s=(s_1,s_2,s_E),s'=(s_1',s_2',s_E')\in S$ and joint action $\alpha=(a_1,a_2)\in A$, if $a_i \in \Delta_i(s_i)$ when $\Delta_i(s_i) \neq \emptyset$ and $a_i= \bot$ otherwise for $1 \leq i \leq 2$, then $\delta(s,\alpha)$ is defined and, if $s_i'=(\loc_i',\per_i')$, $\per_i'=\obs_i(\loc_1',\loc_2',s_E')$ for $1 \leq i \leq 2$  and $s_E'=\delta_E(s_E,\alpha)$,  then
    \[    \delta(s,\alpha)(s')= \delta_1(s_1, \alpha)(\loc'_1)  \delta_2(s_2, \alpha)(\loc'_2)
    \]
    and otherwise $\delta(s,\alpha)(s')=0$.
\end{itemize}
\end{defi}
Notice that the CSG $\sem{\csg}$ is over percept compatible states
and that, by definition of $\obs_i$ for each agent $\agent_i$, the underlying transition relation $\delta$ is closed with respect to percept compatible states.
Since $\delta_E$ is deterministic and $\Loc_i$ is a finite set, the set of successors of $s$ under $\alpha$, denoted $\Theta_s^\alpha=\{s'  \mid \delta(s,\alpha)(s')>0\}$, is finite for all $s\in S$ and $\alpha\in A(s)$. 
While the semantics of an NS-CSG is an instance of the general class of uncountable state space CSGs,
its particular structure induced by perception functions (see \defiref{defi:NS-CSG}) will be important
in order to establish measurability and finite representability to allow us to derive 
our algorithms.


\subsection{Zero-sum NS-CSGs}\label{sub:zero-sum-ns-csgs}
%
For an NS-CSG $\csg$, the objectives we consider are \emph{discounted accumulated rewards}, and we assume the first agent tries to maximise the expected value of this objective and the second tries to minimise it.
More precisely, for a reward structure 
$r=(r_A,r_S)$, where $r_A: (S\times A) \to \mathbb{R}$ is an action reward function and $r_S:S\to \mathbb{R}$ is a state reward function, and discount factor $\beta\in(0,1)$, the accumulated discounted reward for a path $\pi$ of $\sem{\csg}$ over the infinite-horizon is defined by:
\begin{equation}\label{eq:reward-over-path}
    Y(\pi)= \mbox{$\sum_{k=0}^{\infty}$} \, \beta^k\left(r_A(\pi(k),\pi[k])+r_S(\pi(k))\right) \, .
\end{equation}

\begin{examp}\label{eg2-eg}
Returning to the dynamic vehicle parking model of \egref{eg1-eg}, we suppose the objective for $\agent_1$ is to try and park at its {\revise currently preferred} parking spot without crashing into $\agent_2$ and, since we consider zero-sum NS-CSGs whose objectives must be directly opposing, the objective of $\agent_2$ is to try to crash into $\agent_1$ and prevent it from parking.
We can represent this scenario using a discounted 
reward structure, where all action rewards are zero and for the state rewards we set: there is a negative reward if it is perceived that $\agent_1$ has yet to reach its current parking spot and the agents have crashed; a positive reward if it is observed that $\agent_1$ has reached its parking spot,  which is higher
if the agents are not perceived to have crashed; and 0 otherwise. 

Formally, for $s= (s_1,s_2,(w_1,w_2))$ where $s_1 =(\loc_1,(\mathit{cell}_1,\mathit{cell}_2)))$, we define the state reward function as follows:
\[
r_S(s) = \left\{ \begin{array}{cl}
- 1000 & \mbox{if $\mathit{cell}_1 \neq {\revise f_\mathcal{R}^{\max}(\loc_1)}$ and $\mathit{cell}_1 = \mathit{cell}_2$} \\
500 & \mbox{if $\mathit{cell}_1 = {\revise f_\mathcal{R}^{\max}(\loc_1)}$ and $\mathit{cell}_1 = \mathit{cell}_2$} \\
1000 & \mbox{if $\mathit{cell}_1 = {\revise f_\mathcal{R}^{\max}(\loc_1)}$ and $\mathit{cell}_1 \neq \mathit{cell}_2$} \\
0 & \mbox{otherwise.}
\end{array} \right.
\]
For the discount factor, we let $\beta = 0.6$. \hfill$\blacksquare$
\end{examp}

\subsection{Strategies of NS-CSGs}
%
Since the state space $S$ is uncountable due to the continuous environment state space,
we follow the approach of \cite{PRK-THS:81} and require Borel measurable conditions on the choices that the strategies can make to ensure the measurability of the induced sets of paths.




The semantics of any NS-CSG will turn out to be an instance of the class of CSGs from \cite{PRK-THS:81}, for which \emph{stationary strategies} achieve optimal values \cite[Theorem 2(ii), Theorem 3]{PRK-THS:81}, and therefore, to simplify the presentation, we restrict our attention to stationary strategies and refer to them simply as strategies. Before we give their formal definition, since we work with real vector spaces we require the following lemma.

\begin{lema}[Borel spaces]\label{lema:borel-space-S-A}
The sets $S$, $S_i$, $S_E$ and $A_i$ for $1 \leq i \leq 2$ are Borel spaces.
\end{lema}
\begin{proof}
By Theorem 27 \cite[Chapter 9.6]{HLR-PF:10} and Theorem 12 \cite[Chapter 9.4]{HLR-PF:10}, $S_1$ and $S_2$ are finite and $S_E$ are complete separable metric spaces, and hence are Borel spaces. Furthermore, we have that $S_1 \times S_2 \times S_E$ is the Cartesian product of Borel spaces, and therefore, using Theorem 1.10 \cite[Chapter 1]{KRP:67}, is also a Borel space. 
Since we assume $\obs_i$ is Borel measurable for $1 \leq i \leq 2$ (see \aspref{asp:continuous-PWC} below), for $(\loc_i, \per_i) \in S_i$ and $1 \leq i \leq 2$, the set:
\[\{((\loc_1, \per_1), (\loc_2, \per_2), s_E) \in S \mid \obs_i( \loc_1, \loc_2, s_E) = \per_i \mbox{ for } 1 \leq i \leq 2  \}
\]
is a Borel subset of $S_1 \times S_2 \times S_E$. Hence, since  $S_1$ and $S_2$ are finite, it follows that $S$ is a Borel space. Finally, for $1 \leq i \leq 2$, since $A_i$ is finite it is a Borel space.
\end{proof}

\begin{defi}[Strategy]\label{defi:stationary-strategy}
A \emph{(stationary) strategy} for $\agent_i$ of an NS-CSG $\csg$ is a stochastic kernel $\sigma_i:S\to\mathbb{P}(A_i)$, i.e., $\sigma_i \in \mathbb{P}(A_i \mid S)$, such that $\sigma_i(A_i(s) \mid s)=1$ for all $s\in S$. A \emph{(strategy) profile}  $\sigma=(\sigma_1, \sigma_2)$ is a pair of strategies for each agent. We denote by $\Sigma_i$ the set of all strategies of $\agent_i$ and by $\Sigma=\Sigma_1\times \Sigma_2$ the set of profiles.
\end{defi}
%
%
%
For $s \in S$ and $1 \leq i \leq 2$, we let $\mathbb{P}(A_i(s)) = \{ u_i \in \mathbb{P}(A_i) \mid u_i(A_i(s))=1 \}$. 

\subsection{Assumptions on NS-CSGs}
%
%
Finally, in this section we 
{\revise summarise and motivate}
the assumptions over NS-CSGs
that are required for the results presented in the remainder of the paper. First, NS-CSGs are designed to model neuro-symbolic agents,
whose operation depends on particular perception functions, which may result in imperfect information.
However, we assume \emph{full observability},
i.e., where agents' decisions can depend on the full state space.
It is straightforward to extend the semantics above to
\emph{partially observable} CSGs (POSGs)~\cite{Rei79,Rei84} where, for any state, each agent's observation function returns the agent's observable component of the state,  by restricting to observationally-equivalent strategies, but this comes at a significant increase in complexity. Instead, we focus on full observability, which can serve as a vehicle to solve the more complex imperfect information game via an appropriate adaptation of the belief-space construction.

Regarding the structure of NS-CSGs, we make the following assumptions to ensure determinacy and that our finite abstract representations of value functions and strategies are closed under {\revise both value and policy iteration. Recall that the BFCP, bimeasurable function, B-PWC function and BFCP intertible function are defined in \sectref{background-sect}.}
%
\begin{asp}\label{asp:continuous-PWC}
For any NS-CSG $\csg$ and reward structure $r=(r_A,r_S)$:
\begin{enumerate}[(i)]
    \item\label{itm:ass1} $\delta_E({}\cdot{},\alpha): S_E \to S_E$ is bimeasurable and BFCP invertible for $\alpha \in A$;

    \item\label{itm:ass2} $\obs_i(\loc_1, \loc_2,{}\cdot{}) : S_E \to \Per_i$ is B-PWC
    for $\loc_i \in \Loc_i$ and $1 \leq i \leq 2$; 
    

    \item\label{itm:ass3} $r_A({}\cdot{},\alpha),~r_S : S \to \mathbb{R}$ are B-PWC
    for $\alpha \in A$.
\end{enumerate}
\end{asp}
%
%
The above assumptions for NS-CSGs differ from existing stochastic games with Borel state spaces \cite{PRK-THS:81,ASN:85-2,OH-JBL:00} 
in that the states have both discrete and continuous elements, while the perception and reward functions are required to be B-PWC.
{\revise The B-PWC requirements in \aspref{asp:continuous-PWC}(\ref{itm:ass2}) and (\ref{itm:ass3}) and BFCP invertibility in \aspref{asp:continuous-PWC}(\ref{itm:ass1}) are needed to achieve 
B-PWC closure, and hence 
ensure finitely many abstract state regions at each transition (and are used in Lemmas \ref{lema:perception-BFCP}, \ref{lema:reward-BFCP}, \ref{lema:reachability-consistency} and \thomref{thom:B-PWC-closure-VI} below). 
The B-PWC closure in \thomref{thom:B-PWC-closure-VI} allows us to derive a VI algorithm in \sectref{vi-sect} that employs B-PWC functions,  whose convergence follows from the classical Banach’s fixed point theorem.}  {\revise Bimeasurability in \aspref{asp:continuous-PWC}(\ref{itm:ass1}) enables the application of the results in \cite{PRK-THS:81} and  ensures the existence of the value of an NS-CSG with respect to a reward structure in \sectref{values-sect} (and is used in \lemaref{lema:reachability-consistency}).}

In the case that, for each pair of local states of the agents, the perception function $\obs_i$ of $\agent_i$ is implemented via an NN classifier
$f:S_E\to \mathbb{P}(Per_i)$ (see \rekref{rek:perception-via-NN}),  
we have that, since $f$ is continuous, it is also Borel measurable.

\begin{examp}\label{eg3-eg}
    {\revise Returning to \egref{eg1-eg}, we now give two possible implementations for the classifier $f_{\mathcal{R}}: \mathcal{R} \rightarrow \mathbb{P}(\mathit{Grid})$ used in the example, where $\mathit{Grid} = \{(i,j)\mid i,j \in \{1, 2, 3, 4\}\}$. It takes the coordinates of a vehicle (or parking spot) as input and outputs a probability distribution over 16 abstract grid cells.
    This leads to two different implementations for perception functions for the agents, which we recall are defined by  $\obs_i(\loc_1,\loc_2,(w_1,w_2)) = (f_\mathcal{R}^{\max}(w_1),f_\mathcal{R}^{\max}(w_2))$ for $\loc_1 \in \Loc_1$, $\loc_2 \in \Loc_2$, $(w_1,w_2) \in \mathcal{R} \times \mathcal{R}$ and $1 \leq i \leq 2$.
    
    The first implementation of the classifier $f_{\mathcal{R}}$ is via the linear regression model for multi-class classification, i.e.,  $f_\mathcal{R}^{\max}(x,y) = (\lceil x \rceil, \lceil y \rceil)$ for all $(x,y) \in \mathcal{R}$, where $\lceil \cdot \rceil$ is the ceiling function, see Fig.~\ref{fig:vehicle-parking-setup} (centre).

    The second implementation of the classifier $f_{\mathcal{R}}$ is a feed-forward NN classifier, which has one hidden ReLU
layer with 10 neurons, and is trained from labelled data. We break ties using a total order over the abstract grid cells, which is Borel measurable.}
    \hfill$\blacksquare$.
\end{examp}
{\revise We remark that, while \aspref{asp:continuous-PWC}(\ref{itm:ass2}) allows a wide range of perception functions, our main focus is on NNs. We discuss the case when perception functions are instantiated using ReLU neural networks in Section~\ref{vi-sect}.}

%% file: figure/vehicle-alternative.tex
\newcommand{\grid}[2]{
    \node at (#1 + 3.95, #2 + -0.3) {$4$};
    \node at (#1 + -0.2, #2 + -0.3) {$0$};
    \node at (#1 + -0.25, #2 + 3.95) {$4$};
    \node at (#1 + -0.25, #2 + 3) {$3$};
    \node at (#1 + -0.25, #2 + 2) {$2$};
    \node at (#1 + -0.25, #2 + 1) {$1$};
    \node at (#1 + 3, #2 + -0.3) {$3$};
    \node at (#1 + 2, #2 + -0.3) {$2$};
    \node at (#1 + 1, #2 + -0.3) {$1$};
    \draw[black, very thick] (#1, #2 + 0) rectangle (#1 + 4, #2 + 4);
    \draw[black, thick] (#1, #2 + 1) -- (#1 + 4, #2 + 1);
    \draw[black, thick] (#1, #2 + 2) -- (#1 + 4, #2 + 2);
    \draw[black, thick] (#1, #2 + 3) -- (#1 + 4, #2 + 3);
    \draw[black, thick] (#1 + 1, #2 + 0) -- (#1 + 1, #2 + 4);
    \draw[black, thick] (#1 + 2, #2 + 0) -- (#1 + 2, #2 + 4);
    \draw[black, thick] (#1 + 3, #2 + 0) -- (#1 + 3, #2 + 4);
}

\begin{figure}
\centering
    \begin{tikzpicture}[scale = 0.75]
        \hspace{-0.75cm}
    
	    \draw[black, very thick] (0,0) rectangle (4,4); 
	    \node at (3.95, -0.3) {$4$};
	    \node at (-0.2, -0.3) {$0$};
	    \node at (-0.25, 3.95) {$4$};
	    \node [sedan top,body color=red!30,window color=black!80,minimum width=0.75cm, rotate = 90] at (0.5,0.75) {};
	    \node [sedan top,body color=blue!30,window color=black!80,minimum width=0.75cm, rotate = 180] at (3.25,3.5) {};
	    \filldraw[color=green!40, fill=green!30, very thick](1.5,3.5) circle (0.14);
	    \filldraw[color=green!40, fill=green!30, very thick](3.5,0.5) circle (0.14);

        \hspace{0.5cm}

        \def\xi{5}
        \def\yi{0}
        \grid{\xi}{\yi}
        \node [sedan top,body color=blue!30,window color=black!80,minimum width=0.75cm, rotate = 180] at (\xi + 3.25,3.5) {};	
        \node [sedan top,body color=red!30,window color=black!80,minimum width=0.75cm, rotate = 90] at (\xi + 0.5,0.75) {};	
        \draw[red, thick, fill = green!30, very thick] (\xi + 1, \yi + 3) rectangle (\xi + 2, \yi + 4); 
        \draw[black, thin, fill = green!30, very thick] (\xi + 3, \yi + 0) rectangle (\xi + 4, \yi + 1); 
        \draw[black, thin, fill = blue!30, opacity=0.4] (\xi + 3, \yi + 3) rectangle (\xi + 4, \yi + 4); 
        \draw[black, thin, fill = red!30, opacity=0.4] (\xi + 0, \yi + 0) rectangle (\xi + 1, \yi + 1); 
        
        \draw[very thick, blue, ->] (\xi + 3.25,3.5) -- (\xi + 2.5,3.5);
        \draw[very thick, red, ->] (\xi + 0.5,0.75) -- (\xi + 0.5,1.6);

        \draw[|-|] (\xi + 0.5,0.5) -- (\xi + 1.5, \yi + 3.5) node [midway, above, xshift=-5.0] {\scriptsize $d_1$};
        \draw[-,dashed] (\xi + 0.5,0.5) -- (\xi + 1.5,0.5);
        \draw[-,dashed] (\xi + 1.5,0.5) -- (\xi + 1.5, \yi + 3.5);

        \draw[|-|] (\xi + 3.5,3.5) -- (\xi + 1.5,3.5) node [midway, above, xshift=-2.0, yshift=-2.0] {\scriptsize $d_2$};

        \draw[->] (\xi + 4.0, 2.0) -- (\xi + 5.5, \yi + 4.0) node [midway, above, xshift=-5.0] {$\frac{1}{2}$};

        \draw[->] (\xi + 4.0, 2.0) -- (\xi + 5.5, \yi + 0.0) node [midway, below, xshift=-5.0] {$\frac{1}{2}$};

        \hspace{0.75cm}

        \def\xi{10}
        \def\yi{2.5}
        \grid{\xi}{\yi}
        \grid{10}{2.5}
        \node [sedan top,body color=blue!30,window color=black!80,minimum width=0.75cm, rotate = 180] at (\xi + 2.75,\yi + 3.5) {};	
        \node [sedan top,body color=red!30,window color=black!80,minimum width=0.75cm, rotate = 90] at (\xi + 0.5,\yi + 1.25) {};	
        \draw[red, thick, fill = green!30, very thick] (\xi + 1, \yi + 3) rectangle (\xi + 2, \yi + 4); 
        \draw[black, thin, fill = green!30, very thick] (\xi + 3, \yi + 0) rectangle (\xi + 4, \yi + 1); 
        \draw[black, thin, fill = blue!30, opacity=0.4] (\xi + 2, \yi + 3) rectangle (\xi + 3, \yi + 4); 
        \draw[black, thin, fill = red!30, opacity=0.4] (\xi + 0, \yi + 1) rectangle (\xi + 1, \yi + 2);

        \def\xi{10}
        \def\yi{-2.5}
        \grid{\xi}{\yi}
        \node [sedan top,body color=blue!30,window color=black!80,minimum width=0.75cm, rotate = 180] at (\xi + 2.75,\yi + 3.5) {};	
        \node [sedan top,body color=red!30,window color=black!80,minimum width=0.75cm, rotate = 90] at (\xi + 0.5,\yi + 1.25) {};		
        \draw[black, thin, fill = green!30, very thick] (\xi + 1, \yi + 3) rectangle (\xi + 2, \yi + 4); 
        \draw[red, thick, fill = green!30, very thick] (\xi + 3, \yi + 0) rectangle (\xi + 4, \yi + 1); 
        \draw[black, thin, fill = blue!30, opacity=0.4] (\xi + 2, \yi + 3) rectangle (\xi + 3, \yi + 4); 
        \draw[black, thin, fill = red!30, opacity=0.4] (\xi + 0, \yi + 1) rectangle (\xi + 1, \yi + 2); 
        
    \end{tikzpicture}

    \caption{Dynamic vehicle parking: continuous environment $[0,4]^2$ (left); discrete percepts corresponding the $4 {\times} 4$ abstract grid cells (middle) and probabilistic transitions following joint action $(\mathit{up}, \mathit{left}$) (right). Red vehicle's parking preference is also indicated in red. }\label{fig:vehicle-parking-setup}
\end{figure}
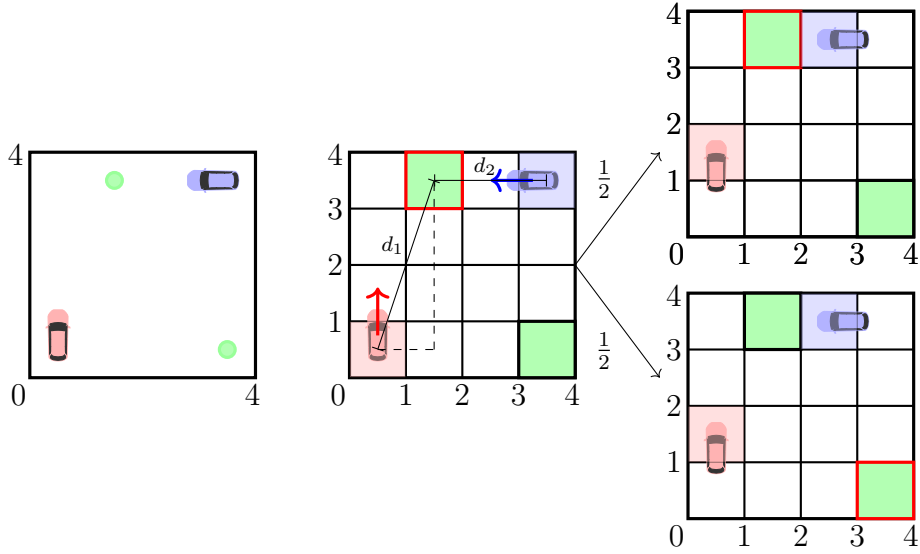

%% file: structures.tex
\section{Game structures for NS-CSGs}\label{structures-sect}





In this section, we present three finite {\revise abstract} representations for the continuous state space of an NS-CSG {\revise to enable value and policy iterations.}
These take the form of {\revise finite decompositions of the environment, represented as BFCPs,}
with respect to the perception, reward and transition functions of the NS-CSG. 
Recall, from Section~\ref{background-sect}, that a BFCP of a set is a finite family of disjoint Borel sets (regions) that cover the set. Using \aspref{asp:continuous-PWC}, we construct these BFCPs over the state space such that the states in each region are \emph{equivalent} with respect to either the perception, reward or transition function, e.g., for any region of the perception BFCP all states in the region yield \emph{the same} percept. These BFCPs allow us to abstract an uncountable 
state space into a finite set of regions when performing our VI and PI algorithms.
In particular, Sections~\ref{vi-sect} and \ref{pi-sect} demonstrate how these different BFCPs can be used, together with intersection, image and preimage operations, to iteratively refine the abstract, {\revise symbolic}
representations of the environment while maintaining the necessary conditions for correctness and convergence of value functions. 

For the remainder of this section we fix an NS-CSG $\csg$ and reward structure $r$.

\begin{lema}[Perception BFCP]\label{lema:perception-BFCP}
    There exists a smallest BFCP of $S$, called the perception BFCP, denoted $\Phi_{P}$, such that, for any $\phi \in \Phi_P$, all states in $\phi$ have the same agents' states, i.e., if $(s_1, s_2, s_E),(s_1', s_2', s_E')\in \phi$, then $s_i = s_i'$ for $1 \leq i \leq 2$.
\end{lema}
\begin{proof}
    For $1 \leq i \leq 2$, since $\obs_i$ is PWC and $S_i$ is finite, using \defiref{defi:NS-CSG} we have that, for any $s_i = (\loc_i, \per_i) \in S_i$, the set $S_E^{s_1,s_2} = \{ s_E \in S_E \mid \obs_i(\loc_1, \loc_2, s_E) = \per_i \wedge 1 \leq i \leq 2\}$ can be expressed as a number of disjoint regions of $S_E$ and we let $\Phi_E^{s_1,s_2}$ be such a representation that minimises the number of the regions. It then follows that $ \Phi_P \coloneqq \{ \{ (s_1, s_2, s_E) \mid s_E \in \phi_E \} \mid \phi_E \in \Phi_E^{s_1,s_2} \wedge s_1 \in S_1 \wedge s_2 \in S_2\}$ is a smallest FCP of $S$ such that all states in any region have the same agents' states. 

    Next we prove that $\Phi_P$ is a BFCP of $S$. We consider a region $\phi \in \Phi_P$. Thus all states in $\phi$ have the same agents' states, say $s_1 = (\loc_1, \per_1)$ and $s_2 = (\loc_2, \per_2)$. According to \aspref{asp:continuous-PWC},  $\obs_i(\loc_1,\loc_2,{}\cdot{}): S_E \rightarrow \Per_i$ for $1 \leq i \leq 2$ is B-PWC.
The preimage of $(per_1, per_2)$ under $\obs_1$ and $\obs_2$ over $S$ given $s_1 = (\loc_1, \per_1)$ and $s_2 = (\loc_2, \per_2)$, denoted $\obs^{-1}(per_1, per_2 \mid s_1,s_2)$, equals: 
\[
      \{ (s_1, s_2, s_E) \in S \mid \obs_1(\loc_1, \loc_2, s_E) = \per_1 \wedge \obs_2(\loc_1, \loc_2, s_E) = \per_2 \}
\]
and therefore is a Borel set of $S$. Since $\Phi_P$ is the smallest such partition of $S$, 
the regions in $\Phi_P$, which lead to the percept $(\per_1, \per_2)$ given $s_1$ and $s_2$, have no common boundary. Thus, $\obs^{-1}(per_1, per_2 \mid s_1,s_2)$ is a finite union of disjoint regions in $\Phi_P$, which include the agents' states $s_1$ and $s_2$. Thus, each such region is a Borel set of $S$, meaning that $\phi \in \mathcal{B}(S)$. Thus, $\Phi_P$ is a BFCP of $S$. 
\end{proof}

\begin{lema}[Reward BFCP]\label{lema:reward-BFCP}
    For each $\alpha \in A$, there exists a smallest BFCP of $S$, called the \emph{reward BFCP} of $S$ under $\alpha$ and denoted $\Phi_{R}^{\alpha}$, such that for any $\phi \in \Phi_{R}^{\alpha}$ all states in $\phi$ have the same state reward and action reward when $\alpha$ is chosen, i.e., if $s,s'\in \phi$, then $r_A(s,\alpha)=r_A(s',\alpha)$ and $r_S(s)= r_S(s')$.
\end{lema}

\begin{proof}
    For any $\alpha \in A$, since $r_A({}\cdot{},\alpha)+r_S({}\cdot{}) : S \rightarrow \mathbb{R}$ is B-PWC
     by \aspref{asp:continuous-PWC}, we can show that $\Phi_R^{\alpha}$ is a BFCP of $S$ by a similar argument to that in the proof of \lemaref{lema:perception-BFCP}.
\end{proof}
%
 %
%
%
\noindent
Using \aspref{asp:continuous-PWC}, we show that, given any joint action $\alpha$, the perception BFCP $\Phi_P$ can be refined into a new BFCP, such that the states in each region of this BFCP all reach, under the transition function of $\sem{\csg}$, the same regions of the image of $\Phi_P$ under the transition function.
This result, {\revise referred to as \emph{reachability consistency}},  will be used for the existence of the value of $\sem{\csg}$ and in our algorithms. 

\begin{lema}[Preimage BFCP]\label{lema:reachability-consistency}
For each $\alpha \in A$, there exists a refinement BFCP of $\Phi_P$, denoted $\Phi_P^\alpha$,  such that, for each $\phi \in \Phi_P^\alpha$ and $\phi' \in \Phi_P$, if $\delta(s, \alpha)$ is defined for $s \in \phi$, then there exists $p_{\alpha}(\phi, \phi') \in [0,1]$ such that:
\begin{enumerate}
\item
either $\delta(s, \alpha)(s') = p_{\alpha}(\phi, \phi') = 0$ for all $s \in \phi$ and $s' \in \phi'$; 
\item
or $(i)$ if $s, \tilde{s} \in \phi$, then there exist unique states $s', \tilde{s}' \in S$ such that 
$\{s'\} = \Theta_s^{\alpha} \cap \phi'$, $\{\tilde{s}'\} = \Theta_{\tilde{s}}^{\alpha} \cap \phi'$ and $\delta(s, \alpha)(s') = \delta(\tilde{s}, \alpha)(\tilde{s}') = p_{\alpha}(\phi, \phi') > 0$, and $(ii)$ there exists a bimeasurable, BFCP invertible function $q_{\alpha} : \phi \to \phi'$ such that $\{q_{\alpha}(s)\} = \Theta_s^{\alpha}  \cap \phi'$ and $\delta(s, \alpha)(q_\alpha(s)) = p_{\alpha}(\phi, \phi')$ for all $s \in \phi$. 
\end{enumerate}
\end{lema}
\begin{proof}
We compute the refinement $\Phi_P^\alpha$ of $\Phi_P$ by dividing each $\phi$ of $\Phi_P$ such that the required (reachability consistency) property  holds. Now, for any $\alpha\in A$ and $\phi \in \Phi_P$, by \lemaref{lema:perception-BFCP}, all states in $\phi$ have the same agents' states, say $s_1$ and $s_2$. To aid the proof, for each $\phi' \in \Phi_P$, we will construct a BFCP of $\phi$ based on $\phi'$, denoted $\Phi'(\phi, \phi')$, such that the reachability consistency to the region $\phi'$ holds in each region of $\Phi'(\phi, \phi')$. If $\delta(s, \alpha)$ is not defined for $s \in \phi$, we do not divide $\phi$ and
let $\Phi'(\phi, \phi') = \{\phi\}$ for all $\phi' \in \Phi_P$ and the reachability consistency to $\phi'$ is preserved. 

It remains to consider the case when $\delta(s, \alpha)$ is defined. Considering any $\phi' \in \Phi_P$, by \lemaref{lema:perception-BFCP} there exists agent states $s_1' = (\loc_1', \per'_1)$ and $s_2' = (\loc_2', \per'_2)$ such that if $(s_1'',s_2'',s_E'') \in \phi'$ then $s_1''=s_1'$ and $s_2''=s_2'$. We have the following two cases.
\begin{itemize}
\item
If $ \{ (s_1', s_2', \delta_E(s_E, \alpha)) \in S \mid (s_1, s_2, s_E) \in \phi \} \cap \phi' = \emptyset$, $\delta_1(s_1, \alpha)(\loc_1') = 0$ or $\delta_2(s_2, \alpha)(\loc_2') = 0$, then we do not divide $\phi$ and let $\Phi'(\phi, \phi') = \{\phi\}$ and we have $\delta(s, \alpha)(s') = p_{\alpha}(\phi, \phi') = 0$ for all $s \in \phi$ and $s' \in \phi'$. 
\item
If $( \cup_{s \in \phi} \Theta_s^{\alpha} ) \cap \phi'$ is non-empty, then since $\delta_E({}\cdot{},\alpha) : S_E \rightarrow S_E$ is BFCP invertible using \aspref{asp:continuous-PWC} and $\phi'$ is a Borel measurable region, there exists a BFCP $\Phi'(\phi, \phi')$ of $\phi$ such that for each $\phi_1 \in \Phi'(\phi, \phi')$:
\begin{itemize}
\item
either $\delta(s, \alpha)(s') = p_{\alpha}(\phi_1, \phi') = 0$ for all $s \in \phi_1$ and $s' \in \phi'$;
\item
or for $s, \tilde{s} \in \phi_1$ there exist unique states $s', \tilde{s}' \in S$ such that $s' = \Theta_s^{\alpha} \cap \phi'$, $\tilde{s}' = \Theta_{\tilde{s}}^{\alpha} \cap \phi'$ and $\delta(s, \alpha)(s') = \delta(\tilde{s}, \alpha)(\tilde{s}') = p_{\alpha}(\phi_1, \phi') > 0$.
\end{itemize}
It remains to show that the bimeasurable, BFCP invertible function $q_{\alpha}$ of $2. (ii)$ exists, which follows from the the fact that $\delta_E({}\cdot{},\alpha) : S_E \rightarrow S_E$ is bimeasurable and BFCP invertible.
\end{itemize}
Finally, we divide $\phi$ into a BFCP $\sum_{\phi' \in \Phi_P} \Phi'(\phi, \phi')$, and therefore each region of this BFCP has the required reachability consistency.
\end{proof}

\begin{examp}\label{eg4-eg}
Returning to \egref{eg1-eg},  we now give the perception BFCPs for the two implementations of the classifier $f_{\mathcal{R}}$ proposed in \egref{eg3-eg}. In each case the perception BFCP is of the form $\Phi_P = \Loc_1 \times \Loc_2 \times \Phi_E$, where $\Phi_E$ is a BFCP for the environment state space and the perception BFCP is also the reward BCFP $\Phi_R^{\alpha}$ for $\alpha \in A$. In this example, all pairs of local states correspond to the same classifier for both agents.

For the first implementation of $f_{\mathcal{R}}$, which employs a linear regression model, the  BFCP $\Phi_E$ for the environment state space is given by:
\begin{equation*}
 \left\{ \{ (x, y) \in \mathcal{R} \mid (i < x \leq i+1) \wedge (j < y \leq j+1) \} \mid i,j \in \{0, 1, 2, 3 \} \right\}^2 
\end{equation*}
as shown in Fig. \ref{fig:vehicle-parking-bfcps} (left). 
{\revise For the second implementation, the BFCP $\Phi_E$ can be found by computing the (exact) preimage of the feed-forward NN classifier of \egref{eg3-eg} using~\cite{KM-FF:20}, and is shown in Fig.~\ref{fig:vehicle-parking-bfcps} (right). Since an environment state represents the position of each agent, in each case
the perception BFCP is two copies, one for each agent, of how the corresponding NN classifier divides the bounded area into 16 abstract grid cells.
Each abstract grid cell in Fig.~\ref{fig:vehicle-parking-bfcps} (right) is a union of polytopes and is indicated by a different colour.} 
%
\hfill$\blacksquare$
\end{examp}

\input{figure/vehicle-bfcps}

%% file: figure/vehicle-bfcps.tex
\def\bias{5}
\begin{figure}
\centering
 \edef\xlabels{1,2,3,4}
    \edef\ylabels{1,2,3,4}
    \edef\xlabelsall{0.5,1,1.5,2,2.5,3,3.5,4}
    \edef\ylabelsall{0.5,1,1.5,2,2.5,3,3.5,4}
    \edef\planes{0,1}
    \edef\planedisx{5}
    \edef\planedisy{0}
    \begin{tikzpicture}[scale = 0.48]
    \end{tikzpicture}
    \begin{tikzpicture}[scale = 0.55]
    \foreach \plane in \planes
    {
    \draw[fill=blue!10, very thick] (0 + \planedisx*\plane, 0 - \planedisy*\plane) rectangle (4 + \planedisx*\plane, 4 - \planedisy*\plane);
    \foreach \x in \xlabels
        \draw[black, thick] (\x + \planedisx*\plane, 0 - \planedisy*\plane) -- (\x + \planedisx*\plane, 4 - \planedisy*\plane);
    \foreach \y in \ylabels
        \draw[black, thick] (0 + \planedisx*\plane, \y - \planedisy*\plane) -- (4 + \planedisx*\plane, \y - \planedisy*\plane);
    }
    \node at (15, 2) {\includegraphics[width=0.17\textwidth] {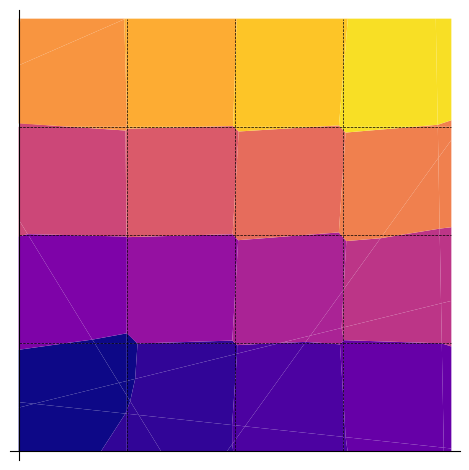}};
    \node at (4.5, 2) {$\times$};
    \node at (17.5, 2) {$\times$};
    \node at (20, 2) {\includegraphics[width=0.17\textwidth] {figure/4x4_perception_no_axes.png}};
    \end{tikzpicture}
  \caption{Perception BFCP of the environment states $\Phi_E$. Left: linear regression model. {\revise Right: feed-forward NN model, where each subfigure, computed as an exact preimage of the feed-forward NN classifier in \egref{eg3-eg} using~\cite{KM-FF:20},
  depicts the boundaries of the 
  16 abstract grid cells learnt by the classifier. 
  Each abstract grid cell in the right figure is a union of polytopes and is indicated by a different colour.}
  }\label{fig:vehicle-parking-bfcps}
\end{figure}

%% file: values.tex
\section{Values of zero-sum NS-CSGs}\label{values-sect}

We now proceed by establishing the value of an NS-CSG $\csg$ with respect to a {\revise discounted accummulated reward}  objective $Y$, i.e., for a reward structure $r$ and discount factor $\beta$. We prove the existence of this value, which is a fixed point of a minimax operator. Using Banach's fixed-point theorem, a sequence of bounded, Borel measurable functions converging to this value is constructed.


Given a state $s$ and (strategy) profile $\sigma=(\sigma_1,\sigma_2)$ of $\sem{\csg}$, we denote by $\mathbb{E}_s^{\sigma}[Y]$ the expected value of the objective $Y$ when starting from state $s$, given by \eqref{eq:reward-over-path}. The functions  $\underline{V},\overline{V} : S \rightarrow \mathbb{R}$, where $s \in S$:
\begin{eqnarray*}
      \underline{V}(s) & \coloneqq & \sup\nolimits_{\sigma_1 \in \Sigma_1}\inf\nolimits_{\sigma_2 \in \Sigma_2} \mathbb{E}_s^{\sigma_1,\sigma_2}[Y] \\
      \overline{V}(s) & \coloneqq & \inf\nolimits_{\sigma_2 \in \Sigma_2} \sup\nolimits_{\sigma_1 \in \Sigma_1} \mathbb{E}_s^{\sigma_1,\sigma_2}[Y]
\end{eqnarray*}
are called the \emph{lower value} and \emph{upper value} of $Y$, respectively. 

\begin{defi}[Value function]\label{defi:value_function}
If $\underline{V}(s)=\overline{V}(s)$ for all $s \in S$, then $\sem{\csg}$ is determined with respect to the objective $Y$ and the common function is called the \emph{value} of $\sem{\csg}$, denoted by $V^\star$, with respect to $Y$.
\end{defi}
%
%
We next introduce the spaces of feasible state-action pairs and state-action-distribution tuples, and present properties of these spaces. More precisely, for $1 \leq i \leq 2$, we let:
\begin{align*}
\Xi_i & \coloneqq \{(s,a_i) \in S {\times} A_i \mid  a_i \in A_i(s) \} \\
\Lambda_i &    \coloneqq \{(s,u_i) \in S {\times} \mathbb{P}(A_i) \mid u_i \in \mathbb{P}(A_i(s)) \} \\
\Xi_{12}  &   \coloneqq \{(s,(a_1,a_2)) \in S {\times} (A_1 {\times} A_2) \mid a_1 \in A_1(s) \wedge  a_2 \in A_2(s)\} \\
\Lambda_{12}   &  \coloneqq \{(s,(u_1,u_2)) \in S {\times} (\mathbb{P}(A_1) {\times} \mathbb{P}(A_2)) \mid u_1 \in \mathbb{P}(A_1(s)) \wedge u_2 \in \mathbb{P}(A_2(s))\} \, .
\end{align*}    
%

\begin{lema}[Borel sets]\label{eq:borel-space-xi-lambda}
For $1 \leq i\leq 2$, the sets $\Xi_i$ and $\Lambda_i$ are Borel sets of $S \times A_i$ and $S \times \mathbb{P}(A_i)$, respectively. Furthermore, the sets $\Xi_{12}$ and $\Lambda_{12}$ are Borel sets of $S \times (A_1 \times A_2)$ and $S \times (\mathbb{P}(A_1) \times \mathbb{P}(A_2))$, respectively.
\end{lema}
\begin{proof}
We first consider $\Xi_i$ and $\Lambda_i$ for $i=1$ (the case for $i=2$ follows similarly). Since $A_1$ is finite, the sets $\Xi_1$ and $\Lambda_1$ can be rearranged as:
\begin{eqnarray*}
    \Xi_1  & = & \mbox{$\bigcup_{\hat{A}_1 \subseteq A_1}$} \big( \{s_1 \mid \Delta_1(s_1)=\hat{A}_1 \} \times S_{2} \times S_E  \times  \hat{A}_1 \big) \cap (S \times A_1) \\
    \Lambda_1  & = & \mbox{$\bigcup_{\hat{A}_1 \subseteq A_1}$} \big( \{s_1 \mid \Delta_1(s_1)=\hat{A}_1 \} \times S_2 \times S_E \times \mathbb{P}(\hat{A}_1) \big) \cap (S \times \mathbb{P}(A_1)) \, .
\end{eqnarray*}
Since $\hat{A}_1$ is a subset of the finite set $A_1$, the sets $\hat{A}_1$ and $\mathbb{P}(\hat{A}_1)$ are Borel sets of $A_1$ and $\mathbb{P}(A_1)$, respectively. Since $S_1$ is a finite set, for any $\hat{A}_1 \subseteq A_1$, the set $\{s_1 \mid \Delta_1(s_1)=\hat{A}_1 \}$ is a Borel set of $S_1$. Since $S_{2}$ and $S_E$ are both Borel sets by \lemaref{lema:borel-space-S-A}, the result follows by Theorem 1.10 \cite[Chapter 1]{KRP:67}. Using similar reasoning, it follows that $\Xi_{12}$ and $\Lambda_{12}$ are also Borel sets of the respective spaces.
\end{proof}

\begin{pro}[Stochastic kernel transition function]\label{prop:stochastic-kernel-trans}
The probabilistic transition function $\delta$ of $\sem{\csg}$ is a stochastic kernel.
\end{pro}

\begin{proof}
From \defiref{semantics-def}, it follows that, for any $(s,\alpha) \in \Xi_{12}$, we have $\delta(s,\alpha)({}\cdot{}) \in \mathbb{P}(S)$. We show that, if $B \in \mathcal{B}(S)$, then $\delta({}\cdot{},{}\cdot{})(B) : (S \times A) \rightarrow \mathbb{R}$ is Borel measurable on $\Xi_{12}$. More precisely, we prove that, for any $c \in \mathbb{R}$, the preimage of the Borel set $[c,\infty)$ of $\mathbb{R}$ under $\delta({}\cdot{},{}\cdot{})(B)$ which is given by:
\[
  \delta^{-1}([c,\infty))(B)=\{ (s, \alpha)\in \Xi_{12} \mid \delta(s, \alpha)(B) \ge c \} 
\]
is an element of $\mathcal{B}(\Xi_{12})$. If $c >1$, then $\delta^{-1}([c,\infty))(B)=\emptyset \in \mathcal{B}(\Xi_{12})$, and if $c\leq0$, then $\delta^{-1}([c,\infty))(B)=\Xi_{12} \in \mathcal{B}(\Xi_{12})$. 

Therefore, it remains to consider the case when $0 < c \leq1$. 
Consider any $\alpha \in A$ and let $\Phi_P^\alpha$ be the refinement of $\Phi_P$ of \lemaref{lema:reachability-consistency}. For each $\phi \in \Phi_P^\alpha$ and $\phi' \in \Phi_P$ such that $p_{\alpha}(\phi, \phi') > 0$, let $q_{\alpha}: \phi \to \phi'$ be the associated bimeasurable, BFCP invertible function from \lemaref{lema:reachability-consistency}. The image of $\phi$ under $q_{\alpha}$ into $\phi'$ is given by:
\[
    \hat{q}_{\alpha}(\phi, \phi') = \{s'\in \phi' \mid s' = q_{\alpha}(s) \wedge s \in \phi\} \, .
\]
By Lemmas \ref{lema:perception-BFCP} and \ref{lema:reachability-consistency}, both $\phi$ and $\phi'$ are Borel sets and $q_{\alpha}$ is bimeasurable, and therefore $\hat{q}_{\alpha}(\phi, \phi')$ is a Borel set. Next, since $q_{\alpha}$ is Borel measurable, the preimage of the Borel set $\hat{q}_{\alpha}(\phi, \phi') \cap B$ under $q_{\alpha}$ over the region $\phi$, which is given by:
\[
    \hat{q}_{\alpha}^{-1}(\phi, \hat{q}_{\alpha}(\phi, \phi') \cap B) = \{ s \in \phi \mid q_{\alpha}(s) \in \hat{q}_{\alpha}(\phi, \phi') \cap B\}
\]
is a Borel set. By combining this result with  \lemaref{lema:reachability-consistency}, each state in $\hat{q}_{\alpha}^{-1}(\phi, \hat{q}_{\alpha}(\phi, \phi') \cap B)$ under $\alpha$ transitions to $B$ with probability $p_{\alpha}(\phi, \phi')$. We denote the set of all transition probabilities from $\phi$ under $\alpha$ by $P_{\alpha}(\phi) = \{ p_{\alpha}(\phi, \phi') > 0 \mid \phi' \in \Phi_P \}$. Then, the collection of the subsets of $P_{\alpha}(\phi)$ for which the sum of their elements is greater or equal to $c$ is given by:
\[
    P_{\alpha}^{\ge c}(\phi) \coloneqq \big\{ P' \subseteq P_{\alpha}(\phi) \mid \mbox{$\sum\nolimits_{p' \in P'}$} p' \ge c\big\}
\]
and is finite. Now for each set $P' \in P_{\alpha}^{\ge c}(\phi)$, the states in the set:
\[
    O_{\alpha}(\phi, P') = \mbox{$\bigcap_{p_{\alpha}(\phi, \phi') \in P'}$} \hat{q}_{\alpha}^{-1}(\phi, \hat{q}_{\alpha}(\phi, \phi') \cap B)
\]
transition to $B$ under $\alpha$ with probability greater or equal to $c$ and $O_{\alpha}(\phi, P')$ is a Borel set as $P'$ is a finite set. Thus, the states in $\phi$ reaching $B$ under $\alpha$ with probability greater or equal to $c$ are given by:
\[
    O_{\alpha}(\phi) = \mbox{$\bigcup_{P' \in P_{\alpha}^{\ge c}(\phi)}$} O_{\alpha}(\phi, P')
\]
which is a Borel set since $P_{\alpha}^{\ge c}(\phi)$ is a finite set. Finally, we have:
\[
    \delta^{-1}([c,\infty))(B) = \mbox{$\bigcup_{\alpha \in A} \bigcup_{\phi \in \Phi_P^\alpha}$} \{(s, \alpha) \in \Xi_{12} \mid s \in O_{\alpha}(\phi)\}
\]
and therefore, combining Lemmas \ref{lema:reachability-consistency} and \ref{eq:borel-space-xi-lambda}, it follows that $\delta^{-1}([c,\infty))(B) \in \mathcal{B}(\Xi_{12})$  as required.
\end{proof}
\noindent
Before presenting properties of the value function, we introduce the following operator based on the classical Bellman equation. {\revise Recall that $r_A$ and $r_S$ are the action and state reward functions defined in Section~\ref{sub:zero-sum-ns-csgs}, respectively, and $\delta$ is the transition function of the NS-CSG (\defiref{semantics-def}). Further, we require that the set of successors of $s$ under $\alpha$,  denoted $\Theta_s^\alpha=\{s'  \mid \delta(s,\alpha)(s')>0\}$, is finite for all $s\in S$ and $\alpha\in A(s)$ (see Section~\ref{sub:semantics-ns-csg}).}

%
\begin{defi}[Minimax operator]\label{defi:minimax-operator}
Given a  bounded, Borel measurable real-valued function $V \in \mathbb{F}(S)$, the minimax operator $T : \mathbb{F}(S) \rightarrow \mathbb{F}(S)$ is defined, for any $s \in S$, by:
\[
  {\revise [TV](s)} \coloneqq \max_{u_1\in \mathbb{P}(A_1(s))}\min_{u_2\in \mathbb{P}(A_2(s))} \mbox{$\sum\limits_{a_1\in A_1(s)}$} \mbox{$\sum\limits_{a_2\in A_2(s)}$} Q(s,(a_1,a_2),V)u_1(a_1)u_2(a_2)
\]
where for any $\alpha \in A(s)$:
\[
   Q(s,\alpha,V) \coloneqq r_A(s,\alpha)+r_S(s)+ \beta \mbox{$\sum_{s'\in \Theta_s^\alpha}$} \delta(s,\alpha)(s')V(s') \, .
\]
\end{defi}
{\revise We use the standard notation $[TV]$ for $T(V)$ in the remainder of the paper to simplify the presentation.}
%
%

%
\begin{thom}[Value function]\label{thom:value-function}
If $\csg$ is an NS-CSG and $Y$ is a discounted zero-sum objective, then
\begin{enumerate}[(i)]
    \item\label{itm:value-existence} $\sem{\csg}$ is determined with respected to $Y$, i.e., $V^{\star}$ exists;
    \item\label{itm:fixed-point} $V^\star$ is the unique fixed point of the operator $T$;
    \item\label{itm:bounded-Borel} $V^\star$ is a bounded, Borel measurable function. 
\end{enumerate}
\end{thom}

\begin{proof}
The proof follows through showing that $\sem{\csg}$ is an instance of a zero-sum stochastic game that satisfies the conditions of the Borel model presented in \cite{PRK-THS:81}.

From \lemaref{lema:borel-space-S-A}, we have that $A_1$, $A_2$ and $S$ are complete and separable metric spaces. By \lemaref{eq:borel-space-xi-lambda}, the spaces $\Xi_i$ and $\Lambda_i$ are Borel sets of $S \times A_i$ and $S \times \mathbb{P}(A_i)$ for $1 \leq i\leq 2$, respectively. By \proref{prop:stochastic-kernel-trans}, $\delta$ is a Borel stochastic kernel.
Furthermore, from \aspref{asp:continuous-PWC} we have that $r_A+r_S : (S \times A) \rightarrow \mathbb{R}$ is bounded, and therefore it follows that $\sem{\csg}$ with respect to the zero-sum objective $Y$ is an instance of a zero-sum stochastic game with Borel model and discounted payoffs introduced in \cite{PRK-THS:81}. Hence, (\ref{itm:value-existence}) follows from \cite[Theorems 2 and 3]{PRK-THS:81}, and (\ref{itm:fixed-point}) from the discounted case of \cite[Theorem 1]{PRK-THS:81}. 
Finally, for (\ref{itm:bounded-Borel}), since $\beta \in(0,1)$, we have that $V^\star$ is bounded, and therefore $V^\star$ is Borel measurable using \cite[Lemma 3]{PRK-THS:81}.
\end{proof}

\noindent
The following guarantees that value iteration (VI) converges to the value function.

\begin{pro}[Convergence]\label{prop:convergence-seq}
For any $V^0 \in \mathbb{F}(S)$, the sequence $( V^t )_{t \in \mathbb{N}}$, where 
{\revise $V^{t+1}=[TV^t]$,}
converges to $V^\star$. Moreover, each $V^t$ is bounded, Borel measurable.
\end{pro}
\begin{proof}
Since $r_A+r_S : (S \times A) \rightarrow \mathbb{R}$ is bounded, using \cite[Lemma 2]{PRK-THS:81} we have that, if $V^t$ is bounded, Borel measurable, then {\revise $[TV^t]$} is also bounded. The result then follows from the fact that $V^\star(s)=\lim_{t\to\infty}V^t(s)$ for all $s \in S$ if {\revise $V^{t+1}=[TV^t]$} for all $t \in \mathbb{N}$~\cite{PRK-THS:81}.
\end{proof}







%% file: value_iteration.tex
\section{Value iteration}\label{vi-sect}

Despite the convergence result of \proref{prop:convergence-seq},
in practice there may not exist 
finite representations of general bounded Borel measurable functions $(V^t)_{t \in \mathbb{N}}$
due to the uncountable state space.
We now show how VI can be used to \emph{approximate} the values of $\sem{\csg}$, 
based on a sequence of {\revise Borel piecewise constant}
(B-PWC) functions (Section~\ref{background-sect}).


%
\subsection{B-PWC closure and convergence}
For NS-CSGs, we demonstrate that, under \aspref{asp:continuous-PWC}, 
a B-PWC representation of value functions is closed under the minimax operator and ensures convergence of value iteration.
\begin{thom}[B-PWC closure and convergence]\label{thom:B-PWC-closure-VI}
If $V \in \mathbb{F}(S)$ and B-PWC, then so is $Q({}\cdot{},\alpha,V)$ and $[TV]$ for $\alpha \in A$. If $V^0 \in \mathbb{F}(S)$ and B-PWC, 
the sequence $(V^t)_{t \in \mathbb{N}}$ such that {\revise $V^{t+1} = [TV^t]$} converges to $V^\star$, and each $V^t$ is B-PWC.
\end{thom}
\begin{proof}
Considering any B-PWC function
$V \in \mathbb{F}(S)$ and joint action $\alpha \in A$, since $r_A({}\cdot{},\alpha)+r_S({}\cdot{})$ is B-PWC by \aspref{asp:continuous-PWC}, the fact that $Q({}\cdot{},\alpha,V)$ is B-PWC follows if, by \defiref{defi:minimax-operator}, we can show that the function $\overline{Q}({}\cdot{},\alpha,V)$ where:
\[
    \overline{Q}({}\cdot{},\alpha,V) \coloneqq \mbox{$\sum\nolimits_{s'\in \Theta_{{}\cdot{}}^{\alpha}}$} \delta({}\cdot{},\alpha)(s')V(s')
\]
is B-PWC. Boundedness follows because $V$ is bounded. The indicator function of a subset $S' \subseteq S$ is the function $\chi_{S'}:S \to \mathbb{R}$ such that $\chi_{S'}(s) = 1$ if $s \in S'$ and $0$ otherwise. Now $\chi_{S'}$ is Borel measurable if and only if $S'$ is a Borel set of $S$~\cite{HLR-PF:10}. For clarity, we use $q_{\alpha}(s;\phi,\phi')$ to refer to $q_{\alpha}$ from \lemaref{lema:reachability-consistency} for $\alpha \in A$, $s\in \phi$, $\phi \in \Phi_P^\alpha$ and $\phi' \in \Phi_P$ (where again $\Phi_P^\alpha$ is from \lemaref{lema:reachability-consistency}). For any $s \in S$ such that $\delta(s, \alpha)$ is defined, we have:
\begin{align*}
 \overline{Q}(s,\alpha,V) & = \; \mbox{$\sum\nolimits_{\phi \in \Phi_P^\alpha}$}  \chi_{\phi}(s) \mbox{$\sum\nolimits_{s'\in \Theta_s^{\alpha}}$} \delta(s,\alpha)(s')V(s') \\
 & = \; \mbox{$\sum\nolimits_{\phi \in \Phi_P^\alpha}$} \chi_{\phi}(s) \mbox{$\sum\nolimits_{\phi' \in \Phi_{P}}$} p_{\alpha}(\phi, \phi')V(q_{\alpha}(s; \phi, \phi')) & \mbox{by \lemaref{lema:reachability-consistency}} \\
& = \; \mbox{$\sum\nolimits_{ \phi \in \Phi_P^\alpha}$}
\mbox{$\sum\nolimits_{\phi' \in \Phi_P}$} p_{\alpha}(\phi, \phi') \chi_{\phi}(s) V(q_{\alpha}(s; \phi, \phi')) & \mbox{rearranging.}
\end{align*}
Since $\phi$ is a Borel set of $S$, we have that $\chi_{\phi}$ is Borel measurable. Next, we show that $V(q_{\alpha}({}\cdot{}; \phi, \phi'))$ is Borel measurable on $\phi$.
Let $\Phi_{V}$ be a constant-BFCP of $S$ for $V$. Given $c \in \mathbb{R}$, we denote by $\Phi_{V}^{\ge c}$ the set of regions in $\Phi_V$ on which $V\ge c$ holds. The preimage of $[c,\infty)$ under $V(q_{\alpha}({}\cdot{}; \phi, \phi'))$ defined on $\phi$ is given by:
\begin{eqnarray*}
      V^{-1}(q_{\alpha}([c, \infty); \phi,\phi')) & = & \{ s \in \phi \mid V(q_{\alpha}(s ; \phi,\phi')) \ge c \} \\
      & = & \mbox{$\bigcup_{\phi_V \in \Phi_V^{\ge c}}$} \{ s \in \phi  \mid q_{\alpha}(s; \phi, \phi') \in \phi_V \} \, .
\end{eqnarray*}
Since $q_{\alpha}(s; \phi,\phi')$ is Borel measurable in $s \in \phi$ (see \lemaref{lema:reachability-consistency}) and $\phi_V$ is a Borel set of $S$, then $\{ s \in \phi \mid q_{\alpha}(s; \phi, \phi') \in \phi_V \}$ is a Borel set of $\phi$. Since $V^{-1}(q_{\alpha}([c, \infty); \phi, \phi'))$ is also a Borel set of $\phi$ by noting that $\Phi_V^{\ge c}$ is finite, it follows that 
$V(q_{\alpha}({}\cdot{}; \phi,\phi'))$ is Borel measurable on $\phi$. Therefore $\overline{Q}({}\cdot{},\alpha,V)$ is Borel measurable.

Next, since $q_{\alpha}({}\cdot{}; \phi,\phi')$ is BFCP invertible on $\phi$ by \lemaref{lema:reachability-consistency}, there exists a BFCP $\Phi_{q}$ of $\phi$ such that all states in each region of $\Phi_{q}$ are mapped into the same region of $\Phi_V$ under $q_{\alpha}({}\cdot{}; \phi,\phi')$. Following this, $V(q_{\alpha}({}\cdot{}; \phi,\phi'))$ is constant on each region of $\Phi_{q}$. Therefore, using the fact that $\chi_{\phi}$ is PWC, it follows that  $\overline{Q}({}\cdot{},\alpha,V)$ is PWC, which completes the proof that $\overline{Q}({}\cdot{},\alpha,V)$ is B-PWC.

From \proref{prop:convergence-seq} we have that $[TV]$ is bounded, Borel measurable. Since $Q({}\cdot{},\alpha,V)$ is PWC for any joint action $\alpha \in A$, $A(s)$ is PWC and $A$ is finite, it follows that $[TV]$ is PWC using the fact that the value of a zero-sum normal-formal game induced at every $s \in S$ is unique. Thus, $[TV]$ is B-PWC. The remainder of the proof follows directly from Banach's fixed point theorem and the fact we have proved that, if $V \in \mathbb{F}(S)$ and B-PWC, so is $[TV]$.
\end{proof}

\subsection{B-PWC VI algorithm}\label{B-PWC-VI-algorithm}

We use the closure property of B-PWC value functions under the minimax operator from \thomref{thom:B-PWC-closure-VI}
to iteratively construct a sequence $(V^t )_{t \in \mathbb{N}}$ of 
such functions to approximate $V^{\star}$ to within a convergence guarantee.
\algoref{alg:PWC-value-VI} presents our B-PWC VI scheme,  where the BFCP of the B-PWC value function at each iteration is refined (line 6) and subsequently the B-PWC value function is updated via minimax computations (line 8) for a state sampled from each of its regions.
 



\begin{algorithm}[t]
\caption{B-PWC VI}\label{alg:PWC-value-VI}
\begin{algorithmic}[1]
\State{\textbf{Input:} NS-CSG $\csg$,
perception FCP $\Phi_{P}$, reward FCPs $(\Phi_R^{\alpha} )_{\alpha \in A}$, error $\varepsilon$}
\State{\textbf{Output:} Approximate value function $V$}
\State Initialize $(\Phi_{V^0}, V^0)$
\State {$\bar{\varepsilon} \leftarrow 2\varepsilon$, \ $t \leftarrow 0$ }
\While{$\bar{\varepsilon} > \varepsilon$}
\State {$\Phi_{V^{t+1}} \leftarrow \mathit{Preimage\_BFCP}(\Phi_{V^t}, \Phi_{P}, ( \Phi_R^{\alpha} )_{\alpha \in A})$ \ (\algoref{alg:BFCP-iteration-VI}) }
\For{$\phi \in \Phi_{V^{t+1}} $}
\State Take one state $s \in \phi$, \ $V^{t+1}_{\phi} \leftarrow [TV^t](s)$
\EndFor
\State{$\bar{\varepsilon} \leftarrow \mathit{Dist}(V^{t+1}, V^t)$}
\State{$t \leftarrow t+1$}
\EndWhile
\State \Return {$V \leftarrow V^t$}
\end{algorithmic}
\end{algorithm}



%

\begin{algorithm}[t]
\caption{BFCP iteration for B-PWC VI}\label{alg:BFCP-iteration-VI}
\begin{algorithmic}[1]
\Procedure{$\mathit{Preimage\_BFCP}$}{$\Phi$, $\Phi_{P}$, $( \Phi_R^{\alpha} )_{\alpha \in A}$}
    \State{$\Phi_{\textup{pre}} \leftarrow \emptyset$ }
    \For{$\phi \in \Phi_P + \sum_{\alpha \in A} \Phi_R^{\alpha}$}
    \State{$\Phi_{\textup{pre}}^\phi \leftarrow \emptyset$}
    \For{$\alpha \in A$, $ \phi' \in \{ \phi' \in  \Phi \mid (\cup_{s \in \phi} \Theta_s^{\alpha} ) \cap \phi' \neq \emptyset \}$}
    \State{$\Phi_{\textup{pre}}^\phi \leftarrow \Phi_{\textup{pre}}^\phi \cup \left\{ \{s \in \phi \mid \Theta_s^{\alpha} \cap \phi' \neq \emptyset \} \right\}$}
    \EndFor
    \State{$\Phi_{\textup{pre}} \leftarrow \Phi_{\textup{pre}} \cup \{\phi_1 \in \mathit{Intersect}(\phi,\Phi_{\textup{pre}}^\phi) \}$}
    \EndFor
    \State \Return {$\Phi_{\textup{pre}}$}
\EndProcedure
\end{algorithmic}
\end{algorithm}


  

\startpara{Initialization} The function  $V^0$ is initialised as a 0-valued B-PWC function
defined over the BFCP $ \Phi_{V^0} = \Phi_P + \sum_{\alpha \in A}\Phi_R^{\alpha}$ of $S$, i.e.,  $V_{\phi}^0 = 0$ for $\phi \in \Phi_{V^0}$.

%
\input{figure/VI_graph}

\startpara{The algorithm} The steps of 
our B-PWC
VI algorithm are illustrated in Fig. \ref{fig:VI-graph}. These steps use $\mathit{Preimage\_BFCP}(\Phi_{V^t}, \Phi_P, ( \Phi_R^{\alpha} )_{\alpha \in A})$, see \algoref{alg:BFCP-iteration-VI}, to compute a refinement of $\Phi_P + \sum_{\alpha \in A} \Phi_R^{\alpha}$ that is a preimage BFCP of $\Phi_{V^{t}}$ for $\delta$.
Then, in order to compute the value $V^{t+1}_{\phi}$ over each region $\phi \in \Phi$, we take one state $s\in \phi$ and then find the value of a zero-sum normal form game~\cite{NMK+44} 
at $s$ induced by \defiref{defi:minimax-operator}.

As a convergence criterion for
B-PWC VI in \algoref{alg:PWC-value-VI}, we detect when the difference between
successive value approximations falls below a threshold $\varepsilon$
(as usual for VI, this does \emph{not} guarantee an $\varepsilon$-optimal solution).
The function $\mathit{Dist}(V^{t+1}, V^t)$ computes the difference between $V^{t+1}$ and $V^t$,
which may have different regions
due to the possible inconsistency between $\Phi_{V^{t+1}}$ and $\Phi_{V^{t}}$.
An intuitive method is to evaluate $V^{t+1}$ and $V^t$ at a finite set of points, and then compute the maximum difference.
In the usual manner for VI, an approximately optimal strategy
can be extracted from the final step of the computation.

\algoref{alg:BFCP-iteration-VI} requires region-wise computations involving the image and preimage of a region, region intersection and the sum of BFCPs.
In particular, $\mathit{Intersect}(\phi,\Phi_{\textup{pre}}^\phi)$ is the refinement of $\phi$ obtained by computing all pairwise intersections of $\phi$ with regions in $\Phi_{\textup{pre}}^\phi$ and, by construction, is a preimage BFCP of $\Phi$ for $\delta$ over $\phi$. The following corollary then follows from \lemaref{lema:reachability-consistency} and \thomref{thom:B-PWC-closure-VI}. 


\begin{cor}[BFCP iteration for B-PWC VI]\label{cor:bfcp-iteration}
In \algoref{alg:BFCP-iteration-VI}, $\Phi_{\textup{pre}}$ is a refinement of $\Phi_P + \sum_{\alpha \in A} \Phi_R^{\alpha}$ and is a preimage BFCP of $\Phi$ for $\delta$.
\end{cor}

\input{figure/region_refinement}

{\revise
\begin{examp}\label{eg-new-add}
Returning to \egref{eg1-eg}, in Fig.~\ref{fig:refin} we illustrate the refinement process for a region in the perception BFCP when performing \algoref{alg:BFCP-iteration-VI}. 
We focus on the perception
function 
that is implemented via the linear regression model given in \egref{eg3-eg}. The preimage BFCP is constructed as described in \lemaref{lema:reachability-consistency} and \cororef{cor:bfcp-iteration} for the case $\Phi = \Phi_P$. For simplicity, here we only consider actions and coordinates of $\agent_1$. 


 Fig.~\ref{fig:refin} demonstrates how the grey region in the initial BFCP in the top left grid is subdivided when considering each action of $\agent_1$. 
In general, when subdividing a region, we also have to take into account the reward BFCP as indicated in Algorithms~\ref{alg:PWC-value-VI} and \ref{alg:BFCP-iteration-VI}. 
 The shaded squares (diagonal lines) in the top centre and top right grids are the images of the initial grey region under the actions $\textit{up}$ and $\textit{down}$, respectively. Based on their intersections and the corresponding preimages, the grey region in the top left grid is subdivided into green and purple subregions in the bottom left grid such that all points in each subregion lead to the same new percepts under both actions,  $\textit{up}$ and $\textit{down}$. The bottom right grid then shows the subdivision of the grey region in the top left grid after considering all actions.
\end{examp}
}


\startpara{Polytope regions}
Our B-PWC VI algorithm assumes that each region in a BFCP is finitely representable.
We now briefly discuss the use of BFCPs defined by \emph{polytopes}, which
suffice 
for {\revise a symbolic representation of}
perception BCFPs of ReLU NNs (discussed below). The focus is the region-based computations required by \algoref{alg:BFCP-iteration-VI}.
%
A polytope $\phi \subseteq \mathbb{R}^{m}$ is an intersection of $\ell$ halfspaces $\{ x \in \mathbb{R}^{m}  \mid g_k(x) \ge 0 \mbox{ for } 1 \leq k \leq \ell \}$, where $g_k(x) = W_k^\top x + b_k$ is a linear function, i.e., $W_k \in \mathbb{R}^{m}$ and $b_k \in \mathbb{R}$, for $1 \leq k \leq \ell$.
%
If $\phi_1$ and $\phi_2$ are polytopes, represented by $\{(W_k, b_k)\}_{k=1}^{\ell'}$ and $\{(W_k, b_k)\}_{k=\ell' + 1}^{\ell}$, respectively, then the intersection
$\phi_1 \cap \phi_2$, 
is the intersection of $\ell$ halfspaces 
and can be represented as $\{(W_k, b_k)\}_{k=1}^{\ell}$. Therefore, the sum $\Phi_1 + \Phi_2$ of two BFCPs $\Phi_1$ and $\Phi_2$ can be computed by considering the intersection $\phi_1 \cap \phi_2$
of all pairwise
combinations of regions $\phi_1 \in \Phi_1$ and $\phi_2 \in \Phi_2$.

The \emph{image} of a polytope $\phi = \{ x \in \mathbb{R}^{m} \mid g_k(x) \ge 0 \mbox{ for } 1 \leq k \leq \ell \}$ under a linear function $f : \mathbb{R}^{m} \rightarrow \mathbb{R}^{m}$, where $f(x) = D x + b$, $D \in \mathbb{R}^{m \times m}$ is non-singular and $b \in \mathbb{R}^{m}$, 
is the polytope $f(\phi) = \{ x \in \mathbb{R}^{m} \mid W_k^\top D^{-1} x + b_k - W_k^\top D^{-1} b \ge 0 \mbox{ for } 1 \leq k \leq \ell \}$ with the representation $\{ (D^{-\top}W_k, b_k - W_k^\top D^{-1} b)\}_{k=1}^{\ell}$.
The \emph{preimage} of $\phi$ under $f$ is the polytope $f^{-1}(\phi) = \{ x \in \mathbb{R}^{m} \mid W_k^\top D x + b_k + W_k^\top b \ge 0 \mbox{ for } 1 \leq k \leq \ell \}$ with the representation $\{ (D^\top W_k, b_k + W_k^\top b) \}_{k=1}^{\ell}$. 
Checking the feasibility of a set constrained by a set of linear inequalities can be solved by a linear program solver \cite{2020SciPy-NMeth}. 

\startpara{ReLU networks}
If each perception function $obs_i$ is implemented via a ReLU NN 
classifier, where the activation function is B-PWL, then the preimages of the ReLU NN for each percept~\cite{KM-FF:20}  have linear boundaries, and therefore all regions in the corresponding perception BFCP $\Phi_P$ can be {\revise (symbolically)}
represented by polytopes (see \egref{eg4-eg}).
If there exist polytope constant-BFCPs for B-PWC $r_A({}\cdot{}, \alpha)$ and $r_S$ for all $\alpha \in A$, then all regions in $\Phi_R^{\alpha}$ for $\alpha\in A$ are polytopes. If $\delta_E({}\cdot{}, \alpha)$ is piecewise linear and invertible and $\phi'$ is a polytope (line 5 in \algoref{alg:BFCP-iteration-VI}), then $\{s \in \phi \mid \Theta_s^{\alpha} \cap \phi' \neq \emptyset \}$ is a polytope. 
Therefore, each region in $\Phi_{\textup{pre}}$ is a polytope after every iteration and the operations over polytopes, including intersections, image and preimage computations, directly follow from the computation above.

\input{simulation}

%% file: figure/VI_graph.tex
\begin{figure}[t]
\centering
\def\scalesize{0.3}
\def\offsetright{15}
\def\offsetrighttwo{30}
\def\offsetdown{0}
\centering
\begin{tikzpicture}[scale = \scalesize]
\begin{scope}
\path
  coordinate (aux0) at (0,\offsetdown + 2)
  coordinate (aux1) at (0,\offsetdown + 3)
  coordinate (aux2) at (10,\offsetdown + 10)
  coordinate (aux3) at (9,\offsetdown + 6)
  coordinate (aux4) at (5,\offsetdown + 0)
  coordinate (aux5) at (7,\offsetdown + 0)
  coordinate (aux6) at (2,\offsetdown + 6)
  coordinate (aux7) at (5,\offsetdown + 6)
  coordinate (esp1) at (0.2,\offsetdown + 2.5)
  coordinate (esp2) at (1.5,\offsetdown + 1.5)
  coordinate (esp3) at (3,\offsetdown + 0.1)
  coordinate (esp4) at (5.5,\offsetdown + 1.1)
  coordinate (esp5) at (8,\offsetdown + 0.5)
  coordinate (esp6) at (8.75,\offsetdown + 2)
  coordinate (esp7) at (9.7,\offsetdown + 3)
  coordinate (esp8) at (6.5,\offsetdown + 4.5)
  coordinate (esp9) at (3.8,\offsetdown + 5.8)
  coordinate (esp10) at (1.5,\offsetdown + 4)
  ;
\draw[line width=0.8pt]
  (esp1) to[out=-90,in=170]
  (esp2) to[out=-10,in=170]
  (esp3) to[out=-10,in=180]
  (esp4) to[out=0,in=180]
  (esp5) to[out=10,in=-150]
  (esp6) to[out=20,in=-90]
  (esp7) to[out=90,in=-60]
  (esp8) to[out=120,in=0]
  (esp9) to[out=180,in=0]
  (esp10) to[out=180,in=90]
  cycle;    
\clip
  (esp1) to[out=-90,in=170]
  (esp2) to[out=-10,in=170]
  (esp3) to[out=-10,in=180]
  (esp4) to[out=0,in=180]
  (esp5) to[out=10,in=-150]
  (esp6) to[out=20,in=-90]
  (esp7) to[out=90,in=-60]
  (esp8) to[out=120,in=0]
  (esp9) to[out=180,in=0]
  (esp10) to[out=180,in=90]
  cycle;    
\filldraw[fill=cyan!40]
  (aux4) to[bend right=10]
  (aux6) --
  (aux7) to[bend left=10]
  (aux5) -- cycle;
\filldraw[fill=red!40]
  (aux5) to[bend right=10]
  (aux7) --
  (10,\offsetdown + 6) --
  (10,\offsetdown + 0) -- cycle;
\filldraw[fill=yellow!50]
  (0,\offsetdown + 0) -- 
  (aux4) to[bend right=10]
  (aux6) --
  (0,\offsetdown + 6) -- 
  (0,\offsetdown + 0) -- cycle;
\end{scope}

\begin{scope}
\path
  coordinate (aux0) at (\offsetright + 0,2)
  coordinate (aux1) at (\offsetright + 0,3)
  coordinate (aux2) at (\offsetright + 10,10)
  coordinate (aux3) at (\offsetright + 9,6)
  coordinate (aux4) at (\offsetright + 2,0)
  coordinate (aux5) at (\offsetright + 5,0)
  coordinate (aux6) at (\offsetright + 2,6)
  coordinate (aux7) at (\offsetright + 5,6)
  coordinate (esp1) at (\offsetright + 0.2,2.5)
  coordinate (esp2) at (\offsetright + 1.5,1.5)
  coordinate (esp3) at (\offsetright + 3,0.1)
  coordinate (esp4) at (\offsetright + 5.5,1.1)
  coordinate (esp5) at (\offsetright + 8,0.5)
  coordinate (esp6) at (\offsetright + 8.75,2)
  coordinate (esp7) at (\offsetright + 9.7,3)
  coordinate (esp8) at (\offsetright + 6.5,4.5)
  coordinate (esp9) at (\offsetright + 3.8,5.8)
  coordinate (esp10) at (\offsetright + 1.5,4)
  ;
\draw[line width=0.8pt]
  (esp1) to[out=-90,in=170]
  (esp2) to[out=-10,in=170]
  (esp3) to[out=-10,in=180]
  (esp4) to[out=0,in=180]
  (esp5) to[out=10,in=-150]
  (esp6) to[out=20,in=-90]
  (esp7) to[out=90,in=-60]
  (esp8) to[out=120,in=0]
  (esp9) to[out=180,in=0]
  (esp10) to[out=180,in=90]
  cycle;    
\clip
  (esp1) to[out=-90,in=170]
  (esp2) to[out=-10,in=170]
  (esp3) to[out=-10,in=180]
  (esp4) to[out=0,in=180]
  (esp5) to[out=10,in=-150]
  (esp6) to[out=20,in=-90]
  (esp7) to[out=90,in=-60]
  (esp8) to[out=120,in=0]
  (esp9) to[out=180,in=0]
  (esp10) to[out=180,in=90]
  cycle;    
\filldraw[fill=cyan!40]
  (aux4) to[bend right=10]
  (aux6) --
  (aux7) to[bend left=10]
  (aux5) -- cycle;
\filldraw[fill=orange!60]
  (aux5) to[bend right=10]
  (aux7) --
  (\offsetright + 10,6) --
  (\offsetright + 10,0) -- cycle;
\filldraw[fill=red!40]
  (aux6) -- 
  (aux1) to[bend right=20]
  (aux3) --
  (aux2) to[bend left=10] cycle;
\filldraw[fill=yellow!50]
  (\offsetright + 0,0) -- 
  (aux4) to[bend right=10]
  (aux6) --
  (\offsetright + 0,6) -- 
  (\offsetright + 0,0) -- cycle;
\end{scope}

\begin{scope}
\path
  coordinate (aux0) at (\offsetrighttwo + 0,2)
  coordinate (aux1) at (\offsetrighttwo + 0,3)
  coordinate (aux2) at (\offsetrighttwo + 10,10)
  coordinate (aux3) at (\offsetrighttwo + 9,6)
  coordinate (aux4) at (\offsetrighttwo + 2,0)
  coordinate (aux5) at (\offsetrighttwo + 5,0)
  coordinate (aux6) at (\offsetrighttwo + 2,6)
  coordinate (aux7) at (\offsetrighttwo + 5,6)
  coordinate (esp1) at (\offsetrighttwo + 0.2,2.5)
  coordinate (esp2) at (\offsetrighttwo + 1.5,1.5)
  coordinate (esp3) at (\offsetrighttwo + 3,0.1)
  coordinate (esp4) at (\offsetrighttwo + 5.5,1.1)
  coordinate (esp5) at (\offsetrighttwo + 8,0.5)
  coordinate (esp6) at (\offsetrighttwo + 8.75,2)
  coordinate (esp7) at (\offsetrighttwo + 9.7,3)
  coordinate (esp8) at (\offsetrighttwo + 6.5,4.5)
  coordinate (esp9) at (\offsetrighttwo + 3.8,5.8)
  coordinate (esp10) at (\offsetrighttwo + 1.5,4)
  ;
\draw[line width=0.8pt]
  (esp1) to[out=-90,in=170]
  (esp2) to[out=-10,in=170]
  (esp3) to[out=-10,in=180]
  (esp4) to[out=0,in=180]
  (esp5) to[out=10,in=-150]
  (esp6) to[out=20,in=-90]
  (esp7) to[out=90,in=-60]
  (esp8) to[out=120,in=0]
  (esp9) to[out=180,in=0]
  (esp10) to[out=180,in=90]
  cycle;    
\clip
  (esp1) to[out=-90,in=170]
  (esp2) to[out=-10,in=170]
  (esp3) to[out=-10,in=180]
  (esp4) to[out=0,in=180]
  (esp5) to[out=10,in=-150]
  (esp6) to[out=20,in=-90]
  (esp7) to[out=90,in=-60]
  (esp8) to[out=120,in=0]
  (esp9) to[out=180,in=0]
  (esp10) to[out=180,in=90]
  cycle;    
\filldraw[fill=cyan!40]
  (aux4) to[bend right=10]
  (aux6) --
  (aux7) to[bend left=10]
  (aux5) -- cycle;
\filldraw[fill=orange!60]
  (aux5) to[bend right=10]
  (aux7) --
  (\offsetrighttwo + 10,6) --
  (\offsetrighttwo + 10,0) -- cycle;
\filldraw[fill=red!40]
  (aux6) -- 
  (aux1) to[bend right=20]
  (aux3) --
  (aux2) to[bend left=10] cycle;
\filldraw[fill=yellow!50]
  (\offsetrighttwo + 0,0) -- 
  (aux4) to[bend right=10]
  (aux6) --
  (\offsetrighttwo + 0,6) -- 
  (\offsetrighttwo + 0,0) -- cycle;
\end{scope}
\node at (5,\offsetdown - 1.2) {$(\Phi_{V^t}, V^t)$};
\node at (\offsetright + 5,\offsetdown - 1.2) {$\Phi_{V^{t+1}}$};
\node at (\offsetrighttwo + 5,\offsetdown - 1.2) {$(\Phi_{V^{t+1}}, V^{t+1})$};

\draw[-{Latex[length=2mm]},thick] (11, 3) -- (14, 3);
\draw[-{Latex[length=2mm]},thick] (11 + \offsetright, 3) -- (14 + \offsetright, 3);
\node at (12.5, 4.5) {$(a)$};
\node at (12.5 + \offsetright, 4.5) {$(b)$};
\end{tikzpicture}
\caption{B-PWC VI in \algoref{alg:PWC-value-VI}.
$(a)$ Find new BFCP $\Phi_{V^{t+1}}$: refine $\Phi_P + \sum_{\alpha \in A} \Phi_R^{\alpha}$ to be a pre-image BFCP of $\Phi_{V^{t}}$ for $\delta$; $(b)$ compute a value for each $\phi \in \Phi_{V^{t+1}}$: take one state $s \in \phi$ and compute $V^{t+1}$ by assigning to each region $\phi$ the value $[TV^t](s)$.}
\label{fig:VI-graph}
\vspace*{-0.5em}
\end{figure}
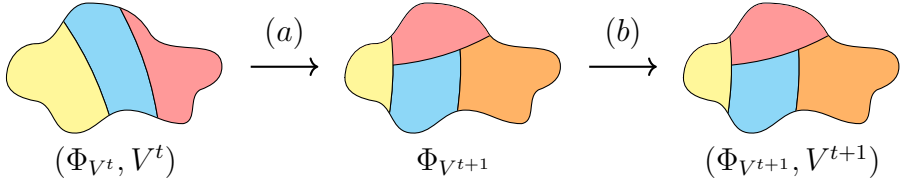

%% file: figure/region_refinement.tex
\newcommand{\grid}[2]{
    \node at (#1 + 3.95, #2 + -0.3) {$4$};
    \node at (#1 + -0.2, #2 + -0.3) {$0$};
    \node at (#1 + -0.25, #2 + 3.95) {$4$};
    \node at (#1 + -0.25, #2 + 3) {$3$};
    \node at (#1 + -0.25, #2 + 2) {$2$};
    \node at (#1 + -0.25, #2 + 1) {$1$};
    \node at (#1 + 3, #2 + -0.3) {$3$};
    \node at (#1 + 2, #2 + -0.3) {$2$};
    \node at (#1 + 1, #2 + -0.3) {$1$};
    \draw[black, very thick] (#1, #2 + 0) rectangle (#1 + 4, #2 + 4);
    \draw[black, thick] (#1, #2 + 1) -- (#1 + 4, #2 + 1);
    \draw[black, thick] (#1, #2 + 2) -- (#1 + 4, #2 + 2);
    \draw[black, thick] (#1, #2 + 3) -- (#1 + 4, #2 + 3);
    \draw[black, thick] (#1 + 1, #2 + 0) -- (#1 + 1, #2 + 4);
    \draw[black, thick] (#1 + 2, #2 + 0) -- (#1 + 2, #2 + 4);
    \draw[black, thick] (#1 + 3, #2 + 0) -- (#1 + 3, #2 + 4);
}

\begin{figure}
\centering
    \begin{tikzpicture}[scale=0.75]

        \def\xi{0}
        \def\yi{0}
        \draw[fill = gray!20] (\xi + 1, \yi + 1) rectangle (2,2);
        \draw[black, thin, fill = magenta!30] (\xi + 0, \yi + 1) rectangle (\xi + 1, \yi + 2); 
        \draw[black, thin, fill = blue!30] (\xi + 1, \yi + 0) rectangle (\xi + 2, \yi + 1); 
        \draw[black, thin, fill = green!30] (\xi + 1, \yi + 2) rectangle (\xi + 2, \yi + 3); 
        \draw[black, thin, fill = orange!30] (\xi + 2, \yi + 1) rectangle (\xi + 3, \yi + 2); 
        \grid{\xi}{\yi}	
        \node [sedan top,body color=red!30,window color=black!80,minimum width=0.75cm, rotate = 90] at (\xi + 1.5, \yi + 1.75) {};
        \filldraw[color=red] (\xi + 1.5, \yi + 1.75) circle (1.5pt);
        \draw[->] (\xi + 1.5, \yi + 1.75) -- (\xi + 1.5, \yi + 2.5) node [pos=1.625, below, xshift=-1.25] {\tiny \textit{up}};
        \draw[->] (\xi + 1.5, \yi + 1.75) -- (\xi + 2.5, \yi + 1.75) node [pos=1.3, below, xshift=-5.0] {\tiny \textit{right}};
        \draw[->] (\xi + 1.5, \yi + 1.75) -- (\xi + 0.5, \yi + 1.75) node [pos=0.9, below, xshift=-5.0] {\tiny \textit{left}};
        \draw[->] (\xi + 1.5, \yi + 1.75) -- (\xi + 1.5, \yi + 0.5) node [pos=0.9, below, xshift=-1.25] {\tiny \textit{down}};
        \draw[-{Latex[length=1.8mm]}] (\xi + 4.5, \yi + 2.0) -- (\xi + 5.2, \yi + 2.0) node [midway, above, yshift=1mm] {\scriptsize \textit{up}};
        \def\xi{6}
        \def\yi{0}
        \draw[black, thin, fill = magenta!30] (\xi + 0, \yi + 1) rectangle (\xi + 1, \yi + 2); 
        \draw[black, thin, fill = blue!30] (\xi + 1, \yi + 0) rectangle (\xi + 2, \yi + 1); 
        \draw[black, thin, fill = green!30] (\xi + 1, \yi + 2) rectangle (\xi + 2, \yi + 3); 
        \draw[black, thin, fill = orange!30] (\xi + 2, \yi + 1) rectangle (\xi + 3, \yi + 2); 
        \grid{\xi}{\yi}	
        \draw[black, thin, pattern=north west lines] (\xi + 1.0, \yi + 1.5) rectangle (\xi + 2.0, \yi + 2.5); 
        \filldraw[color=red] (\xi + 1.5, \yi + 1.75) circle (1.5pt);
        \draw[-{Latex[length=1.8mm]}] (\xi + 4.5, \yi + 2.0) -- (\xi + 5.2, \yi + 2.0) node [midway, above, yshift=1mm] {\scriptsize \textit{down}};
        \def\xi{12}
        \def\yi{0}
        \draw[black, thin, fill = magenta!30] (\xi + 0, \yi + 1) rectangle (\xi + 1, \yi + 2); 
        \draw[black, thin, fill = blue!30] (\xi + 1, \yi + 0) rectangle (\xi + 2, \yi + 1); 
        \draw[black, thin, fill = green!30] (\xi + 1, \yi + 2) rectangle (\xi + 2, \yi + 3); 
        \draw[black, thin, fill = orange!30] (\xi + 2, \yi + 1) rectangle (\xi + 3, \yi + 2); 
        \grid{\xi}{\yi}	
        \draw[pattern=north west lines, pattern color = green] (\xi + 1.0, \yi + 1.5) rectangle (\xi + 2.0, \yi + 2.0);
        \draw[black, thin, pattern=north east lines] (\xi + 1.0, \yi + 1.5) rectangle (\xi + 2.0, \yi + 0.5); 
        \filldraw[color=red] (\xi + 1.5, \yi + 1.75) circle (1.5pt);
        \draw[-{Latex[length=1.8mm]}] (\xi + 2.0, \yi - 0.75) -- (\xi + 2.0, \yi - 1.25)  -- (\xi - 10.0, \yi - 1.25) node [midway, above] {\scriptsize \textit{division after up and down}} -- (\xi - 10.0, \yi - 1.77);

        \def\xi{0}
        \def\yi{-6}
        \draw[black, thin, fill = magenta!30] (\xi + 0, \yi + 1) rectangle (\xi + 1, \yi + 2); 
        \draw[black, thin, fill = blue!30] (\xi + 1, \yi + 0) rectangle (\xi + 2, \yi + 1); 
        \draw[black, thin, fill = green!30] (\xi + 1, \yi + 2) rectangle (\xi + 2, \yi + 3); 
        \draw[black, thin, fill = orange!30] (\xi + 2, \yi + 1) rectangle (\xi + 3, \yi + 2); 
        \grid{\xi}{\yi}	
        \draw[pattern=north west lines, pattern color = green] (\xi + 1.0, \yi + 1.5) rectangle (\xi + 2.0, \yi + 2.0);
        \draw[pattern=north east lines, pattern color = blue] (\xi + 1.0, \yi + 1.5) rectangle (\xi + 2.0, \yi + 1.0);
        \filldraw[color=red] (\xi + 1.5, \yi + 1.75) circle (1.5pt);
        %
        %
        \def\xi{0}
        \def\yi{-4}
        \draw[-{Latex[length=1.8mm]}] (\xi + 6.0, \yi) -- (\xi + 10.0, \yi) node [pos=0.5, below, xshift=-1.25] {$\cdots$} node [midway, above, yshift=1mm] {\scriptsize \textit{left, right}};
        \def\xi{12}
        \def\yi{-6}
        \draw[black, thin, fill = magenta!30] (\xi + 0, \yi + 1) rectangle (\xi + 1, \yi + 2); 
        \draw[black, thin, fill = blue!30] (\xi + 1, \yi + 0) rectangle (\xi + 2, \yi + 1); 
        \draw[black, thin, fill = green!30] (\xi + 1, \yi + 2) rectangle (\xi + 2, \yi + 3); 
        \draw[black, thin, fill = orange!30] (\xi + 2, \yi + 1) rectangle (\xi + 3, \yi + 2); 
        \grid{\xi}{\yi}	
        \draw[pattern=north west lines, pattern color = green] (\xi + 1.0, \yi + 1.5) rectangle (\xi + 2.0, \yi + 2.0);
        \draw[pattern=north east lines, pattern color = blue] (\xi + 1.0, \yi + 1.5) rectangle (\xi + 2.0, \yi + 1.0);
        \filldraw[color=red] (\xi + 1.5, \yi + 1.75) circle (1.5pt);
        \draw[pattern=vertical lines, pattern color = magenta] (\xi + 1.0, \yi + 1.0) rectangle (\xi + 1.5, \yi + 2.0);
        \draw[pattern=horizontal lines, pattern color = orange] (\xi + 1.5, \yi + 1.0) rectangle (\xi + 2.0, \yi + 2.0);
    \end{tikzpicture}
    \caption{Illustration of the region refinement process.}
    \label{fig:refin}
\end{figure}

%% file: simulation.tex


\begin{examp}\label{eg5-eg} 
We now return to the NS-CSG model, presented in \egref{eg1-eg},  
of a dynamic vehicle parking problem with the perception functions implemented via the linear regression model given in \egref{eg3-eg}. To demonstrate {\revise the practicality}
of our approach we synthesise strategies using a prototype Python implementation of the B-PWC VI algorithm. 
\input{figure/optimal_strategy}
The implementation uses a polyhedral representation of regions and the values of the zero-sum normal-form games involved in the minimax operator at step 8 of Algorithm~\ref{alg:PWC-value-VI} are found by solving the corresponding linear program \cite{NMK+44} using the SciPy library~\cite{2020SciPy-NMeth}. We have partitioned the state space of the game into two sets corresponding to the two possible local states of $\agent_1$. The B-PWC VI algorithm converges after $46$ iterations when $\varepsilon = 10^{-6}$ and takes $3,825$s to complete. For each set in the partition of the state space, the BFCP of this set converges to the product of two $8 \times 8$ grids. 
For the currently preferred 
parking spot of $\agent_1$ (red square) and coordinate of $\agent_2$ (purple triangle), the value function with respect to the coordinate of $\agent_1$ is presented in  Fig.~\ref{fig:value_function_vehicle} (left)
and shows that, the closer $\agent_1$ is to its preferred 
parking spot, the higher the (approximate) optimal value. The lightest-colour class is caused by an immediate crash, and its position follows from the observation function. 





An (approximately) optimal strategy for $\agent_1$ is presented in Fig.~\ref{fig:value_function_vehicle} (right), where the colour intensity of an arrow is proportional to the probability of moving in that direction and the rotating arrow represents the parking action.
{\revise 
We see that there are several choices that may not at first appear intuitive. 
As an example, $\agent_1$ moves left when in cell $[1.0, 1.5] \times [3.5, 4.0]$ (top left), which is away from its current chosen parking spot. However, since $\agent_2$ is closer to its current parking spot, keeping away from this parking spot will avoid a crash with $\agent_2$. In addition, under $\agent_1$'s transition function, the chosen parking spot will change with probability $0.5$ at each step, so with high probability $\agent_1$ will be able to park before $\agent_2$ can reach it and a crash can occur.} \hfill$\blacksquare$
\end{examp}

%% file: figure/optimal_strategy.tex
\begin{figure}
    \centering
    \includegraphics[scale=0.35]{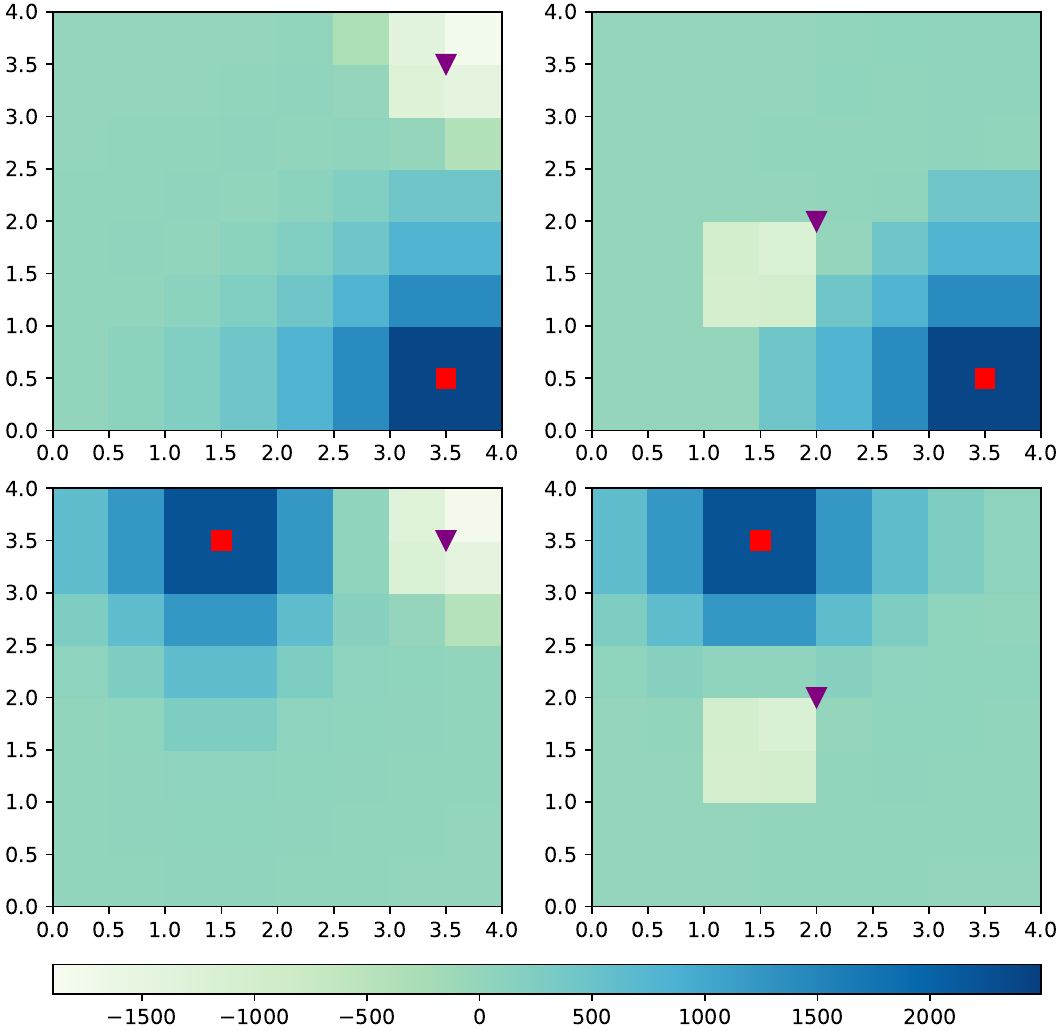}
    \hfil
    \includegraphics[scale=0.35]{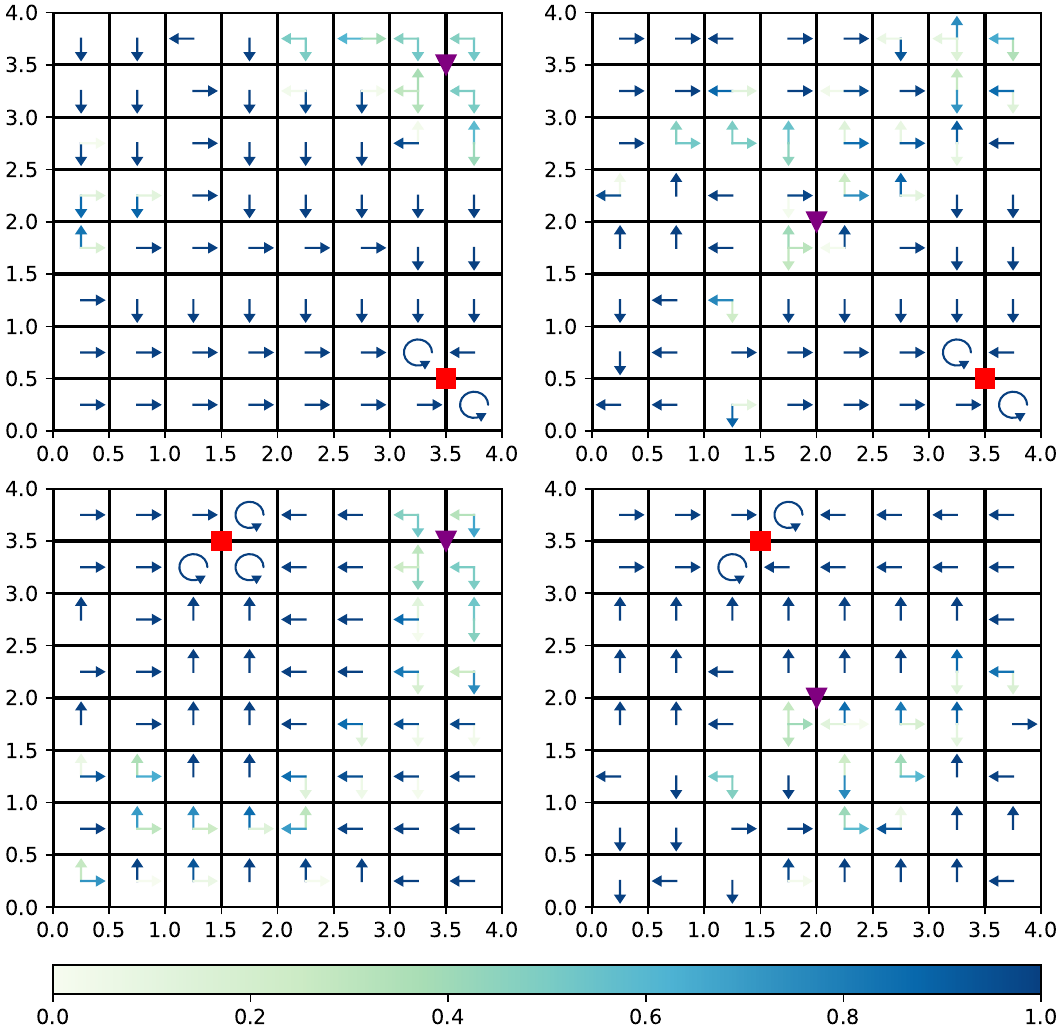}
    \caption{
    Strategy synthesis for \egref{eg5-eg}. Value function (left) and optimal strategy (right) over different coordinates of $\agent_1$ for a fixed local state of $\agent_1$ (red square) and a fixed coordinate of $\agent_2$ (purple triangle).}
    \label{fig:value_function_vehicle}
\end{figure}


%% file: policy_iteration.tex
\section{Policy iteration}\label{pi-sect}


It is known that, for MDPs, PI algorithms generally converge faster than VI algorithms, since policy improvement can jump over policies directly~\cite{DB:22}.
Motived by this fact, in this section we show how PI can be used to approximate
the values and optimal strategies of an NS-CSG $\csg$ with respect to a  discounted accumulated reward objective $Y$.
Our algorithm takes ideas from recent work \cite{DB:21}, which proposed a new PI method to solve zero-sum stochastic games with finite state spaces, and is the first PI algorithm for CSGs with Borel state spaces and with a convergence guarantee. Our PI algorithm ensures that the strategies and value functions generated during each iteration never leave a finitely representable class of functions. In addition, when computing values of CSGs, efficiencies are gained over alternative algorithms as there is no need to solve normal-form games, which is required by our B-PWC VI and Pollatschek-Avi-Itzhak's PI algorithm~\cite{MAP-BAI:69}, nor to solve MDPs, which adds complexity to Hoffman-Karp's PI algorithm~\cite{AJH-RMK:66}. This  results in cheaper computations and faster convergence over these alternatives, as for PI over VI for MDPs. 

\subsection{Operators, functions and solutions}
Before presenting the algorithm, the following operators, functions and solutions are proposed. Let $\gamma \in \mathbb{R}$ be a constant such that $\gamma>1$ and $\gamma\beta<1$, which will be used to distribute the discount factor $\beta$ between policy evaluation and policy improvement of the two agents. 

\startpara{Operators for Max-Min and Min-Max} Before introducing operators for Max-Min and Min-Max, we require the notion of a stationary Stackelberg (follower) strategy for $\agent_2$, which is a stochastic kernel $\overline{\sigma}_2:\Lambda_1 \to \mathbb{P}(A_2)$, i.e.,  $\overline{\sigma}_2 \in \mathbb{P}(A_2 \mid \Lambda_1)$ such that $\overline{\sigma}_2(A_2(s) \mid (s,u_1))=1$ for $(s,u_1)\in \Lambda_1$. This strategy is introduced only for the PI algorithm and implies that $\agent_2$ makes decisions conditioned on the current state $s$ and the current choice of $\agent_1$, i.e.\ action distribution $u_1$, and thus allows us to split the maximum and minimum operations of the two agents. 
We denote by $\overline{\Sigma}_2$ the set of all stationary Stackelberg strategies for $\agent_2$.

\begin{defi}[Operator for the Max-Min value]\label{def:operator-H1-definition}
For strategy $\sigma_1 \in \Sigma_1$ of $\agent_1$ and
function $V_2 \in \mathbb{F}(\Lambda_1)$, we define the operator $H_{\sigma_1,V_2}^1 : \mathbb{F}(\Lambda_1) \rightarrow \mathbb{F}(S)$ such that for $J_2 \in \mathbb{F}(\Lambda_1)$ and $s \in S$: 
\begin{align*}
    [H_{\sigma_1,V_2}^1J_2](s) & = \; \gamma^{-1}\min\{J_2(s,\sigma_1(s)), V_2(s, \sigma_1(s))\}  \\
    & = \; \gamma^{-1}\min\{J_2(s,u_1), V_2(s, u_1)\}
\end{align*}
where $\sigma_1(s)=u_1 \in \mathbb{P}(A_1(s))$. 
\end{defi}
\begin{defi}[Operator for the Min-Max value]\label{def:operator-H2-definition}
For Stackelberg (follower) strategy $\overline{\sigma}_2 \in \overline{\Sigma}_2$ of $\agent_2$ and function $V_1 \in \mathbb{F}(S)$, we define the operator $H^2_{\overline{\sigma}_2,V_1} : \mathbb{F}(S) \rightarrow \mathbb{F}(\Lambda_1)$ such that for $J_1 \in \mathbb{F}(S)$ and $(s, u_1) \in \Lambda_1$:
\begin{align*}
   [H^2_{\overline{\sigma}_2,V_1}J_1](s,u_1) & = \mbox{$\sum\nolimits_{(a_1,a_2) \in A(s)}$}  Q(s,(a_1,a_2),\gamma \max[J_1, V_1])u_1(a_1)\overline{\sigma}_2(a_2 | (s,u_1)) \\ 
   & = \mbox{$\sum\nolimits_{(a_1,a_2) \in A(s)}$} Q(s,(a_1,a_2),\gamma\max[J_1, V_1])u_1(a_1)u_2(a_2)
\end{align*}
where $\overline{\sigma}_2(\; \cdot \mid (s,u_1)) = u_2 \in \mathbb{P}(A_2(s))$. 
\end{defi}
%
%
Unlike the classical PI algorithms by Hoffman and Karp~\cite{AJH-RMK:66} and Pollatschek and Avi-Itzhak \cite{MAP-BAI:69}, following~\cite{DB:21},
our PI algorithm separates the policy evaluation and policy improvement of the maximiser ($\agent_1$) and the minimiser ($\agent_2$) through the use of the operators of \defiref{def:operator-H1-definition} and \defiref{def:operator-H2-definition}, respectively.
To track the value functions after performing policy evaluation of $\agent_1$ and $\agent_2$, 
our PI algorithm introduces value functions $J_1$ and $J_2$. 
In addition, the value functions $V_1$ and $V_2$ are introduced to avoid the oscillatory behavior of the Pollatschek and Avi-Itzhak PI algorithm~\cite{MAP-BAI:69}, thus ensuring convergence, and are updated only during policy improvement. The role of $\gamma$ is to split the discount factor $\beta$ such that all the operators corresponding to policy evaluation and policy improvement of the two agents are contraction mappings, which then ensures convergence. 

\startpara{Two function representations} We next define two classes of functions, which play a key role in characterizing the functions and strategies generated during each iteration of our PI algorithm. 
\begin{defi}[CON-PWL Borel measurable function]\label{defi:s-CON-u1-PWL}
A function $f \in \mathbb{F}(\Lambda_1)$ is a \emph{constant-piecewise-linear (CON-PWL) Borel measurable function} if there exists a BFCP $\Phi$ of $S$ such that, for each $\phi \in \Phi$, $A_1(s) = A_1(s')$ for $s, s' \in \phi$, and 
$\Phi$ generates $\Theta = \{ \theta(\phi) \mid \phi \in \Phi \} $ where $\theta(\phi)=\{(s,u_1) \in \Lambda_1 \mid s \in \phi\}$, a BFCP of $\Lambda_1$, such that for $\theta(\phi) \in \Theta$:
\begin{enumerate}[(i)]
    \item $f({}\cdot{},u_1): \phi \to \mathbb{R}$ is constant for $u_1 \in \mathbb{P}(A_1(s))$ where $s \in \phi$;
    \item $f(s,{}\cdot{}) : \mathbb{P}(A_1(s)) \to \mathbb{R}$ is B-PWL for $s \in \phi$.
\end{enumerate}
\end{defi}

\begin{defi}[CON-PWC stochastic kernel]\label{defi:s-CON-u1-PWC}
A function $f \in \overline{\Sigma}_2$ is a \emph{constant-piecewise-constant (CON-PWC) stochastic kernel} if there exists a BFCP $\Phi$ of $S$ such that, for each $\phi \in \Phi$, $A(s) = A(s')$ for $s, s' \in \phi$, and 
$\Phi$ generates $\Theta=\{ \theta(\phi) \mid \phi \in \Phi \}$ where $\theta(\phi)=\{(s,u_1) \in \Lambda_1 \mid s \in \phi \}$, a BFCP of $\Lambda_1$, such that for $\theta(\phi) \in \Theta$:
\begin{enumerate}[(i)]
    \item $f({}\cdot{},u_1): \phi \to \mathbb{P}(A_2(s))$ is constant for $u_1 \in \mathbb{P}(A_1(s))$ where $s \in \phi$;
    \item $f(s,{}\cdot{}) : \mathbb{P}(A_1(s)) \to \mathbb{P}(A_2(s))$ is B-PWC for $s \in \phi$.
\end{enumerate}
\end{defi}
\begin{figure}[t]
    \centering
    \includegraphics[width=45mm,height=35mm]{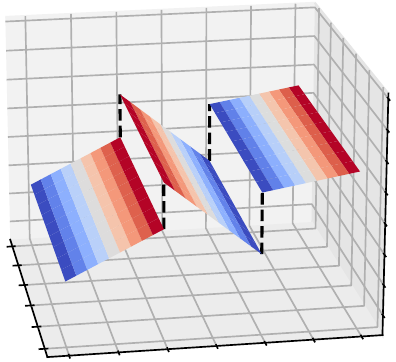} \hfil
    \includegraphics[width=45mm,height=35mm]{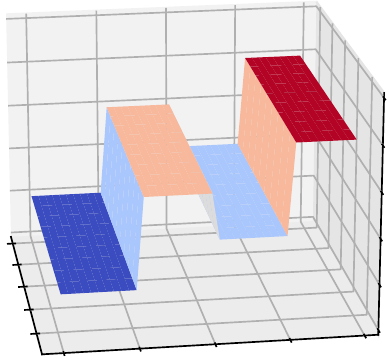}
    \put(-232, -6){$u_1$}
    \put(-305, 14){$s$}
    \put(-295, 81){$f_1(s,u_1)$}
    \put(-58, -6){$u_1$}
    \put(-131, 14){$s$}
    \put(-119, 81){$f_2(s,u_1)$}
    \caption{Two functions over one region of a BFCP of $\Lambda_1$. Constant-piecewise-linear (CON-PWL) Borel measurable function (left): given $u_1$, $f_1(s, u_1)$ is constant in $s$, and given $s$, $f_1(s, u_1)$ is B-PWL in $u_1$. Constant-piecewise-constant (CON-PWC) stochastic kernel (right): given $u_1$, $f_2(s, u_1)$ is constant in $s$, and given $s$, $f_2(s, u_1)$ is B-PWC in $u_1$.}
    \label{fig:CON-PWC-PWL}
\end{figure}
Fig. \ref{fig:CON-PWC-PWL} presents an example of a CON-PWL Borel measurable function and CON-PWC stochastic kernel over a region. We now show that these two functions can be represented by finite sets of vectors.
Each CON-PWL Borel measurable function $f$ can be represented by a finite set of vectors $\{(D_{\phi, \phi'}, b_{\phi, \phi'}) \in \mathbb{R}^{|A_1|} \times \mathbb{R} \mid \phi \in \Phi \wedge \phi' \in \Phi'(\phi)\}$ such that $f(s, u_1) = D_{\phi, \phi'}^\top u_1 + b_{\phi, \phi'}$ for $s \in \phi$ and $u_1 \in \phi'$, where $\Phi$ is a BFCP of $S$ for $f$ using \defiref{defi:s-CON-u1-PWL} and $\Phi'(\phi)$ is a BFCP of $\{ u_1 \in  \mathbb{P}(A_1) \mid (s, u_1) \in \theta(\phi)\} $, and $\theta(\phi) \in \Theta$ again using \defiref{defi:s-CON-u1-PWL} is such that, over each region $\phi' \in \Phi'(\phi)$, $f(s, u_1)$ is linear in $u_1$ given $s \in \phi$. 
Similarly using \defiref{defi:s-CON-u1-PWC}, each CON-PWC stochastic kernel $f$ can be represented by a finite set of vectors $\{ D_{\phi, \phi'} \in \mathbb{P}(A_2) \mid \phi \in \Phi \wedge\phi' \in \Phi'(\phi)\}$ such that $f(s, u_1) = D_{\phi, \phi'}$ for $s \in \phi$ and $u_1 \in \phi'$,
where $\Phi$ is a BFCP of $S$ for $f$ using \defiref{defi:s-CON-u1-PWC}, $\Phi'(\phi)$ is a BFCP of $\{ u_1 \in  \mathbb{P}(A_1) \mid (s, u_1) \in \theta(\phi)\} $, $\theta(\phi) \in \Theta$ using \defiref{defi:s-CON-u1-PWC} is such that, over each region $\phi' \in \Phi'(\phi)$, $f(s, u_1)$ is constant in $u_1$ given $s \in \phi$.

\startpara{Maximum or minimum solutions}
We introduce a criterion for selecting the maximum or minimum solution over a region, by which the strategies from policy improvement are finitely representable.

\begin{defi}[CON-$1$ solution]\label{defi:con1-solution} Let $f \in \mathbb{F}(\Lambda_1)$ be a CON-PWL Borel measurable function. Using \defiref{defi:s-CON-u1-PWL} there exists a BFCP $\Phi$  of $S$ for $f$. Now, for each $\phi \in \Phi$, if there exists $u_1^{\phi} \in \mathbb{P}(A_1(s))$ such that:
\[
    f(s,u_1^{\phi}) = \max\nolimits_{u_1 \in \mathbb{P}(A_1(s))} f(s,u_1)
\]
for $s\in \phi$, and $\sigma_1$ is a strategy of $\agent_1$ such that $\sigma_1(s) = u_1^{\phi}$ for $s\in \phi$, then $\sigma_1$ is a \emph{constant-$1$ (CON-$1$) solution} of $f$ over $\phi$. 
\end{defi}

\begin{defi}[CON-$2$ solution]\label{defi:con2-solution} Let $f \in \mathbb{F}(\Lambda_{12})$ be a Borel measurable function. If there exists a BFCP $\Theta$ of $\Lambda_1$ where, for each $\theta \in \Theta$, $A_2(s)$ is constant for $(s, u_1) \in \theta$ and there exists $u_2^{\theta} \in \mathbb{P}(A_2(s))$ such that:
\[
f(s,u_1,u_2^{\theta}) = \min\nolimits_{u_2 \in \mathbb{P}(A_2(s))} f(s,u_1,u_2)
\]
for $(s,u_1) \in \theta$, and $\overline{\sigma}_2$ is a Stackelberg strategy for $\agent_2$ such that $\overline{\sigma}_2(s,u_1) = u_2^{\theta}$ for $(s,u_1) \in \theta$, then $\overline{\sigma}_2$ is a \emph{constant-$2$ (CON-$2$) solution} of $f$ over $\theta$.
\end{defi}

\subsection{Minimax-action-free PI}
{\revise We now use the operators for the Max-Min and Min-Max values (Definitions \ref{def:operator-H1-definition} and \ref{def:operator-H2-definition})} 
to derive a PI algorithm called \emph{Minimax-action-free PI} (\algoref{alg:maxmin-free-PI}) for strategy synthesis for NS-CSGs with Borel state spaces. Our algorithm closely follows the PI method of \cite{DB:21} for finite state spaces, but has to resolve a number of issues due to the uncountability of the underlying state space and the need to ensure Borel measurability at each iteration. To overcome these issues we
(i) introduce CON-PWL Borel measurable functions (\defiref{defi:s-CON-u1-PWL}) and CON-PWC Borel measurable strategies (\defiref{defi:s-CON-u1-PWC}) to ensure measurability and finite representability;
(ii) work with CON-$1$ and CON-$2$ solutions (\defiref{defi:con1-solution} and (\defiref{defi:con2-solution}) for policy improvement to ensure that the strategies generated are finitely representable and consistent; and  (iii) propose a BFCP iteration algorithm (\algoref{alg:BFCP-iteration-PI}) and a BFCP-based computation algorithm 
(\algoref{alg:BFCP-based-computation-PI}) to compute a new BFCP of the state space and the values or strategies over this BFCP.
We also provide a simpler proof than that presented in~\cite{DB:21}, which does not require the introduction of any new concepts except those used in the algorithm.

\begin{algorithm}[t]
\caption{Iteration $t$ of Minimax-action-free PI}\label{alg:maxmin-free-PI}
\begin{algorithmic}[1]
\State{\textbf{Input:}  NS-CSG $\csg$, PWC $\sigma_1^t \in \Sigma_1$, CON-PWC $\overline{\sigma}_2^t \in \overline{\Sigma}_2$, PWC $J_1^t, V_1^t\in \mathbb{F}(S)$, CON-PWL $J_2^t,V_2^t \in \mathbb{F}(\Lambda_1)$} 
\State {\bf Perform one of the following four iterations.} 
\State {\quad Policy evaluation of $\agent_1$:}
\State {\qquad $J_1^{t+1} \leftarrow [H_{\sigma_1^t,V_2^t}^1J_2^t]$ via $\mathit{PE1}$, $\sigma_1^{t+1} \leftarrow \sigma_1^t$, }
\State {\qquad $V_1^{t+1} \leftarrow V_1^t$, $\overline{\sigma}_2^{t+1} \leftarrow \overline{\sigma}_2^{t}$, $J_2^{t+1} \leftarrow J_2^t$, $V_2^{t+1} \leftarrow V_2^t$}
\State {\quad Policy improvement of $\agent_1$ by CON-$1$ solution:}
\State {\qquad $\sigma_1^{t+1}(s) \in \argmax_{u_1 \in \mathbb{P}(A_1(s))} [H_{u_1,V_2^t}^1J^t_2](s)$,}
\State {\qquad $V_1^{t+1} \leftarrow [H_{\sigma_1^{t+1},V_2^t}^1J_2^t]$ via $\mathit{PI1}$,}
\State {\qquad $J_1^{t+1} \leftarrow J_1^t$,
$\overline{\sigma}_2^{t+1} \leftarrow \overline{\sigma}_2^{t}$, $J_2^{t+1} \leftarrow J_2^t$, $V_2^{t+1} \leftarrow V_2^t$}
\State {\quad Policy evaluation of $\agent_2$:}
\State {\qquad $J_2^{t+1} \leftarrow [H_{\overline{\sigma}_2^t,V_1^t}^2J_1^t]$ via $\mathit{PE2}$, $\sigma_1^{t+1} \leftarrow \sigma_1^t$,}
\State {\qquad $J_1^{t+1} \leftarrow J_1^t$, $V_1^{t+1} \leftarrow V_1^t$, $\overline{\sigma}_2^{t+1} \leftarrow \overline{\sigma}_2^{t}$, $V_2^{t+1} \leftarrow V_2^t$}
\State {\quad Policy improvement of $\agent_2$ by CON-$2$ solution:}
\State {\qquad $\overline{\sigma}_2^{t+1}(s,u_1) \in \argmin\nolimits_{u_2 \in \mathbb{P}(A_2(s))} [H_{u_2,V_1^t}^2J_1^t](s,u_1)$,}
\State {\qquad $V_2^{t+1} \leftarrow [H_{\overline{\sigma}_2^{t+1},V_1^t}^2J_1^t]$ via $\mathit{PI2}$,}
\State {\qquad $\sigma_1^{t+1} \leftarrow \sigma_1^t$,
$J_1^{t+1} \leftarrow J_1^t$, $V_1^{t+1} \leftarrow V_1^t$, $J_2^{t+1} \leftarrow J_2^t$}
\State {$t \leftarrow t+1$}
\end{algorithmic}
\end{algorithm}

\begin{algorithm}
\caption{BFCP iteration $t$ for Minimax-action-free PI}\label{alg:BFCP-iteration-PI}
\begin{algorithmic}[1]
\State{\textbf{Input:} Perception FCP $\Phi_{P}$, reward FCPs $(\Phi_R^{\alpha} )_{\alpha \in A}$}
\State{\textbf{Output:} BFCPs $\langle \Phi_{J_1^{t}}, \Phi_{V_1^{t}}, \Phi_{\sigma_1^{t}}, \Theta_{J_2^{t}}, \Theta_{V_2^{t}}, \Theta_{\overline{\sigma}_2^{t}}\rangle_{t \in \mathbb{N}}$ 

\qquad \  \ for $\langle J_1^t, V_1^t, \sigma_1^t, J_2^{t}, V_2^{t}, \overline{\sigma}_2^{t} \rangle_{t \in \mathbb{N}}$
}
\State {$\Phi_{J_1^0}, \Phi_{V_1^0}, \Phi_{\sigma_1^0} \leftarrow \{S\}$, \ $\Theta_{J_2^0}, \Theta_{V_2^0}, \Theta_{\overline{\sigma}_2^0} \leftarrow \{\Lambda_1\}$}

\While{\algoref{alg:maxmin-free-PI} performs iteration $t$}
\If{policy evaluation of $\agent_1$ is chosen}
\State {$\mathit{Preprocess\_maximiser}()$,}

\State {$\Phi_{J_1^{t+1}} \leftarrow \Phi_{\sigma_1^t} + \Phi_{J_2^t} + \Phi_{V_2^t}$, \ $\Phi_{V_1^{t+1}} \leftarrow \Phi_{V_1^{t}}$, \ $\Phi_{\sigma_1^{t+1}} \leftarrow \Phi_{\sigma_1^t}$}
\EndIf

\If{policy improvement of $\agent_1$ is chosen
}
\State {$\mathit{Preprocess\_maximiser}()$,}

\State {$\Phi_{\sigma_1^{t+1}} \leftarrow \Phi_{J_2^t} + \Phi_{V_2^t} $, \ $\Phi_{V_1^{t+1}} \leftarrow \Phi_{\sigma_1^{t+1}}$, \ $\Phi_{J_1^{t+1}} \leftarrow \Phi_{J_1^{t}}$}
\EndIf

\If{policy evaluation of $\agent_2$ is chosen}
\State {$\mathit{Preprocess\_minimiser}()$,}

\State{$\Phi_{\overline{\sigma}_2^t} \leftarrow \big\{\{ s \mid (s, u_1) \in \theta\} \mid \theta \in \Theta_{\overline{\sigma}_2^t} \big\}$},

\State {$\Theta_{J_2^{t+1}} \leftarrow \big\{\{(s, u_1) \in \Lambda_1 \mid s \in \phi \} \mid \phi \in \Phi_{\hat{Q}^{t+1}} + \Phi_{\overline{\sigma}_2^t}\big\}$,} 

\State {$\Theta_{V_2^{t+1}} \leftarrow \Theta_{V_2^{t}}$, \ $\Theta_{\overline{\sigma}_2^{t+1}} \leftarrow \Theta_{\overline{\sigma}_2^{t}}$}
\EndIf

\If{policy improvement of $\agent_2$ is chosen}
\State {$\mathit{Preprocess\_minimiser}()$,}

\State {$\Theta_{\overline{\sigma}_2^{t+1}} \leftarrow \big\{\{(s, u_1) \in \Lambda_1 \mid s \in \phi \} \mid \phi \in \Phi_{\hat{Q}^{t+1}} \big\}$,}

\State {$\Theta_{V_2^{t+1}} \leftarrow \Theta_{\overline{\sigma}_2^{t+1}}$, \ $\Theta_{J_2^{t+1}} \leftarrow \Theta_{J_2^{t}}$}
\EndIf
\EndWhile
\State \Return $\langle \Phi_{J_1^{t}}, \Phi_{V_1^{t}}, \Phi_{\sigma_1^{t}}, \Theta_{J_2^{t}}, \Theta_{V_2^{t}}, \Theta_{\overline{\sigma}_2^{t}} \rangle_{t \in \mathbb{N}}$ 

\\
\Procedure{$\mathit{Preprocess\_maximiser}()$}{}
\State {$\Theta_{J_2^{t+1}} \leftarrow \Theta_{J_2^t}$, \ $\Theta_{V_2^{t+1}} \leftarrow \Theta_{V_2^t}$, \ $\Phi_{J_2^t} \leftarrow \big\{\{ s \mid (s, u_1) \in \theta\} \mid \theta \in \Theta_{J_2^t} \big\}$, }

\State {$\Phi_{V_2^t} \leftarrow \big\{\{ s \mid (s, u_1) \in \theta\} \mid \theta \in \Theta_{V_2^t} \big\}$, \ $\Theta_{\overline{\sigma}_2^{t+1}} \leftarrow \Theta_{\overline{\sigma}_2^t}$}
\EndProcedure

\Procedure{$\mathit{Preprocess\_minimiser}()$}{}
\State {$\Phi_{J_1^{t+1}} \leftarrow \Phi_{J_1^{t}}$, \ $\Phi_{V_1^{t+1}} \leftarrow \Phi_{V_1^{t}}$, \ $\Phi_{\sigma_1^{t+1}} \leftarrow \Phi_{\sigma_1^t}$,}

\State {$\Phi_{\hat{Q}^{t+1}} \leftarrow \mathit{Preimage\_BFCP}(\Phi_{J_1^t} + \Phi_{V_1^t}, \Phi_{P}, ( \Phi_R^{\alpha} )_{\alpha \in A})$}
\EndProcedure
\end{algorithmic}
\end{algorithm}
\begin{algorithm}
\caption{BFCP based computation for Minimax-action-free PI}\label{alg:BFCP-based-computation-PI}
\begin{algorithmic}[1]
\State{\textbf{Input:} $J_1^t, V_1^t, \sigma_1^t, J_2^{t}, V_2^{t}, \overline{\sigma}_2^{t}, \Phi_{J_1^{t+1}}, \Phi_{\sigma_1^{t+1}}, \Theta_{J_2^{t+1}}, \Theta_{\overline{\sigma}_2^{t+1}}$}

\Procedure{$\mathit{PE1}$}{}
    \For{$\phi \in \Phi_{J_1^{t+1}}$}
    \State {Take one state $s \in \phi$, and then $J_{1, \phi}^{t+1} \leftarrow [H_{\sigma_1^t,V_2^t}^1J_2^t](s)$}
    \EndFor
    \State \Return $J_1^{t+1} \leftarrow ( J_{1, \phi}^{t+1})_{\phi \in \Phi_{J_{\scaleobj{.75}{1}}^{\scaleobj{.75}{t+1}}}}$
\EndProcedure

\Procedure{$\mathit{PI1}$}{}
    \For{$\phi \in \Phi_{\sigma_1^{t+1}}$}
    \State {Take $s \in \phi$, and then $u_1 \in \argmax_{u_1 \in \mathbb{P}(A_1(s))} [H_{u_1,V_2^t}^1J^t_2](s)$}
    
    \State {$\sigma_{1, \phi}^{t+1} \leftarrow u_1$, \ $V_{1, \phi}^{t+1} \leftarrow \max_{u_1 \in \mathbb{P}(A_1(s))} [H_{u_1,V_2^t}^1J^t_2](s)$}
    \EndFor
    \State \Return {$\sigma_1^{t+1}  \leftarrow ( \sigma_{1, \phi}^{t+1})_{\phi \in \Phi_{\sigma_{\scaleobj{.75}{1}}^{\scaleobj{.75}{t+1}}}}$, \ $V_1^{t+1} \leftarrow  ( V_{1, \phi}^{t+1})_{\phi \in \Phi_{\sigma_{\scaleobj{.75}{1}}^{\scaleobj{.75}{t+1}}}}$}
\EndProcedure

\Procedure{$\mathit{PE2}$}{}
\For{$\theta \in \Theta_{J_2^{t+1}}$}
    \State {$\phi \leftarrow \{ s \mid (s, u_1) \in \theta \}$}
    
    \State \multiline{Take $s \in \phi$, and then compute a BFCP $\Phi_{u}$ of $\mathbb{P}(A_1(s))$ such that over $\phi_u \in \Phi_u$, $[H_{\overline{\sigma}_2^t,V_1^t}^2J_1^t](s, u_1)$ is linear in $u_1$}
    
    \State {$J_{2, \theta}^{t+1} \leftarrow [H_{\overline{\sigma}_2^t,V_1^t}^2J_1^t](s, u_1)$ is linear in $u_1$}
    \EndFor
    \State \Return {$J_2^{t+1} \leftarrow ( J_{2, \theta}^{t+1})_{\theta \in \Theta_{J_{\scaleobj{.75}{2}}^{\scaleobj{.75}{t+1}}}}$}
\EndProcedure

\Procedure{$\mathit{PI2}$}{}
\For{$\theta \in \Theta_{\overline{\sigma}_2^{t+1}}$}
    \State {$\phi \leftarrow \{ s \mid (s, u_1) \in \theta \}$}
    
    \State \multiline{Take $s' \in \phi$, and then compute a BFCP $\Phi_{u}$ of $\mathbb{P}(A_1(s'))$ such that over $\phi_u \in \Phi_u$, $
    \min\limits_{u_2 \in \mathbb{P}(A_2(s'))}[H^2_{u_2,V_1^t}J^t_1](s',u_1) $
    is constant for $u_1 \in \phi_u$}
    
    \State {Take $u_1' \in \phi_u$ and $u_2' \in \argmin\limits _{u_2 \in \mathbb{P}(A_2(s'))} [H^2_{u_2,V_1^t}J^t_1](s',u_1')$ for $\phi_u \in \Phi_u$}
    
    \State {$\overline{\sigma}_{2, \theta}^{t+1} \leftarrow u_2'$, \ $V_{2, \theta}^{t+1} \leftarrow [H^2_{u_2',V_1^t}J^t_1](s',u_1)$ is linear in $u_1$}
    \EndFor
    \State \Return {$\overline{\sigma}_2^{t+1} \leftarrow ( \overline{\sigma}_{2, \theta}^{t+1})_{\theta \in \Theta_{\overline{\sigma}_{\scaleobj{.75}{2}}^{\scaleobj{.75}{t+1}}}}$, $V_2^{t+1} \leftarrow ( V_{2, \theta}^{t+1})_{\theta \in \Theta_{\overline{\sigma}_{\scaleobj{.75}{2}}^{\scaleobj{.75}{t+1}}}}$ }
\EndProcedure
\end{algorithmic}
\end{algorithm}

\startpara{Initialization} The Minimax-action-free PI algorithm is initialized with strategies $\sigma_1^0$ and $\overline{\sigma}_2^0$ for each player, which are uniform distributions over available actions/state-action pairs, i.e., $\sigma_1^0(s) = \frac{1}{|A_1(s)|}$ for all $s \in S$ and $\overline{\sigma}_2^0(s, u_1) = \frac{1}{|A_2(s)|}$ for all $(s, u_1) \in \Lambda_1$, and four $0$-valued functions, $J_1^0$, $V_1^0$, $J_2^0$ $V_2^0$, i.e., $J_1^0(s) = V_1^0(s) = 0$ for all $s \in S$ and $J_2^0(s, u_1) = V_2^0(s, u_1) = 0$ for all $(s, u_1) \in \Lambda_1$, 
and \algoref{alg:BFCP-iteration-PI} gives one BFCP for each strategy and function,  

\startpara{The algorithm} An iteration of the Minimax-action-free PI is given in \algoref{alg:maxmin-free-PI}.
As shown later, the order and frequency by which the possible four iterations of \algoref{alg:maxmin-free-PI} are run do not affect the convergence, as long as each is performed infinitely often.
This permits an asynchronous implementation of the Minimax-action-free PI algorithm, as discussed in \cite{DB:21} and for its single-agent counterparts in~\cite{BDP-YH:12}. 

For each of the four iterations, \algoref{alg:BFCP-iteration-PI} provides a way to compute new BFCPs and the results below demonstrate that, over each region of these BFCPs, the corresponding computed strategies and value functions are either constant, PWC or PWL. Therefore, we can follow similar steps to our VI algorithm (see \algoref{alg:PWC-value-VI}) to compute the value functions of these new strategies and value functions (see \algoref{alg:BFCP-based-computation-PI}). The idea is to first compute the BFCPs $\Phi_{J_1^{t + 1}}$, $\Phi_{V_1^{t + 1}}$, $\Phi_{\sigma_1^{t + 1}}$, $\Theta_{J_2^{t + 1}}$, $\Theta_{V_2^{t + 1}}$ and $\Theta_{\overline{\sigma}_2^{t + 1}}$ via \algoref{alg:BFCP-iteration-PI} and then use them to compute strategies and value functions using \algoref{alg:BFCP-based-computation-PI}. For instance, if policy improvement of $\agent_2$ is chosen at iteration $t \in \mathbb{N}$ then we proceed as follows. First, new BFCPs are computed via \algoref{alg:BFCP-iteration-PI}. Second, procedure $\mathit{PI2}$ of \algoref{alg:BFCP-based-computation-PI} is performed. In this step we take each region $\theta \in \Theta_{\overline{\sigma}_2^{t+1}}$, let $\phi = \{ s \mid (s, u_1) \in \theta \}$, then take one state $s' \in \phi$, and compute a BFCP $\Phi_{u}$ of $\mathbb{P}(A_1(s'))$ such that $\min_{u_2 \in \mathbb{P}(A_2(s'))}[H^2_{u_2,V_1^t}J^t_1](s',u_1)$ is constant over $\phi_u \in \Phi_u$ and for $u_1 \in \phi_u$. Third, take one $u_1' \in \phi_u$ and find $u_2' \in \mathbb{P}(A_2(s'))$ that minimises $[H^2_{u_2,V_1^t}J^t_1](s',u_1')$. Fourth, we let $\overline{\sigma}_2^{t+1}(s, u_1) = u_2'$ for $s \in \phi$ and $u_1 \in \phi_u$, which is a CON-$2$ solution of $[H^2_{u_2,V_1^t}J^t_1](s,u_1)$ over $\{(s, u_1) \mid s \in \phi \wedge u_1 \in \phi_u\}$ by \lemaref{lema:improvement-consistency-agent-2} and $V_2^{t+1}(s, u_1)$ is CON-linear in $s \in \phi$ and $u_1 \in \phi_u$. Finally, we copy the other strategies and value functions for the next iteration.

\startpara{Representation closures} The following lemmas show the strategies and value functions generated during each iteration of the Minimax-action-free PI algorithm are closed under B-PWC, CON-PWL and CON-PWC functions, and are thus finitely representable.

\begin{lema}[Evaluation closure for $\agent_1$]\label{lema:evaluation-consistency-agent-1}
If $\sigma_1^t \in \Sigma_1$ is a PWC stochastic kernel,  
$J_2^t,V_2^t \in \mathbb{F}(\Lambda_1)$ are CON-PWL Borel measurable and policy evaluation of $\agent_1$ is performed (procedure $\mathit{PE1}$), then $J_1^{t+1}=[H_{\sigma_1^t,V_2^t}^1J_2^t]$ is B-PWC.
\end{lema}
\begin{proof}
Suppose $\sigma_1^t \in \Sigma_1$ is a PWC stochastic kernel and $J_2^t,V_2^t \in \mathbb{F}(\Lambda_1)$ are CON-PWL  Borel measurable. Since $\sigma_1^t$ is a PWC stochastic kernel, there exists a constant-BFCP $\Phi_{\sigma_1^t}$ of $S$ for $\sigma_1^t$. Since $J_2^t$ is a CON-PWL Borel measurable function, there exists a BFCP $\Phi_{J_2^t}$ of $S$ satisfying the properties of \defiref{defi:s-CON-u1-PWL} for $J_2^t$. Therefore $J_2^t(s,\sigma_1^t(s))$ is constant on each region of the BFCP $\Phi_{\sigma_1^t}+\Phi_{J_2^t}$. We can similarly show that $V_2^t(s,\sigma_1^t(s))$ is constant on each region of the BFCP $\Phi_{\sigma_1^t}+\Phi_{V_2^t}$, where $\Phi_{V_2^t}$ is a BFCP of $S$ from \defiref{defi:s-CON-u1-PWL} for $V_2^t$. Consider the policy evaluation of $\agent_1$ (procedure $\mathit{PE1}$). Using \defiref{def:operator-H1-definition} we have that $J_1^{t+1}=[H_{\sigma_1^t,V_2^t}^1J^t_2]$ is constant on each region of the BFCP $\Phi_{\sigma_1^t}+\Phi_{J_2^t}+\Phi_{V_2^t}$, which also implies that $J_1^{t+1}$ is Borel measurable. Since $J_2^t$ and $V_2^t$ are bounded, then $J_1^{t+1}$ is also bounded as required.
\end{proof}
%
\begin{lema}[Improvement closure for $\agent_1$]\label{lema:improvement-consistency-agent-1}
If  $J_2^t,V_2^t \in \mathbb{F}(\Lambda_1)$ are CON-PWL Borel measurable and policy improvement of $\agent_1$ is performed (procedure $\mathit{PI1}$), then
$\sigma_1^{t+1}(s) \in \argmax\nolimits_{u_1 \in \mathbb{P}(A_1(s))} [H_{u_1,V_2^t}^1J^t_2](s)$
is a PWC stochastic kernel, and $V_1^{t+1}=[H_{\sigma_1^{t+1},V_2^t}^1J_2^t]$ is B-PWC.
\end{lema}

\begin{proof}
Suppose $J_2^t,V_2^t \in \mathbb{F}(\Lambda_1)$ are CON-PWL Borel measurable functions. Using \cite[Chapter 18.1]{HLR-PF:10} and \defiref{defi:s-CON-u1-PWL} it follows that the function $K^t  \coloneqq \min[J_2^t, V_2^t]$
is Borel measurable. Note that, over each region of $\Phi_{J_2^t} + \Phi_{V_2^t}$, $K^t(s, u_1)$ is constant in $s$ given $u_1$, and PWL in $u_1$ given $s$ (where $\Phi_{J_2^t}$ and $\Phi_{V_2^t}$ are from \lemaref{lema:evaluation-consistency-agent-1}), and therefore $K^t$ is CON-PWL. 

Let $\Phi_{K^t} = \Phi_{J_2^t} + \Phi_{V_2^t}$ and $\Theta_{K^t}$ {\revise be} a BFCP  of $\Lambda_1$ satisfying the properties of \defiref{defi:s-CON-u1-PWL} for $K^t$. Every state in each region of the BFCP $\Phi_{K^t}$ has the same set of available actions for $\agent_1$ and same strategy $u_1$ that maximises $K^t(s,u_1)$ on a region of $\Theta_{K^t}$. Therefore, using the CON-$1$ solution in \defiref{defi:con1-solution}, the strategy of $\agent_1$:
\[
\sigma_1^{t+1}(s) \in \argmax\nolimits_{u_1 \in \mathbb{P}(A_1(s))} [H_{u_1,V_2^t}^1J_2^t](s)
\]
is constant on each region of $\Phi_{K^t}$, which also implies that $\sigma_1^{t+1}$ is Borel measurable. 
Since $\sigma_1^{t+1}$ is a PWC stochastic kernel, then \lemaref{lema:evaluation-consistency-agent-1} implies that $V_1^{t+1}$ is B-PWC as required.
\end{proof}

\begin{lema}[Evaluation closure for $\agent_2$]\label{lema:evaluation-consistency-agent-2}
If $J_1^t, V_1^t \in \mathbb{F}(S)$ are B-PWC and $\overline{\sigma}_2^t \in \overline{\Sigma}_2$ is a CON-PWC stochastic kernel and policy evaluation of $\agent_2$ is performed (procedure $\mathit{PE2}$), then $J_2^{t+1} = [H_{\overline{\sigma}_2^t,V_1^t}^2J_1^t]$ is CON-PWL Borel measurable.
\end{lema}

\begin{proof}
Suppose $J_1^t$ and $V_1^t$ are B-PWC and $\overline{\sigma}_2^t \in \overline{\Sigma}_2$ is a CON-PWC stochastic kernel. Using \cite[Chapter 18.1]{HLR-PF:10} it follows that $\gamma\max[J_1^t,V_1^t]$ is  B-PWC. In view of the B-PWC function $Q({}\cdot{},\alpha,V)$ in \thomref{thom:B-PWC-closure-VI}, for each $\alpha  \in A$ the function:
\[
    \hat{Q}_{\alpha}^t(s) \coloneqq Q(s,\alpha,\gamma \max[J_1^t, V_1^t])
\]
is B-PWC. Let $\Phi_{\hat{Q}^t}$ be a BFCP of $S$ such that $\hat{Q}^t_{\alpha}$ is constant on each region of $\Phi_{\hat{Q}^t}$ for $\alpha \in A$. It follows that $A(s)$ is constant on each region of $\Phi_{\hat{Q}^t}$. 

Next, let $\Phi_{\overline{\sigma}_2^t}$ be a BFCP of $S$ satisfying the properties of \defiref{defi:s-CON-u1-PWC} for the CON-PWC stochastic kernel $\overline{\sigma}_2^t$.
For the BFCP $\Phi_{\hat{Q}^t}+\Phi_{\overline{\sigma}_2^t}$ of $S$, we generate a BFCP $\Theta^t_1$ of $\Lambda_1$ such that each region $\theta_1^{t}(\phi) \in \Theta_1^t$, induced by a region $\phi \in \Phi_{\hat{Q}^t}+\Phi_{\overline{\sigma}_2^t}$, is given by $\theta^{t}_1(\phi) = \{ (s, u_1) \in \Lambda_1  \mid s \in \phi \}$. Finally, consider the policy evaluation of $\agent_2$. According to \defiref{def:operator-H2-definition}, for $(s,u_1)\in \theta^{t}_1(\phi)$, $J_2^{t+1}(s,u_1)=[H^2_{\overline{\sigma}_2^t,V_1^t}J^t_1](s,u_1)$ is constant in $s$ for a fixed $u_1$, and PWL in $u_1$ for a fixed $s \in S$. Thus, $J_2^{t+1}$ is CON-PWL. Since $\hat{Q}_{\alpha}^t$ and $\overline{\sigma}_2^t$ are bounded, Borel measurable, then so is $J_2^{t+1}$ by \defiref{def:operator-H2-definition} as required.
\end{proof}

\begin{lema}[Improvement closure for $\agent_2$]\label{lema:improvement-consistency-agent-2}
If $J_1^t, V_1^t \in \mathbb{F}(S)$ are B-PWC and policy improvement of $\agent_2$ is performed (procedure $\mathit{PI2}$), then $\overline{\sigma}_2^{t+1}(s,u_1) \in \argmin_{u_2 \in \mathbb{P}(A_2(s))} [H_{u_2,V_1^t}^2J_1^t](s,u_1)$ is a CON-PWC stochastic kernel, and $V_2^{t+1} = [H_{\overline{\sigma}_2^{t+1},V_1^t}^2J_1^t]$ is CON-PWL Borel measurable.
\end{lema}
\begin{proof}
Suppose $J_1^t, V_1^t \in \mathbb{F}(S)$ are B-PWC.
For the BFCP $\Phi_{\hat{Q}^t}$ of $S$, we generate a BFCP $\Theta_2^t$ of $\Lambda_1$ such that each region $\theta^{t}_2(\phi)$ in $\Theta_2^t$ induced by a region $\phi \in \Phi_{\hat{Q}^t}$ is given by $\theta^{t}_2(\phi) = \{ (s, u_1) \in \Lambda_1  \mid s \in \phi\}$, where $\Phi_{\hat{Q}^t}$ is from the proof of \lemaref{lema:evaluation-consistency-agent-2}. Consider the policy improvement of $\agent_2$ (procedure $\mathit{PI2}$). According to \defiref{def:operator-H2-definition}, by using the CON-$2$ solution in \defiref{defi:con2-solution}, for $(s,u_1) \in \theta^{t}_2(\phi)$, the Stackelberg strategy of $\agent_2$:
\[
    \overline{\sigma}_2^{t+1}(s,u_1) \in \argmin\nolimits_{u_2 \in \mathbb{P}(A_2(s))}[H^2_{u_2,V_1^t}J^t_1](s,u_1)
\]
is constant in $s$ for a fixed $u_1$, and PWC in $u_1$ for a fixed $s$. Thus, $\overline{\sigma}_2^{t+1}$ is CON-PWC.  
Since $\overline{\sigma}_2^{t+1}$ is a CON-PWC stochastic kernel, then \lemaref{lema:evaluation-consistency-agent-2} implies that $V_2^{t+1}$ is CON-PWL Borel measurable as required.
\end{proof}
By fusing Lemmas \ref{lema:evaluation-consistency-agent-1}, \ref{lema:improvement-consistency-agent-1}, \ref{lema:evaluation-consistency-agent-2} and \ref{lema:improvement-consistency-agent-2} we can prove that the strategies and value functions generated during each iteration of  \algoref{alg:maxmin-free-PI} 
never leave a finitely representable class of functions, and 
\algoref{alg:BFCP-iteration-PI} constructs new BFCPs such that the strategies and value functions after one iteration of the Minimax-action-free PI algorithm remain constant, PWC, or PWL on each region of the constructed BFCPs.

\begin{thom}[Representation closure]\label{thom:repsentation-consistency}
In any iteration of the Minimax-action-free PI algorithm (see  \algoref{alg:maxmin-free-PI}), if 
\begin{enumerate}[(i)]
    \item $J_1^t, V_1^t \in \mathbb{F}(S)$ are B-PWC and $\sigma_1^t \in \Sigma_1$ is a PWC stochastic kernel;
    
    \item $J_2^t, V_2^t \in \mathbb{F}(\Lambda_1)$ are CON-PWL Borel measurable and $\overline{\sigma}_2^t \in \overline{\Sigma}_2$ is a CON-PWC stochastic kernel;
\end{enumerate}
then so are $J_1^{t+1}$, $V_1^{t+1}$, $\sigma_1^{t+1}$, $J_2^{t+1}$, $V_2^{t+1}$ and $\overline{\sigma}_2^{t+1}$,  respectively, regardless of which one of the four iterations is performed.
\end{thom}

\begin{proof}
The conclusion follows from one of Lemmas \ref{lema:evaluation-consistency-agent-1}, \ref{lema:improvement-consistency-agent-1}, \ref{lema:evaluation-consistency-agent-2} and \ref{lema:improvement-consistency-agent-2}, depending on which one of the four iterations is executed.
\end{proof}




\begin{cor}[BFCP iteration for Minimax-action-free PI]\label{cor:bfcp_minimax_pi}
After performing \algoref{alg:BFCP-iteration-PI}:
\begin{enumerate}[(i)]
    \item $\Phi_{J_1^{t+1}}$, $\Phi_{V_1^{t+1}}$ and $\Phi_{\sigma_1^{t+1}}$ are constant-BFCPs of $S$ for $J_1^{t+1}=[H_{\sigma_1^t,V_2^t}^1J_2^t]$, $V_1^{t+1}=[H_{\sigma_1^{t+1},V_2^t}^1J_2^t]$ and $\sigma_1^{t+1}(s) \in \argmax_{u_1 \in \mathbb{P}(A_1(s))} [H_{u_1,V_2^t}^1J^t_2](s)$; 
    
    \item $\Theta_{J_2^{t+1}}$ and $\Theta_{V_2^{t+1}}$ are BFCPs of $\Lambda_1$ for $J_2^{t+1} = [H_{\overline{\sigma}_2^t,V_1^t}^2J_1^t]$ and $V_2^{t+1} = [H_{\overline{\sigma}_2^{t+1},V_1^t}^2J_1^t]$ meeting the conditions of \defiref{defi:s-CON-u1-PWL}, and $\Theta_{\overline{\sigma}_2^{t+1}}$ is a BFCP of $\Lambda_1$ for $\overline{\sigma}_2^{t+1}(s,u_1) \in \argmin_{u_2 \in \mathbb{P}(A_2(s))} [H_{u_2,V_1^t}^2J_1^t](s,u_1)$ meeting the conditions of \defiref{defi:s-CON-u1-PWC}.
\end{enumerate}
\end{cor}

\subsection{Convergence analysis and strategy computation}
We next prove the convergence of the Minimax-action-free PI algorithm by showing that 
there exists an operator from the product space of the function spaces over which $J_1$, $V_1$, $J_2$ and $V_2$ are defined to itself,
which is a contraction mapping with a unique fixed point, one of whose components is the value function multiplied by a known constant. The proof closely follows the steps for finite state spaces given in~\cite{DB:21}, but 
is more complex due to the underlying uncountable state space and the need to deal with the requirement of 
Borel measurability and 
finite representation of strategies and value functions.


\startpara{Convergence analysis} Given PWC $\sigma_1 \in \Sigma_1$ and CON-PWC $\overline{\sigma}_2 \in \overline{\Sigma}_2$, we define the operator $G_{\sigma_1,\overline{\sigma}_2} : (\mathbb{F}(S) \times \mathbb{F}(S) \times \mathbb{F}(\Lambda_1) \times \mathbb{F}(\Lambda_1)) \rightarrow (\mathbb{F}(S) \times \mathbb{F}(S) \times \mathbb{F}(\Lambda_1) \times \mathbb{F}(\Lambda_1))$ such that:
\begin{equation}
   \!\! G_{\sigma_1,\overline{\sigma}_2}(J_1,V_1,J_2,V_2) \coloneqq (M_{\sigma_1}^1(J_2,V_2), K^1(J_2,V_2), M_{\overline{\sigma}_2}^2(J_1,V_1), K^2(J_1,V_1)) \label{g-eqn}
\end{equation}
where we assume $J_1,V_1 \in \mathbb{F}(S)$ are B-PWC, $J_2,V_2 \in \mathbb{F}(\Lambda_1)$ are CON-PWL, and the four operators $M_{\sigma_1}^1$, $K^1$, $M_{\overline{\sigma}_2}^2$ and $K^2$ represent the four iterations of the Minimax-action-free PI algorithm from lines 3 to 16, and are defined as follows.
\begin{itemize}
\item $M_{\sigma_1}^1 : \mathbb{F}(\Lambda_1) \times \mathbb{F}(\Lambda_1) \to \mathbb{F}(S)$ corresponds to the policy evaluation of $\agent_1$ (procedure $\mathit{PE1}$) where for any $s \in S$:
    \begin{eqnarray}
        M_{\sigma_1}^1(J_2, V_2)(s) & \coloneqq & [H_{\sigma_1,V_2}^1J_2](s)  \label{eq:M1-definition}
    \end{eqnarray}
       
and is B-PWC using \lemaref{lema:evaluation-consistency-agent-1}.
\item
$K^1 : \mathbb{F}(\Lambda_1) \times \mathbb{F}(\Lambda_1) \to \mathbb{F}(S)$ corresponds to the policy improvement of $\agent_1$ (procedure $\mathit{PI1}$) where for any $s \in S$:
    \begin{eqnarray}
        K^1(J_2, V_2)(s) & \coloneqq & \max\nolimits_{u_1 \in \mathbb{P}(A_1(s))}[H_{u_1,V_2}^1J_2](s) 
        \label{eq:N1-definition}
    \end{eqnarray}
and is B-PWC using \lemaref{lema:improvement-consistency-agent-1}.
\item 
$M_{\overline{\sigma}_2}^2 :\mathbb{F}(S) \times \mathbb{F}(S) \to \mathbb{F}(\Lambda_1)$ corresponds to the policy evaluation of $\agent_2$ (procedure $\mathit{PE2}$) where for any $(s,u_1) \in \Lambda_1$:
    \begin{eqnarray}
        M_{\overline{\sigma}_2}^2(J_1, V_1)(s,u_1) & \coloneqq & [H_{\overline{\sigma}_2,V_1}^2J_1](s,u_1) 
   \label{eq:M2-definition}
    \end{eqnarray}
and is CON-PWL Borel measurable using \lemaref{lema:evaluation-consistency-agent-2}.
\item
$K^2 :\mathbb{F}(S) \times \mathbb{F}(S) \to \mathbb{F}(\Lambda_1)$ corresponds to the policy improvement of $\agent_2$ (procedure $\mathit{PI2}$) where any $(s,u_1) \in \Lambda_1$:
    \begin{eqnarray}
        K^2(J_1, V_1)(s,u_1) & \coloneqq & \min\nolimits_{u_2 \in \mathbb{P}(A_2(s))} [H_{u_2,V_1}^2J_1](s, u_1) 
   \label{eq:N2-definition}
    \end{eqnarray}
and is CON-PWL Borel measurable using \lemaref{lema:improvement-consistency-agent-2}. 
\end{itemize}
For the spaces $\mathbb{F}(S) \times \mathbb{F}(S)$ and $\mathbb{F}(\Lambda_1) \times \mathbb{F}(\Lambda_1)$, we consider the norm $\|(J,V)\|=\max\{\|J\|,\|V\|\} $,
and for the space $\mathbb{F}(S) \times \mathbb{F}(S) \times \mathbb{F}(\Lambda_1) \times \mathbb{F}(\Lambda_1)$ 
the norm $\|(J_1,V_1,J_2,V_2)\|=\max\{\|J_1\|,\|V_1\|,\|J_2\|,\|V_2\|\}$. We next require the following properties of these norms, which follow from \cite{DB:21}.

%
\begin{lema}\label{lema:two-norm-inequalities}
For any $J_1,V_1,J_1',V_1' \in \mathbb{F}(S)$ and $J_2,V_2,J_2',V_2' \in \mathbb{F}(\Lambda_1)$:
\begin{eqnarray*}
     \|\max[J_1,V_1] - \max[J_1',V_1'] \| & \leq & \max \{ \|J_1 - J_1'\|, \|V_1 - V_1'\|\} \\
      \|\min[J_2,V_2] - \min[J_2',V_2'] \| & \leq & \max \{ \|J_2 - J_2'\|, \|V_2 - V_2'\|\} \,.
\end{eqnarray*}
\end{lema}
\begin{proof}
Consider any $J_1,V_1,J_1',V_1' \in \mathbb{F}(S)$. The norm for the space $\mathbb{F}(S)$ implies that for any $s \in S$:
 \begin{eqnarray}
    J_1(s) & \leq & J_1'(s) + \max \{ \|J_1 - J_1'\|, \|V_1 - V_1'\|\} \label{eq:JV1-one-side1}  \\
    V_1(s) &\leq & V_1'(s) + \max \{ \|J_1 - J_1'\|, \|V_1 - V_1'\|\} \label{eq:JV1-one-side2} 
\end{eqnarray}   
from which we have:
\begin{equation}
    \max\{J_1(s), V_1(s) \} \leq   \max \{ J_1'(s), V_1'(s) \} + \max \{ \|J_1 - J_1'\|, \|V_1 - V_1'\|\} \, . \label{eq:J1-V1-one-side-max}
\end{equation}
Exchanging $(J_1,V_1)$ with $(J_1',V_1')$ in \eqref{eq:JV1-one-side1} and \eqref{eq:JV1-one-side2}  derives an inequality similar to \eqref{eq:J1-V1-one-side-max}, and combining it with \eqref{eq:J1-V1-one-side-max} leads to the inequality:
\begin{equation}
    |\max\{J_1(s), V_1(s) \}  - \max \{ J_1'(s), V_1'(s) \}| \leq \max \{ \|J_1 - J_1'\|, \|V_1 - V_1'\|\} \label{eq:sup-J1-V1-max}
\end{equation}
for any $s \in S$. Since $J_1, V_1, J_1'$ and $V_1'$ are bounded, Borel measurable, so is $\max[J_1,V_1]-\max[J_1',V_1']$ by \cite[Chapter 18.1]{HLR-PF:10}, i.e., $\max[J_1,V_1]-\max[J_1',V_1'] \in \mathbb{F}(S)$. Thus, since \eqref{eq:sup-J1-V1-max} holds for any $s \in S$:
\[
    \|\max[J_1,V_1]-\max[J_1',V_1']\| \leq \max \{ \|J_1 - J_1'\|, \|V_1 - V_1'\|\} \,.
\]
The second inequality of the lemma can be proved following the same steps for $J_2,V_2,J_2',V_2' \in \mathbb{F}(\Lambda_1)$.
\end{proof}
Using the above operators and results, we are now in a position to prove the convergence of the Minimax-action-free PI algorithm.

\begin{thom}[Convergence guarantee]\label{convergence-thm}
If each of the four iterations of the Minimax-action-free PI algorithm (\algoref{alg:maxmin-free-PI}) from lines 3 to 16 is performed infinitely often, then the sequence $( \gamma V_1^t )_{t \in \mathbb{N}}$ generated by the algorithm converges to $V^\star$.
\end{thom}

\begin{proof}
We prove each component $G_{\sigma_1,\overline{\sigma}_2}$ satisfies a contraction property. Suppose that $J_1, V_1, J_1',V_1' \in \mathbb{F}(S)$ are B-PWC and $J_2, V_2, J_2',V_2' \in \mathbb{F}(\Lambda_1)$ are CON-PWL Borel measurable. 
\begin{itemize}
\item
For $M_{\sigma_1}^1$, since $M_{\sigma_1}^1(J_2, V_2) - M_{\sigma_1}^1(J_2', V_2') \in \mathbb{F}(S)$ by \cite[Chapter 18.1]{HLR-PF:10}. By \defiref{def:operator-H1-definition}, the sup-norm for $\mathbb{F}(S)$ and rearranging we have:
\begin{align}
    \lefteqn{\hspace*{-0.6cm} \|M_{\sigma_1}^1(J_2, V_2) - M_{\sigma_1}^1(J_2', V_2') \|} \nonumber \\
    &  \lefteqn{\hspace*{-0.6cm} =  \gamma^{-1} \sup\nolimits_{s \in S}| \min\{J_2(s,\sigma_1(s)), V_2(s, \sigma_1(s))\}} \nonumber \\
    & \lefteqn{ -  \min\{J_2'(s,\sigma_1(s)), V_2'(s, \sigma_1(s))\} | } \nonumber \\
    &   \lefteqn{\hspace*{-0.6cm} \leq \gamma^{-1} \sup\nolimits_{(s,u_1) \in \Lambda_1} \left|\min\{J_2(s,u_1), V_2(s, u_1)\} - \min\{J_2'(s,u_1), V_2'(s, u_1)\} \right|} \nonumber \\
    && \mbox{since $\{ (s,\sigma_1(s)) \mid s \in S\} \subseteq \Lambda_1$}  \nonumber \\
   &   \lefteqn{\hspace*{-0.6cm} = \gamma^{-1} \left\|\min[J_2,V_2] - \min[J_2',V_2'] \right\|} \nonumber \\
   && \mbox{since $\min[J_2,V_2] - \min[J_2',V_2'] \in \mathbb{F}(\Lambda_1)$ using \cite[Chapter 18.1]{HLR-PF:10}} \nonumber \\
   & \lefteqn{\hspace*{-0.6cm} \leq  \gamma^{-1} \max \{ \|J_2 - J_2'\|, \|V_2 - V_2'\|\}} & \mbox{by \lemaref{lema:two-norm-inequalities}} \nonumber \\
    & \lefteqn{\hspace*{-0.6cm} \leq  \gamma^{-1} \max \{ \|J_1 - J_1'\|, \|V_1 - V_1'\|, \|J_2 - J_2'\|, \|V_2 - V_2'\|\}\, .}   \label{eq:M1-inequality}
\end{align}
\item
For $K^1$, since $K^1(J_2,V_2) - K^1(J_2',V_2') \in \mathbb{F}(S)$ by \defiref{def:operator-H1-definition} and the sup-norm for $\mathbb{F}(S)$:
\begin{align}
\lefteqn{\|K^1(J_2,V_2) - K^1(J_2',V_2')\|} \nonumber \\
    & \lefteqn{= \; \sup\nolimits_{s \in S}\big|\max\nolimits_{u_1 \in \mathbb{P}(A_1(s))} \gamma^{-1} \min\{J_2(s,u_1), V_2(s, u_1)\} } \nonumber  \\
    & \lefteqn{\quad \;  - \max\nolimits_{u_1 \in \mathbb{P}(A_1(s))} \gamma^{-1} \min\{J_2'(s,u_1), V_2'(s, u_1)\} \big|} \nonumber  \\
    & \lefteqn{\leq \;  \gamma^{-1} \sup\nolimits_{(s,u_1) \in \Lambda_1} \left|\min\{J_2(s,u_1), V_2(s, u_1)\} - \min\{J_2'(s,u_1), V_2'(s, u_1)\} \right|} \nonumber  \\
    && \qquad \qquad \qquad \mbox{rearranging and since $\{ (s,u_1) \mid u_1 \in\mathbb{P}(A_1(s)) \} \subseteq \Lambda_1$} \nonumber \\
    & \lefteqn{\leq  \; \gamma^{-1} \max \{ \|J_1 - J_1'\|, \|V_1 - V_1'\|, \|J_2 - J_2'\|, \|V_2 - V_2'\|\}}  \label{eq:N1-inequality}
\end{align}
where the final inequality follows from similar arguments used in \eqref{eq:M1-inequality}.
\item
For $M_{\overline{\sigma}_2}^2$, since $M_{\overline{\sigma}_2}^2(J_1, V_1) - M_{\overline{\sigma}_2}^2(J_1', V_1') \in \mathbb{F}(\Lambda_1)$ by \defiref{def:operator-H2-definition} and the sup-norm for $\mathbb{F}(\Lambda_1)$ we have:
\begin{align}
    \lefteqn{\hspace*{-0.5cm} \|M_{\overline{\sigma}_2}^2(J_1, V_1) - M_{\overline{\sigma}_2}^2(J_1', V_1') \|}  \nonumber \\
    & \lefteqn{\hspace*{-0.5cm}= \; \sup_{(s,u_1) \in \Lambda_1}\left| \mbox{$\sum\limits_{(a_1,a_2) \in A(s)}$} \big(Q(s,(a_1,a_2),\gamma \max[J_1, V_1]) \right.} \nonumber \\
    & \lefteqn{\hspace*{-0.5cm}\quad \; -Q(s,(a_1,a_2),\gamma \max[J_1', V_1'])\big) u_1(a_1)\overline{\sigma}_2(a_2 \mid (s,u_1)) \Bigg| \nonumber} \\
    & \lefteqn{\hspace*{-0.5cm}=\; \sup_{(s,u_1) \in \Lambda_1}\left| \mbox{$\sum\limits_{(a_1,a_2) \in A(s)}$}  \gamma \beta \mbox{$\sum\limits_{s'\in \Theta(s,(a_1,a_2))}$} \delta(s,(a_1,a_2))(s')  \right.} \nonumber \\
     & \lefteqn{\quad  \left(\max\{J_1(s'),V_1(s')\} -\max\{J_1'(s'),V_1'(s')\})u_1(a_1)\overline{\sigma}_2(a_2 \mid (s,u_1) \right) \Bigg| \nonumber} \\
     && \mbox{rearranging, by \defiref{defi:minimax-operator} and the sup-norm for $\mathbb{F}(\Lambda_1)$} \nonumber \\
    & \lefteqn{\hspace*{-0.5cm}\leq \; \gamma \beta \sup_{(s,u_1) \in \Lambda_1} \mbox{$\sum\limits_{(a_1,a_2) \in A(s)}$} \mbox{$\sum\limits_{s'\in \Theta(s,(a_1,a_2))}$} \delta(s,(a_1,a_2))(s')} \nonumber \\
  & \lefteqn{\hspace*{-0.5cm}\quad   \left| \max\{J_1(s'),V_1(s')\} -\max\{J_1'(s'),V_1'(s')\} \right|u_1(a_1)\overline{\sigma}_2(a_2 \mid (s,u_1))} \nonumber \\
&& \qquad \qquad   \mbox{rearranging and since $\delta$, $u_1$ and $\overline{\sigma}_2$ are non-negative} \nonumber \\ 
  & \lefteqn{\hspace*{-0.5cm}\leq \;  \gamma \beta \sup_{(s,u_1) \in \Lambda_1} \mbox{$\sum\limits_{(a_1,a_2) \in A(s)}$}  \mbox{$\sum\limits_{s'\in \Theta(s,(a_1,a_2)}$} \delta(s,(a_1,a_2))(s')} \nonumber \\
  & \lefteqn{\hspace*{-0.5cm}  \sup_{s'' \in S} \left| \max\{J_1(s''),V_1(s'')\} -\max\{J_1'(s''),V_1'(s'')\} \right|u_1(a_1)\overline{\sigma}_2(a_2 \mid (s,u_1))} \nonumber \\
&& \mbox{since $f(s') \leq \sup_{s'' \in S} f(s'')$ for any $f \in \mathbb{F}(S)$} \nonumber \\    
& \lefteqn{\hspace*{-0.5cm}= \; \gamma \beta \sup_{s'' \in S} \left| \max\{J_1(s''),V_1(s'')\} -\max\{J_1'(s''),V_1'(s'')\} \right|} \nonumber \\
  & \lefteqn{\quad  \sup_{(s,u_1) \in \Lambda_1} \mbox{$\sum\limits_{(a_1,a_2) \in A(s)}$} \mbox{$\sum\limits_{s'\in \Theta(s,(a_1,a_2))}$} \delta(s,(a_1,a_2))(s')u_1(a_1)\overline{\sigma}_2(a_2 \mid (s,u_1))} \nonumber \\
&& \mbox{rearranging} \nonumber \\
& \lefteqn{\hspace*{-0.5cm}= \;  \gamma \beta \sup_{s'' \in S} \left| \max\{J_1(s''),V_1(s'')\} -\max\{J_1'(s''),V_1'(s'')\} \right|} \nonumber \\
  && \mbox{since $\delta \in \mathbb{P}(S\times A)$, $u_1 \in \mathbb{P}(A_1)$ and $\bar{\sigma}_2 \in \mathbb{P}(A_2 \mid\Lambda_1)$} \nonumber \\   & \lefteqn{\hspace*{-0.5cm}= \; \gamma \beta \left\|\max[J_1,V_1] - \max[J_1',V_1'] \right\|} \nonumber  \\
  && \mbox{since $\max[J_1,V_1] - \max[J_1',V_1'] \in \mathbb{F}(S)$} \nonumber \\    
& \lefteqn{\hspace*{-0.5cm}\leq \; \gamma \beta \max \{ \|J_1 - J_1'\|, \|V_1 - V_1'\|\}} & \mbox{by \lemaref{lema:two-norm-inequalities}} \nonumber \\ 
& \lefteqn{\hspace*{-0.5cm}\leq \; \gamma \beta \max \{ \|J_1 - J_1'\|, \|V_1 - V_1'\|, \|J_2 - J_2'\|, \|V_2 - V_2'\|\}.} \, \label{eq:M2-inequality}
\end{align}
\item
For $K^2$, since $K^2(J_1,V_1) - K^2(J_1',V_1') \in \mathbb{F}(\Lambda_1)$, by the sup-norm for $\mathbb{F}(\Lambda_1)$:
\begin{align}
 \lefteqn{\hspace*{-0.5cm} \|K^2(J_1,V_1) - K^2(J_1',V_1')\|} \nonumber \\
    & \lefteqn{\hspace*{-0.5cm}= \; \sup_{(s,u_1) \in \Lambda_1} \left| \min_{u_2 \in \mathbb{P}(A_2(s))} \mbox{$\sum\limits_{(a_1,a_2) \in A(s)}$}\gamma \beta \mbox{$\sum\limits_{s'\in \Theta(s,(a_1,a_2))}$}  \delta(s,(a_1,a_2))(s') \right.} \nonumber \\
  & \lefteqn{\hspace*{-0.5cm}\quad   \left(\max\{J_1(s'),V_1(s')\} -\max\{J_1'(s'),V_1'(s')\} \right) u_1(a_1)u_2(a_2) \Bigg|} \nonumber \\
    & \lefteqn{\hspace*{-0.5cm}\leq \; \gamma \beta \sup_{(s,u_1) \in \Lambda_1} \min_{u_2 \in \mathbb{P}(A_2(s))} \mbox{$\sum\limits_{(a_1,a_2) \in A(s)}$} \mbox{$\sum\limits_{s'\in \Theta(s,(a_1,a_2))}$} \delta(s,(a_1,a_2))(s')}  \nonumber \\
  & \hspace*{-0.5cm} \quad   \left| \max\{J_1(s'),V_1(s')\} -\max\{J_1'(s'),V_1'(s')\} \right|u_1(a_1)u_2(a_2)  &  \mbox{rearranging} \nonumber \\
  & \lefteqn{\hspace*{-0.5cm}\leq \;  \gamma \beta \max \{ \|J_1 - J_1'\|, \|V_1 - V_1'\|, \|J_2 - J_2'\|, \|V_2 - V_2'\|\}} \label{eq:N2-inequality} 
\end{align}
where the final inequality follows from similar arguments used in \eqref{eq:M2-inequality}.
\end{itemize}
Next we prove that $G_{\sigma_1,\overline{\sigma}_2}$ is a contraction mapping using the above inequalities. More precisely, by definition, see \eqref{g-eqn}, we have:
\begin{align}
     \lefteqn{\| G_{\sigma_1,\overline{\sigma}_2}(J_1,V_1,J_2,V_2) - G_{\sigma_1,\overline{\sigma}_2}(J_1',V_1',J_2',V_2') \|} \nonumber \\
    & = \; \|(M_{\sigma_1}^1(J_2,V_2) - M_{\sigma_1}^1(J_2',V_2'), K^1(J_2,V_2) - K^1(J_2',V_2'),   \nonumber \\
    & \qquad  \ M_{\overline{\sigma}_2}^2(J_1,V_1) - M_{\overline{\sigma}_2}^2(J_1',V_1'), K^2(J_1,V_1) - K^2(J_1',V_1') ) \| \nonumber \\
    &= \; \lefteqn{\max \{ \|M_{\sigma_1}^1(J_2,V_2) - M_{\sigma_1}^1(J_2',V_2') \|, \|K^1(J_2,V_2) - K^1(J_2',V_2') \|, } \nonumber \\
    & \qquad   \ \|M_{\overline{\sigma}_2}^2(J_1,V_1) - M_{\overline{\sigma}_2}^2(J_1',V_1') \|, \|K^2(J_1,V_1) - K^2(J_1',V_1') \| \}  \quad \mbox{rearranging} \nonumber \\
    & \leq \; \max \{ \gamma^{-1},\gamma\beta \} \max \{ \|J_1 - J_1'\|, \|V_1 - V_1'\|, \|J_2 - J_2'\|, \|V_2 - V_2'\|\} \nonumber
\end{align}
where the final inequality follows from  \eqref{eq:M1-inequality}, \eqref{eq:M2-inequality}, \eqref{eq:N1-inequality} and \eqref{eq:N2-inequality}.

Therefore, since $\max\{\gamma^{-1},\gamma\beta\} <1$ and assuming $\sigma_1$ is PWC and $\overline{\sigma}_2$ is CON-PWC, we have that $G_{\sigma_1,\overline{\sigma}_2}$ is a contraction mapping for $(\sigma_1,\overline{\sigma}_2) \in \Sigma_1 \times \overline{\Sigma}_2$. Now since $\mathbb{F}(S) \times \mathbb{F}(S) \times \mathbb{F}(\Lambda_1) \times \mathbb{F}(\Lambda_1)$ is a complete metric space with respect to the sup norm, we conclude that $G_{\sigma_1,\overline{\sigma}_2}$ has a unique fixed point $(J_1^\star,V_1^\star,J_2^\star,V_2^\star)$. In view of \eqref{eq:M1-definition}--\eqref{eq:N2-definition}, this fixed point satisfies for each $(s,u_1)\in \Lambda_1$:
\begin{align}
     J_1^\star(s) & = \gamma^{-1} \min\{J_2^\star(s,\sigma_1(s)), V_2^\star(s, \sigma_1(s))\} \label{eq:fixed-point-four1} \\
    V_1^\star(s) & = \gamma^{-1}  \max\nolimits_{u_1 \in \mathbb{P}(A_1(s))}  \min\{J_2^\star(s,u_1), V_2^\star(s, u_1)\} \label{eq:fixed-point-four2} \\
    J_2^\star(s,u_1) & = \mbox{$\sum\limits_{(a_1,a_2) \in A(s)}$}Q(s,(a_1,a_2) \label{eq:fixed-point-four3}  \gamma \max[J_1^\star, V_1^\star]) u_1(a_1)\overline{\sigma}_2(a_2 \mid (s,u_1))  \\
     V_2^\star(s,u_1) & =  \min_{u_2 \in \mathbb{P}(A_2(s))} \mbox{$\sum\limits_{(a_1,a_2) \in A(s)}$} \!\!\!\!\!\! Q(s,(a_1,a_2), \gamma\max[J_1^\star, V_1^\star]) u_1(a_1)u_2(a_2) \, . \label{eq:fixed-point-four4}
\end{align}
By combining \eqref{eq:fixed-point-four1}--\eqref{eq:fixed-point-four4}, we have for each $(s,u_1)\in \Lambda_1$:
\begin{eqnarray*}
    J_1^\star(s) \leq V_1^\star(s) \; \mbox{and} \;
    J_2^\star(s,u_1) \geq V_2^\star(s,u_1)
\end{eqnarray*}
from which \eqref{eq:fixed-point-four2} and \eqref{eq:fixed-point-four4} can be simplified to:
\begin{eqnarray*}
    V_1^\star(s) & = & \max\nolimits_{u_1 \in \mathbb{P}(A_1(s))} \gamma^{-1} V_2^\star(s,u_1) \\
    V_2^\star(s,u_1) &  = & \min\nolimits_{u_2 \in \mathbb{P}(A_2(s))} \mbox{$\sum\nolimits_{(a_1,a_2) \in A(s)}$} Q(s,(a_1,a_2), \gamma V_1^\star) u_1(a_1)u_2(a_2)
\end{eqnarray*}
implying that $\gamma V_1^\star(s)$ equals:
\[ \max_{u_1 \in \mathbb{P}(A_1(s))}\min_{u_2 \in \mathbb{P}(A_2(s))} \mbox{$\sum\nolimits_{(a_1,a_2) \in A(s)}$} Q(s, (a_1,a_2), \gamma V_1^\star) u_1(a_1)u_2(a_2) =  [T(\gamma V_1^\star)] \, .
\]
Thus, we have $\gamma V_1^\star =V^\star$, which completes the proof.
\end{proof}



\startpara{Strategy computation} Next, introducing a criterion for selecting the minimax solution over a region, we compute the strategies for the agents based on the function returned by the Minimax-action-free PI algorithm.

\begin{defi}[CON-$3$ solution]\label{defi:con3-solution} Let $f \in \mathbb{F}(\Lambda_{12})$. If there exists a BFCP $\Phi$ of $S$ where, for each $\phi \in \Phi$: $A(s) = A(s')$ for $s, s' \in \phi$ there exists a pair of probability measures $u_1^{\phi} \in \mathbb{P}(A_1(s))$ and $u_2^{\phi} \in \mathbb{P}(A_2(s))$ for $s \in \phi$ such that $f(s,u_1^{\phi},u_2^{\phi}) = \max_{u_1 \in \mathbb{P}(A_1(s))} \min_{u_2 \in \mathbb{P}(A_2(s))} f(s,u_1,u_2)$ for $s \in \phi$, and $\sigma_1 \in \Sigma_1$, $\sigma_2 \in \Sigma_2$ are such that $\sigma_1(s) = u_1^{\phi}$ and $\sigma_2(s) = u_2^{\phi}$ for $s \in \phi$, then $(\sigma_1,\sigma_2)$ is a constant-$3$ (CON-$3$) solution of $f$ over $\phi$.
\end{defi}

\begin{lema}[PWC strategies]\label{pwc-strat-lem}
If  $V = \gamma V_1^t$, where $V_1^t$ is from iteration  $t \in \mathbb{N}$ of the Minimax-action-free PI algorithm, and 
$(\sigma_1, \sigma_2) \in \Sigma$ achieves the maximum and the minimum in \defiref{defi:minimax-operator} for $V$ and all $s \in S$ via a CON-$3$ solution, then $\sigma_1$ and $\sigma_2$ are PWC stochastic kernels.
\end{lema}

\begin{proof}
By Theorems \ref{thom:repsentation-consistency} and \ref{convergence-thm}, $V$ is B-PWC. For any $\alpha \in A$, the function $Q({}\cdot{},\alpha,V) : S \rightarrow \mathbb{R}$ is B-PWC by \thomref{thom:B-PWC-closure-VI}. Let $\Phi_{Q}$ be a BFCP of $S$ such that $Q({}\cdot{},\alpha,V)$ is constant on each region of $\Phi_{Q}$ for $\alpha \in A$, and $\Phi_{A}$ be a BFCP of $S$ such that $A(s)$ is constant on each region of $\Phi_{A}$. Then, for $u_1 \in \mathbb{P}(A_1(s))$ and $u_2 \in \mathbb{P}(A_2(s))$, the function $Q'({}\cdot{},u_1,u_2) : S \rightarrow \mathbb{R}$, where:
\begin{eqnarray*}
Q'(s,u_1,u_2) & =& \mbox{$\sum\nolimits_{(a_1,a_2)\in A(s)}$} Q(s,(a_1,a_2),V)u_1(a_1)u_2(a_2)
\end{eqnarray*}
 for $s \in S$, is constant in each region of $\Phi_{Q} + \Phi_{A}$. Therefore, there exists a CON-$3$ solution $(\sigma_1,\sigma_2)$ of $Q'(s,u_1,u_2)$ and, since $\Phi_{Q} + \Phi_{A}$ is a BFCP, the result follows.
\end{proof}

%% file: conclusions.tex
\section{Conclusions}

We have proposed a novel modelling formalism called neuro-symbolic concurrent stochastic games (NS-CSGs) for representing probabilistic finite-state agents with {\revise neural}
perception mechanisms interacting in a shared, continuous-state environment. NS-CSGs have the advantage of allowing for the perception of a complex environment to be synthesised from data and implemented via NNs, while the safety-critical decision-making module is symbolic, explainable and knowledge-based.
{\revise We constrain the interface of the neural perception so that agents can reason symbolically using percepts, representing (exactly) learnt subregions of the continuous environment.}

For zero-sum discounted cumulative reward problems, we proved the existence and measurability of the value function of NS-CSGs under Borel measurability and piecewise constant restrictions.
We then presented the first 
{\revise computationally practical} 
B-PWC VI and Minimax-action-free PI algorithms with finite representations for computing the values and optimal strategies of NS-CSGs, assuming a fully observable setting, by proposing B-PWC, CON-PWL and CON-PWC functions.

The B-PWC VI algorithm is, at the region level, the same as VI for finite state spaces, but involves, at each iteration, a division of the uncountable state space into a possibly refined
finite set of regions (i.e., a BFCP). {\revise This is because taking the same action in two distinct states of some region in the current BFCP can yield states that need not be observationally equivalent.} The Minimax-action-free PI algorithm requires multiple divisions of the uncountable state space into BFCPs at each iteration. 
{\revise We resolve a number of technical issues to apply} 
~\cite{DB:21}, which ensures convergence, and,  by not requiring the solution of normal-form games or MDPs at each iteration,
reduces computational complexity. 
However, implementation of the Minimax-action-free PI algorithm is more challenging, requiring a distributed, asynchronous framework.
 We illustrated our approach by modelling a dynamic vehicle parking problem as an NS-CSG and synthesising approximately optimal values and strategies using B-PWC VI.
Future work will involve improving efficiency of the algorithms,  
{\revise generalising to more complex classes of neural perception mechanisms 
by exploiting preimage approximations~\cite{ZWK23}, 
and moving to equilibria-based (nonzero-sum) properties,
where initial progress has been made in a simpler setting of finite unfolding of our NS-CSG model \cite{YSD+22}.  We are also planning to extend the approach to two-sided partially observable neuro-symbolic POSGs, based on a recently proposed one-sided variant, which exploits a finite representation that generalises $\alpha$-vectors to approximate value computation~\cite{RY-GS-GN-DP-MK:23-2}, and for which \emph{online} strategy synthesis methods have been developed~\cite{YSN+24}.}

\startpara{Acknowledgements}
This project was funded by the ERC under the European Union’s Horizon 2020 research and innovation programme (\href{http://www.fun2model.org}{FUN2MODEL}, grant agreement No.~834115).